\documentclass{article} % For LaTeX2e
\usepackage{iclr2026_conference,times}

\usepackage{amsmath,amsfonts,bm}

\def\secref#1{\wasyparagraph\ref{#1}}
\def\eqref#1{equation~\ref{#1}}
\def\1{\bm{1}}

\DeclareMathAlphabet{\mathsfit}{\encodingdefault}{\sfdefault}{m}{sl}
\SetMathAlphabet{\mathsfit}{bold}{\encodingdefault}{\sfdefault}{bx}{n}

\def\gS{{\mathcal{S}}}

\def\gX{{\mathcal{X}}}

\DeclareMathOperator*{\argmin}{arg\,min}

\newcommand{\X}{\mathcal{X}}

\newcommand{\Sd}{\mathbb{S}}
\newcommand{\markov}{\mathcal{M}}
\newcommand{\reciprocal}{\mathcal{R}^{\textup{ref}}}
\newcommand{\process}{\mathcal{P}}

\DeclareMathOperator{\proj}{proj}
\newcommand{\KL}[2]{\text{KL}\left(#1\Vert #2\right)}

\usepackage{hyperref}
\usepackage{url}

\usepackage[table]{xcolor} % for cell coloring
\usepackage{booktabs}
\usepackage{multirow}
\usepackage{caption}
\usepackage{siunitx}

\usepackage{xcolor}
\definecolor{scorebad}{HTML}{D55E00}     % vermillion
\definecolor{scoremidbad}{HTML}{E69F00}  % orange
\definecolor{scoremidgood}{HTML}{F0E442} % yellow
\definecolor{scoregood}{HTML}{009E73}    % bluish-green

\newcommand{\wrapcell}[2]{%
  \if\relax\detokenize{#1}\relax
    #2%
  \else
    #1{#2}%
  \fi
}
\newcommand{\scorecell}[2][]{%
  \ifdim #2 pt<0.5pt
    \cellcolor{scorebad!45}{\wrapcell{#1}{\num[round-precision=2]{#2}}}%
  \else\ifdim #2 pt<0.75pt
    \cellcolor{scoremidbad!45}{\wrapcell{#1}{\num[round-precision=2]{#2}}}%
  \else\ifdim #2 pt<0.85pt
    \cellcolor{scoremidgood!45}{\wrapcell{#1}{\num[round-precision=2]{#2}}}%
  \else
    \cellcolor{scoregood!45}{\wrapcell{#1}{\num[round-precision=2]{#2}}}%
  \fi\fi\fi
}

\newcommand{\scoreklcell}[2][]{%
  \ifdim #2 pt<0.5pt
    \cellcolor{scoregood!45}{\wrapcell{#1}{\num[round-precision=1]{#2}}}%
  \else\ifdim #2 pt<2pt
    \cellcolor{scoremidgood!45}{\wrapcell{#1}{\num[round-precision=1]{#2}}}%
  \else\ifdim #2 pt<10pt
    \cellcolor{scoremidbad!45}{\wrapcell{#1}{\num[round-precision=1]{#2}}}%
  \else
    \cellcolor{scorebad!45}{\wrapcell{#1}{\num[round-precision=1]{#2}}}%
  \fi\fi\fi
}

\usepackage{mathtools}
\usepackage{wasysym}

\usepackage[most]{tcolorbox}
\usepackage{enumitem}   
\usepackage{subcaption}
\usepackage{adjustbox}

\usepackage{amsthm}
\usepackage{booktabs}
\usepackage{multirow}

\newtheorem{theorem}{Theorem}[section]

\newtheorem{proposition}{Proposition}[section]
\renewcommand{\eqref}[1]{(\ref{#1})}
\usepackage{wrapfig}
\allowdisplaybreaks
\usepackage{tikz}
\usetikzlibrary{arrows.meta,calc}

\title{\fontsize{16.8pt}{16.8pt}\selectfont Entering the Era of Discrete Diffusion Models: \\ A Benchmark for Schr\"odinger Bridges and \\ Entropic Optimal Transport}

\author{Xavier Aramayo Carrasco\thanks{Equal contribution}\ \ \thanks{Correspondence to: \texttt{xavier.aramayo2@gmail.com}}\\
    Applied AI Institute, \\
    Moscow, Russia
    \And
    Grigoriy Ksenofontov${^*}$ \\
    Applied AI Institute,\\
    MIRAI\thanks{Moscow Independent Research Institute of Artificial Intelligence} ,\\
    Moscow, Russia
    \And
    Aleksei Leonov \\
    AI Foundation and Algorithm Lab\\
    MIRAI$^{\ddagger}$,\\
    Moscow, Russia
    \And
    \hspace{15mm}Iaroslav Koshelev \\
    \hspace{15mm}AI Foundation and Algorithm Lab\\
    \hspace{15mm}Moscow, Russia
    \And
    \hspace{-15mm}Alexander Korotin \\
    \hspace{-15mm}Applied AI Institute,\\
    \hspace{-15mm}AXXX,\\
    \hspace{-15mm}Moscow, Russia
}

\iclrfinalcopy % Uncomment for camera-ready version, but NOT for submission.
\begin{document}

\maketitle
\vspace{-3mm}
\begin{abstract}
    \vspace{-3mm}
    The Entropic Optimal Transport (EOT) problem and its dynamic counterpart, the Schr\"odinger bridge (SB) problem, play an important role in modern machine learning, linking generative modeling with optimal transport theory. While recent advances in discrete diffusion and flow models have sparked growing interest in applying SB methods to discrete domains, there remains no reliable way to assess how well these methods actually solve the underlying problem. We address this challenge by introducing a benchmark for SB on discrete spaces. Our construction yields pairs of probability distributions with analytically known SB solutions, enabling rigorous evaluation. As a byproduct of building this benchmark, we obtain two new SB algorithms, DLightSB and DLightSB-M, and additionally extend prior related work to construct the $\alpha$-CSBM algorithm. We demonstrate the utility of our benchmark by evaluating both existing and new solvers in high-dimensional discrete settings. This work provides the first step toward proper evaluation of SB methods on discrete spaces, paving the way for more reproducible future studies. The code for the benchmark and all associated experiments is available at 
    \vspace{-2mm}
    \begin{center}
        \url{https://github.com/gregkseno/catsbench}.
    \end{center}
\end{abstract}

\vspace{-5mm}
\section{Introduction}
\label{section:introduction}
\vspace{-2mm}
The Entropic Optimal Transport \citep[EOT]{cuturi2013sinkhorn} problem and its dynamic counterpart, the Schr\"odinger bridge \citep[SB]{schrodinger1931umkehrung}, have recently attracted significant attention in the machine learning community due to their relevance for generative modeling and unpaired learning. A variety of methods have been developed to solve these problems in \emph{continuous data spaces} such as \citep{daniels2021score,gushchin2023entropic,gushchin2024adversarial,mokrov2023energy,vargas2021solving,chen2021likelihood,shi2023diffusion,de2024schr,korotin2024light,gushchin2024light}.  

At the same time, much real world data are \emph{discrete by nature}, including text \citep{austin2021structured, gat2024discrete}, molecular graphs \citep{vignac2022digress, qin2024defog, luo2024crystalflow}, and protein sequences \citep{campbell2024generative}. Others are \emph{discrete by construction}, such as vector-quantized representations of images and audio \citep{van2017neural, esser2021taming}. 

Given the prevalence of such discrete data and the rapid progress in discrete diffusion/flow models \citep{hoogeboom2021argmax, austin2021structured, campbell2022continuous, lou2023discrete, sahoo2024simple, campbell2024generative, gat2024discrete}, research on SBs has attracted growing attention in recent years. For instance, several works have already taken first steps in this direction (\citealp[DDSBM]{kim2024discrete};\citealp[CSBM]{ksenofontov2025categorical}), adapting diffusion methodologies from (\citealp[DiGress]{vignac2022digress};\citealp[D3PM]{austin2021structured}), respectively. Yet, beyond these initial studies, practical, broadly applicable solvers for discrete-space EOT/SB remain largely absent.

To make progress on this front, it is also important to have reliable ways to evaluate SB solvers. In practice, EOT/SB methods are often assessed using proxy metrics such as FID \citep{heusel2017gans} or mean-squared-error between input and output. While being useful, these metrics only indirectly reflect whether a method truly solves the EOT/SB problem, since they can be strongly influenced by parameterization, regularization, and other implementation details. Evaluation benchmarks address this limitation by providing a controlled setting in which solvers can be compared against known SB solutions. This allows performance differences to be attributed directly to the underlying algorithms. Consistent with this goal, several optimal transport (OT) benchmarks have been proposed \citep{korotin2021neural, korotin2022kantorovich}, and similar efforts have recently appeared for continuous-state SB \citep{gushchin2023building}. However, no analogous benchmark currently exists for discrete data.

From the discussion above, two key limitations in discrete-space EOT/SB research emerge: \textbf{(1)} lack of a benchmark to assess the performance of the solvers and \textbf{(2)} the limited availability of discrete-space EOT/SB solvers. In this work, we address both issues; our \textbf{contributions} are detailed below: 
\vspace{-6mm}
\begin{itemize}[leftmargin=*]
    \item \textbf{Methodology.} We present a general methodology to create pairs of discrete probability distributions with known SB solutions (\wasyparagraph\ref{subsection:benchmark_main_theorem}). To overcome tractability issues of the methodology in discrete spaces, we introduce a CP-based parameterization (\wasyparagraph\ref{subsection:benchmark_parameterization}). This parameterization yields a closed-form SB and enables a practically feasible benchmark construction.
    
    \item \textbf{Algorithms.} We present various adaptations of continuous-state SB algorithms to discrete settings and evaluate them on our benchmark. Specifically, we develop the following algorithms:
    \begin{itemize}
        \item DLightSB (\wasyparagraph\ref{section:dlight_sb}) and DLightSB-M (\wasyparagraph\ref{section:dlight_sb_m}) -- solvers obtained as a byproduct of the benchmark, mirroring LightSB and LightSB-M \citep{korotin2024light, gushchin2024light};
        \item $\alpha$-CSBM (\wasyparagraph\ref{section:alpha_csbm}) -- a solver that combines CSBM \citep{ksenofontov2025categorical} with the online update strategy of $\alpha$-DSBM \citep{de2024schr}.
    \end{itemize}
\end{itemize}

\textbf{Notation.}  We consider a discrete state space $\X = \Sd^D$, where $\Sd = \{0, 1, \dots, S-1\}$ is the set of $S$ categories and $D$ is the dimensionality. Each $x \in \X$ is a $D$-dimensional vector $x \!= \!(x^1, \dots, x^D)$. Time is discretized as $\{t_n\}_{n=0}^{N+1}$ with \mbox{$0 \!= \!t_0 \!< \!t_1 \!< \!\dots \!< \!t_N \!<\! t_{N+1}\! = \!1$}. This gives $N\!+\!2$ time points and defines the \emph{path space} $\X^{N+2}$ with the tuple \mbox{$x_{\text{in}} \!\coloneqq \!(x_{t_1}, \dots, x_{t_N}) \!\in\! \X^N$} collecting the intermediate states. The set $\process(\X^{N+2})$ comprises all discrete time stochastic processes on the path space, with $\markov(\X^{N+2}) \!\subset \!\process(\X^{N+2})$ denoting the subset of \emph{Markov processes}. Any $q \!\in \!\markov(\X^{N+2})$ admits forward and backward representations: \mbox{$q(x_0, x_{\text{in}}, x_1)=q(x_0)\prod_{n=1}^{N+1} q(x_{t_n} | x_{t_{n-1}})=q(x_1)\prod_{n=1}^{N+1} q(x_{t_{n-1}} | x_{t_n})$}. Finally, $q(\cdot|\cdot)$ is used to denote \textit{conditional} ($x_0\!\rightarrow\!x_1$) and \textit{transition} ($x_{t_{n-1}}\!\rightarrow\!x_{t_{n}}$) distributions.

\vspace{-2mm}
\section{Background}
\label{section:background_problem_statement}
\vspace{-2mm}
This section recalls EOT and SB, which form the core framework for our benchmark and methods. We begin with the dynamic SB formulation and its connection to the static SB problem (\wasyparagraph\ref{subsection:background_sb_discrete}). We then recall different reference processes examples (\wasyparagraph\ref{subsection:formulation_reference_process}) that induce practical cost functions and thereby connect SB to the EOT framework (\wasyparagraph\ref{subsection:background_eot}). Finally, we introduce the discrete-space generative EOT/SB task and specify the benchmark evaluation criteria used in this work (\wasyparagraph\ref{subsection:background_problem_setting}).

\vspace{-2mm}
\subsection{Dynamic and Static Schr\"odinger Bridges on Discrete Spaces}
\label{subsection:background_sb_discrete}
\vspace{-1mm}
\textbf{Dynamic Schr\"odinger Bridge.} The original SB problem \citep{schrodinger1931umkehrung, schrodinger1932theorie, leonard2013survey} seeks to find a process $q^* \in \process(\X^{N+2})$ interpolating between an initial distribution $p_0$ at $t_0=0$ and a final distribution $p_1$ at $t_{N+1}=1$. This distribution is found by minimizing the Kullback-Leibler (KL) divergence with respect to a given \textit{Markov reference process} $q^{\text{ref}} \in \markov(\X^{N+2})$ subject to the marginal constraints $p_0(x_0)=q(x_0)$ and $p_1(x_1)=q(x_1)$. One finds the following \textit{optimal process}:

\vspace{-2mm}
\begin{equation}
    \label{equation:dynamic_sb}
    q^*=\argmin_{q \in \Pi_N(p_0, p_1)} \KL{q(x_0, x_{\text{in}}, x_1)}{q^{\text{ref}}(x_0, x_{\text{in}}, x_1)},
\end{equation}
\vspace{-2mm}

where $\Pi_N(p_0, p_1) \subset \process(\X^{N+2})$ denotes the subset of $\mathcal{X}$-valued stochastic processes which have $p_0$ and $p_1$ as marginals at times $t_0=0$ and $t_{N+1}=1$, respectively. In other words, the dynamic SB problem seeks the stochastic process $q^*$ that minimally deviates from a reference process $q^{\text{ref}}$ while respecting the boundary distributions $p_0$ and $p_1$. 

\textbf{Static Schr\"odinger Bridge.} The previously introduced dynamic SB problem also allows a static formulation. Linking them begins with observing that \eqref{equation:dynamic_sb} admits the following decomposition:

\vspace{-2mm}
\begin{equation}
    \label{equation:static_sb_decomp}
    \min_{q \in \Pi_N(p_0, p_1)}\Bigr[ \KL{q(x_0, x_1)}{q^{\text{ref}}(x_0, x_1)} + \mathbb{E}_{q(x_0,x_1)}\KL{q(x_{\text{in}}|x_0,x_1)}{q^{\text{ref}}(x_{\text{in}}|x_0,x_1)}\Bigr].
\end{equation}
\vspace{-2mm}

We further note that the conditional KL term in \eqref{equation:static_sb_decomp} vanishes when $q(x_{\text{in}}|x_0,x_1) = q^{\text{ref}}(x_{\text{in}}|x_0,x_1)$. Thus, we restrict $q$ to the set of processes that satisfy this condition. This set is known as \emph{the reciprocal class} of $q^{\text{ref}}$ and it is denoted by $\reciprocal(\X^{N+2}) \subset \process(\X^{N+2})$. Under this restriction, the optimization reduces to the first KL term alone, leading directly to the static SB problem:

\vspace{-2mm}
\begin{equation}
    q^*(x_0, x_1)=\argmin_{q \in \Pi(p_0, p_1)} \KL{q(x_0, x_1)}{q^{\text{ref}}(x_0, x_1)} 
    \label{equation:static_sb},
\end{equation}
\vspace{-2mm}

where $\Pi(p_0,p_1)\subset\process(\X^2)$ denotes the set of joint distributions with marginals $p_0$ and $p_1$, and $q^*(x_0,x_1)$ is \textit{the optimal joint distribution}.

Notably, the static SB formulation is closely related to an EOT problem. This connection is established since the reference process induces a corresponding cost function. To make this link explicit, in the next sections, we first introduce commonly used reference processes on discrete spaces and then derive the corresponding connection between static SB and EOT.

\vspace{-2mm}
\subsection{Examples of suitable Reference Processes }
\label{subsection:formulation_reference_process}
\vspace{-1mm}

The key ingredient in both SB formulations is the Markov reference process $q^{\text{ref}}$. In discrete space, it is typically modeled as a discrete-time Markov chain with strictly positive transitions, i.e., $q^{\text{ref}}(x_{t_n} | x_{t_{n-1}})\! > \!0$ for all $(x_{t_{n-1}}, x_{t_n})$. As in standard discrete diffusion models \citep{austin2021structured}, we focus on {factorizable reference processes} $q^{\text{ref}}(x_{t_n} | x_{t_{n-1}})\!=\!\prod^{D}_{d=1}q^{\text{ref}}(x_{t_{n}}^{d} | x_{t_{n-1}}^{d})$ and thus present the construction in the one-dimensional case. We further assume time-homogeneity ($q^{\text{ref}}(x^d_{t_n} | x^d_{t_{n-1}})\!=\!Q^{\text{ref}} \in [0,1]^{S\times S}$ for all $n$), so that the cumulative transition distributions for $n$-steps are defined by the matrix power ${\overline{Q}_{n}^{\text{ref}}\!=\![Q^{\text{ref}}]^{n}}$.

\vspace{-1mm}
\textbf{Remark.} The reference process $q^{\text{ref}}$ can also be defined in continuous time, where transitions are specified by transition rates (see, e.g., \citep{campbell2022continuous}). In this setting, controlling the dynamics is often less direct; moreover, discrete-time Markov chains form a strictly larger class, since not every chain admits a continuous-time analogue (the embeddability problem \citep{kingman1962imbedding}). For the convenience of benchmark construction, we therefore focus on the discrete-time setting.

\vspace{-1mm}
We now introduce two popular diffusion-like transitions: uniform \citep{hoogeboom2021argmax, campbell2022continuous} and Gaussian-like \citep{austin2021structured}.

\vspace{-2mm}
\paragraph{Uniform Reference Process ($q^{\text{unif}}$).} For unordered data, where no relation exists between categories, a natural choice is the so-called uniform transition matrix. In this case, for each dimension $d$, the elements of the transition matrix $Q^{\text{ref}}$ are defined by

\vspace{-3mm}
\begin{align}
    [Q^{\text{ref}}]_{x^d_{t_{n-1}},x^d_{t_n}}=\begin{cases}
      1-\gamma, & \text{if } x^d_{t_n}=x^d_{t_{n-1}}, \\
      \frac{\gamma}{S-1}, & \text{if } x_{t_n}^d\neq x_{t_{n-1}}^d,
    \end{cases}
    \label{equation:piref_uniform}
\end{align}
\vspace{-2mm}

where $\gamma\in [0,1]$ is a stochasticity parameter. This reference process assigns equal probability to transitioning into any different category, thereby ignoring any inherent ordering among categories. In Appendix~\ref{appendix:static_q_ref}, we provide \underline{a closed-form expression} for $q^{\text{ref}}(x_1^d | x_0^d)=\overline{Q}^{\text{ref}}_{N+1}$ in the uniform case.

\vspace{-2mm}
\paragraph{Gaussian Reference Process ($q^{\text{gauss}}$).} For ordered data, where categories are expected to exhibit meaningful relations, a Gaussian-like transition matrix is more appropriate. With the stochasticity parameter $\gamma > 0$ and the maximum category distance $\Delta = S-1$, the transition probabilities are  

\vspace{-2mm}
\begin{equation}
    [Q^{\text{ref}}]_{x^d_{t_{n-1}}, x^d_{t_n}} =
    \frac{\exp\!\left(-\frac{4(x^d_{t_n}-x^d_{t_{n-1}})^2}{(\gamma \Delta)^2}\right)}
    {\sum\limits_{\delta=-\Delta}^{\Delta} \exp\!\left(-\frac{4\delta^2}{(\gamma \Delta)^2}\right)},
    \qquad x^d_{t_n} \neq x^d_{t_{n-1}}.
\end{equation}

The diagonal entries take the remaining probability so that each row sums to 1.

%The dynamic formulation of the reference process establishes a discrete diffusion SB process that closely aligns with the framework explored in the Discrete Diffusion Schr\"odinger Bridge Matching (DDSBM) paper \citep{kim2024discrete}. A principal distinction, however, is that their formulation is time continuous as it operates on an infinite number of time steps (i.e., $N=\infty$), while we compare approaches that use a finite amount ($N< \infty$). 

\vspace{-2mm}
\subsection{Entropic Optimal Transport on Discrete Spaces}
\label{subsection:background_eot}
\vspace{-1mm}
Following the construction of the Markov reference process, the static SB problem (\S\ref{equation:static_sb}) takes a form equivalent to the EOT problem \citep{cuturi2013sinkhorn}. Concretely, expressing $q^{\text{ref}}(x_0,x_1)=q^{\text{ref}}(x_0)q^{\text{ref}}(x_1|x_0)$ and setting $q^{\text{ref}}(x_0)=p_0(x_0)$ lets us rewrite the minimization in \eqref{equation:static_sb} as
\begin{align}
    \label{equation:static_to_eot}
    \min_{q \in \Pi(p_0, p_1)}&\KL{q(x_0, x_1)}{q^{\text{ref}}(x_0, x_1)} =\notag\\[-2mm]
    &= \min_{q \in \Pi(p_0, p_1)} \sum_{x_0, x_1} q(x_0, x_1)
       \log \frac{q(x_0,x_1)}{q^{\text{ref}}(x_0)q^{\text{ref}}(x_1|x_0)} \notag\\
    &= \min_{q \in \Pi(p_0, p_1)} -H(q)
       - \sum_{x_0, x_1} q(x_0, x_1)\log q^{\text{ref}}(x_1| x_0)
       - \!\!\!\underbrace{\sum_{x_0, x_1} q(x_0, x_1)\log q^{\text{ref}}(x_0)}_{=-H(p_0)=\text{const}} \\[-2mm]
    &= \min_{q \in \Pi(p_0, p_1)} \mathbb{E}_{q(x_0,x_1)}\!\bigl[-\log q^{\text{ref}}(x_1|x_0)\bigr]
       - H(q) - \text{const} \notag\\
    &= \min_{q \in \Pi(p_0, p_1)} \mathbb{E}_{(x_0,x_1)\sim q}\!\bigl[c(x_0, x_1)\bigr]
       - H(q) - \text{const}, \notag
\end{align}
where $H(q)$ is the entropy of $q$, while $H(p_0)$ remains constant when minimizing over $q$. Thus, the static SB formulation becomes equivalent to the entropy-regularized optimal transport problem with cost $c(x_0,x_1) = -\log q^{\text{ref}}(x_1|x_0)$. This establishes a direct correspondence between both SB formulations and EOT, which we collectively referred to as the EOT/SB problem, enabling the construction of a unified benchmark for all three problems.

\vspace{-2mm}
\subsection{Discrete Generative Problem Setup and Evaluation Protocol}
\label{subsection:background_problem_setting}
\vspace{-1mm}
Building on this background, we formalize the \emph{generative discrete-space EOT/SB task}. This is a well-established problem in the literature \citep{kim2024discrete,ksenofontov2025categorical}. In short, the goal is to learn an optimal conditional distribution that transports a probability distribution on discrete spaces using available empirical samples. Formally, we consider the following setup:
\begin{tcolorbox}[colback=gray!10, 
    colframe=black, 
    sharp corners, 
    boxrule=0.8pt,
    boxsep=2pt,       % space between frame and content
    left=4pt,         % left padding
    right=4pt,        % right padding
    top=4pt,          % top padding
    bottom=4pt        % bottom padding
]
    We assume the learner is given empirical datasets $\{x_0^{(i)}\}_{i \in I_0}\!$ and $\{x_1^{(j)}\}_{j \in I_1}$, $x_0^{(i)}, x_1^{(j)}\in \X$, consisting of i.i.d. samples from the unknown distributions $p_0, p_1 \in \process(\X)$ where $\X$ is a discrete state space. Then, the task is to use these samples to find a solution $q^*$ to the EOT/SB problem (\ref{equation:dynamic_sb}, \ref{equation:static_sb}, \ref{equation:static_to_eot}) between $p_0$ and $p_1$ for a given reference $q^{\text{ref}}$. Moreover, the solution should support out-of-sample generation so that for any new $(x_0^{\text{new}})$ one can generate $x_1^{\text{new}} \!\sim \!q^{\text{model}}(x_1|x_0^{\text{new}})$.
\end{tcolorbox}

Despite recent progress in developing discrete-space SB methods for this task, there is still no standard evaluation protocol. The main obstacle is the lack of discrete datasets with known ground-truth EOT/SB solutions, i.e., pairs of marginals $(p_0,p_1)$ for which the optimal conditional $q^*(x_1 | x_0)$ is available. Access to such ground truth enables direct comparison with a learned model $q^{\text{model}}(x_1 | x_0)$, allowing evaluation of how accurately a method solves the underlying EOT/SB problem rather than relying on proxy metrics. Inspired by \citep{gushchin2023building}, we therefore propose a pipeline that generates ground-truth benchmark instances, applicable to discrete-space solvers.

% Despite recent progress in the development of SB methods that solve this task, there remains no standard methodology for performance evaluation, mainly because datasets with known ground-truth solutions are unavailable. In particular, the ground truth refers to pairs of marginal distributions $(p_0, p_1)$ for which the optimal SB coupling $q^(x_1 \mid x_0)$ is known. Inspired by \citep{gushchin2023building}, we present a pipeline that generates ground-truth pairs for evaluation.} Such datasets provide more informative metrics and offer a consistent framework for assessing the performance of discrete-space SB methods.

%  Our benchmark construction applies to both static and dynamic SB solvers, as it enables direct comparison between the learned SB coupling and the optimal coupling. Importantly, our evaluation does not depend on formulation-specific quantities tied to either the static or dynamic setting. This generality broadens the applicability of our pipeline, making it suitable across different SB formulations.}

\textbf{Remark.} Our paper is not related to the discrete EOT, which includes solvers such as the Sinkhorn algorithm \citep{cuturi2013sinkhorn} or gradient-based methods \citep{dvurechensky2018computational}. These approaches are designed for a non-generative problem setting, see \citep[\wasyparagraph2.3]{ksenofontov2025categorical}. They treat samples as empirical distributions $p_0(x_0)\!=\!\tfrac{1}{|I_0|}\!\sum_{i \in I_0} \delta_{x_0^{(i)}}$, $p_1(x_1)\!=\!\tfrac{1}{|I_1|}\!\sum_{j \in I_1} \delta_{x_1^{(j)}}$. The resulting joint distribution is then a bi-stochastic $|I_0| \times |I_1|$ matrix, which does not support out-of-sample generation. While some extensions attempt to provide inference for unseen data \citep{hutter2021minimax,pooladian2021entropic,manole2021plugin,deb2021rates}, they are designed for continuous spaces ($\X\!=\!\mathbb{R}^D$) rather than the discrete ($\X\!=\!\Sd^D$) considered in our work.

\vspace{-2mm}
\section{Benchmark}
\label{section:benchmark}
\vspace{-2mm}
In this section, we address the absence of evaluation benchmarks for discrete-space EOT/SB solvers by proposing a novel benchmark construction. In \wasyparagraph\ref{subsection:benchmark_main_theorem} we present the theoretical foundations of the construction. Next, we make it tractable via a CP parameterization in \wasyparagraph\ref{subsection:benchmark_parameterization}. Finally, we use this parameterization to construct a high-dimensional Gaussian mixture benchmark in \wasyparagraph\ref{subsection:benchmark_pairs_construction}. We provide \underline{detailed proofs} for all stated theoretical results in Appendix~\ref{appendix:proofs}.

\vspace{-2mm}
\subsection{Main Theorem for Benchmark Construction}
\label{subsection:benchmark_main_theorem}
\vspace{-1mm}
For an initial distribution $p_0 \in \process(\X)$, we aim to construct a target distribution $p_1 \in \process(\X)$ such that the optimal joint distribution $q^{*}(x_0,x_1)$ between them is known by our construction. The resulting pair $(p_0, p_1)$ together with $q^*$ can then be used as benchmark data for evaluating SB methods. Our following theorem plays the key role in the construction of benchmark pairs. 
\begin{theorem}[Benchmark Pair Construction for Discrete-Space EOT/SB]
    Let $p_0\in \mathcal{P}(\X)$ be a given initial distribution on a discrete space $\X$ and $v^*:\X \rightarrow \mathbb{R}$ be a given scalar-valued function. Consider a joint distribution $q^* \in \mathcal{P}(\X^{2})$ such that $q^*(x_0) = p_0(x_0)$ and \mbox{$q^*(x_1| x_0) \propto v^*(x_1)q^{\textup{ref}}(x_1| x_0)$} define $p_1(x_1) \coloneqq q^*(x_1)$ as its second marginal. Then $q^*$ together with the reference process $q^{\textup{ref}}$ defines the discrete-space EOT/SB (\ref{equation:dynamic_sb},\ref{equation:static_sb},\ref{equation:static_to_eot}) between $p_0$ and $p_1$.
    \label{theorem:benchmark_construction}
\end{theorem}

\vspace{-1mm}
Similar results in continuous spaces appear in \citep{gushchin2023building}. In Theorem~\ref{theorem:benchmark_construction} we provide a discrete-space analog which shows that any pair $(p_0, v^{*})$ induces a corresponding pair $(p_0, p_1)$ with a closed-form $q^*(x_1|x_0)$ on \textbf{discrete space}. We refer to the latter pair as a \textit{benchmark pair}. However, this construction specifies $q^{*}(x_1 | x_0)$ only up to proportionality, necessitating the normalized form:

\vspace{-2mm}
\begin{equation}
    q^*(x_1| x_0)=\tfrac{1}{c^*(x_0)}v^*(x_1) q^{\text{ref}}(x_1| x_0),
    \label{equation:normalized_conditional}
\end{equation}
\vspace{-2mm}

where $c^*(x_0) \coloneqq \sum_{x_1 \in \X}v^*(x_1)q^{\text{ref}}(x_1| x_0)$ is the normalization constant. Although this provides a closed-form expression for $q^*$, implementing it in high-dimensional spaces ($|\mathcal{X}| = S^D$) remains computationally challenging. In particular, evaluating the normalizing constant and sampling from $q^*$ are non-trivial tasks. To address these challenges, we introduce a CP-parameterization that allows efficient computation and sampling, as detailed in the next section.

\vspace{-2mm}
\subsection{Practical Parameterization}
\vspace{-1mm}
\label{subsection:benchmark_parameterization}
We parameterize the scalar-valued function $v^*$ using a rank-1 Canonical Polyadic (CP) decomposition, which captures interactions across dimensions and provides a compact yet expressive representation. Such decompositions act as universal approximators, capable of modeling complex functions when the rank is sufficiently large \citep{cohen2016expressivepowerdeeplearning,basharin2025fasterlanguagemodelsbetter}. Thus, $v^*$ is written as
\begin{equation}
    v^{*}(x_1) \;=\; \sum_{k=1}^K \beta_k \prod_{d=1}^D r_k^d[x_1^d].
    \label{equation:cp_decomposition}
\end{equation}
Expression \eqref{equation:cp_decomposition} defines a mixture of $K$ \emph{factorizable distributions}, each with weight $\beta_k \ge 0$. For each mixture component $k$ and dimension $d$, probabilities are defined by non-negative vectors $r_k^d \in \mathbb{R}^S_+$, referred to as \emph{CP cores}, where $r_k^d[x_1^d]$ denotes the probability of state $x_1^d$. %The key advantage of this parameterization is that the factorization across dimensions makes both the normalizing constant $c^*(x_0)$ and the conditional distribution $q^*(x_1 | x_0)$ computationally tractable. Specifically, the product structure allows efficient ancestral sampling by drawing each dimension independently.

%\greg{write about ancestral sampling for both propositions}
\begin{proposition}[Tractable Parameterization of Conditional Distributions]
    \label{prop:tractable_parameterization}
    Let $q^{\textup{ref}}$ be a factorizable Markov reference process on a discrete space $\X$. Using the CP decomposition of the scalar-valued function $v^*$ in \eqref{equation:cp_decomposition}, the optimal conditional distribution satisfies:
    \begin{minipage}{0.48\linewidth} 
        \vspace{-2mm}
        \begin{gather} 
            q^*(x_1|x_0) = \notag\\[-1mm] \frac{1}{c^*(x_0)} \sum_{k=1}^{K} \beta_k \prod_{d=1}^{D} \Big[r_k^d[x_1^d]q^{\textup{ref}}(x_1^d| x^d_0)\Big]; 
            \label{equation:conditional_closed_form} 
        \end{gather} 
    \end{minipage} 
    \hfill 
    \begin{minipage}{0.48\linewidth} 
        \vspace{-2mm}
        \begin{gather} 
            c^*(x_0) = \notag\\[-1mm] \sum_{k=1}^{K} \beta_k \prod_{d=1}^{D} \left(\sum_{x_1^d=0}^{S-1} r_k^d[x_1^d] q^{\textup{ref}}(x_1^d|x_0^d)\right) 
            \label{equation:normalizing_constant} 
        \end{gather} 
    \end{minipage}
    
    where $c^*(x_0)$ is the normalization constant. This formulation expresses $q^*(x_1|x_0)$ as a mixture of $K$ factorizable distributions, each weighted by a scalar coefficient $\beta_k$.
\end{proposition}

A key consequence of Proposition \ref{prop:tractable_parameterization} is that $c^*(x_0)$ and $q^*(x_1 | x_0)$ are computationally tractable, since both expressions factorize into products of one-dimensional sums. In particular, instead of summing over the full joint space of size $S^D$, computing $c^*(x_0)$ requires only $K$ products of $D$ scalar sums over $S$ states, reducing the complexity from $\mathcal{O}(S^D)$ to $\mathcal{O}(KDS)$. The evaluation of $q^*(x_1|x_0)$ is similar. In practice, computations are performed using log-sum-exp operations for numerical stability. To sample from $q^*(x_1|x_0)$, we use ancestral sampling: we first choose a mixture component $k$ with probability proportional to $\beta_k \prod_{d=1}^{D} \sum_{x_1^d} r_k^d[x_1^d]\, q^{\textup{ref}}(x_1^d \mid x_0^d)$, and then, conditioned on $k$, sample each coordinate $x_1^d$ independently from $q^*(x_1^d \mid x_0^d) \propto r_k^d[x_1^d]\, q^{\textup{ref}}(x_1^d \mid x_0^d)$.

Previously, we described the benchmark construction in the static SB (equivalently, EOT) setting, with focus on $q^*(x_1 | x_0)$. Nevertheless, a similar result can be obtained for the dynamic SB and its forward Markov representation defined by transition distributions $q^*(x_{t_n} | x_{t_{n-1}})$. More precisely, these transitions can be obtained by reweighting the reference process transitions $q^{\text{ref}}(x_{t_n} | x_{t_{n-1}})$ with time-dependent scalar-valued functions $\phi^*_t$ \citep[Thm.~2]{georgiou2015positive}:

\vspace{-5mm}

%For the upcoming derivations, we consider the dynamic SB and its forward Markov representation, defined by transition distributions $q^*(x_{t_n} | x_{t_{n-1}})$. Similar to Theorem \ref{theorem:benchmark_construction}, these transitions can be obtained by reweighting the reference process transitions $q^{\text{ref}}(x_{t_n} | x_{t_{n-1}})$ with time-dependent scalar-valued functions $\phi^*_t$ \citep[Thm.~2]{georgiou2015positive}:}
\begin{equation} 
    \label{equation:normalized_transitional} 
    q^*(x_{t_{n}}| x_{t_{n-1}}) = q^{\text{ref}}\left(x_{t_{n}}| x_{t_{n-1}}\right)\frac{\phi^{*}_{t_n}(x_{t_{n}})}{\phi^{*}_{t_{n-1}}(x_{t_{n-1}})}, \qquad \phi^{*}_{t_{n}}(x_{t_n}) =\mathbb{E}_{q^{\mathrm{ref}}(x_1| x_{t_n})} \big[v^{*}(x_1)\big]. 
\end{equation}

\vspace{-3mm}

The next proposition provides a tractable form of the corresponding transition distributions.

%[Tractable SB Transition Distributions under CP Parameterization]
\begin{proposition}[Tractable Parameterization of Conditional SB Transition Distributions]
    \label{proposition:sb_trainsition_distribution}
    Let $q^{\textup{ref}}$ be a factorizable Markov reference process on a discrete space $\X$. Using the CP decomposition of the scalar-valued function $v^*$ in \eqref{equation:cp_decomposition} together with the definition of the time-dependent functions $\phi_{t_n}^*$ in \eqref{equation:normalized_transitional}, the optimal transition distributions satisfy:
    \vspace{-2mm}
    \begin{equation}
        \label{equation:transition_closed_form}
        q^{*}(x_{t_{n}} | x_{t_{n-1}}) \propto  q^{\textup{ref}}(x_{t_{n}} | x_{t_{n-1}}) \sum_{k=1}^K \beta_k \prod_{d=1}^D u_{k,t_{n}}^d[x_{t_{n}}^d].
        \vspace{-2mm}
    \end{equation}
    where $u^d_{k, t_{n}}\left[x^d_{t_{n}}\right] = \sum^{S-1}_{x^d_1=0}q^{\textup{ref}}(x_1^d | x^d_{t_n})r^d_k\left[x^d_1\right]$. Sampling is done via ancestral sampling.
\end{proposition}

\vspace{-2mm}

By the same argument as for $q^*(x_1 | x_0)$, the transition distributions $q^*(x_{t_n} | x_{t_{n-1}})$ are tractable. In practice, normalization of \eqref{equation:transition_closed_form} is carried out in the log-domain by subtracting the log-sum-exp over all states to obtain numerically stable transition probabilities.

\vspace{-2mm}
\subsection{High-dimensional Gaussian Mixtures Benchmark Construction}
\label{subsection:benchmark_pairs_construction}
\vspace{-1mm}
We now instantiate the proposed construction. We set $p_0$ as a discretized Gaussian on $D\!\in\!\{2,16,64\}$ dimensions with $S\!=\!50$ categories. For $v^*$, we use $K\!=\!4$ components with uniformly initialized weights $\beta \!\in\! \mathbb{R}^K$. The CP cores are initialized using discretized Gaussian probabilities, with means uniformly sampled on a sphere of radius 5 and a standard deviation fixed to $\{1.5,1.5,2.5\}$ for the lowest to highest dimensions, respectively. Given $p_0$ and $v^*$, we then construct $p_1$. This initialization produces a target $p_1$ resembling a discretized Gaussian mixture with a clear visual structure. We construct pairs under different reference processes: $q^{\text{gauss}}$ with $\gamma \in \{0.02, 0.05\}$ and $q^{\text{unif}}$ with $\gamma \in \{0.005, 0.01\}$, using $N+1=128$ for both. Figure \ref{subfigure:benchmark pairs} shows the resulting benchmark pairs.

%In this section, we introduce available approaches and, in order to address the limited availability of discrete-space EOT/SB solvers, we also introduce novel approaches.
\vspace{-2mm}
\section{Solvers for Evaluation}
\label{section:solvers_for_evaluation}
\vspace{-2mm}
In this section, we recall available discrete EOT/SB solvers as well as addressing the limited availability of such approaches. We begin by recalling \emph{Categorical Schr\"odinger Bridge Matching (CSBM)} \citep{ksenofontov2025categorical}, which is an existing SB method tailored to categorical distributions. We then introduce \emph{$\alpha$-CSBM}, which incorporates the online update strategy of \citep{de2024schr} into the CSBM framework. Next, we propose \emph{Discrete Light Schr\"odinger Bridge (DLightSB)} obtained as a byproduct of our benchmark construction (\wasyparagraph\ref{section:benchmark}) and extending \citep{korotin2024light} to the discrete setting. Finally, we present \emph{Discrete Light Schr\"odinger Bridge Matching (DLightSB-M)}, a dynamic extension of DLightSB following \citep{gushchin2024light}. 

\vspace{-2mm}
\subsection{Categorical Schr\"odinger Bridge Matching (CSBM)}
\label{section:csbm}
\vspace{-1mm}
In \citep[Theorem 3.1]{ksenofontov2025categorical}, the discrete-space dynamic SB problem is addressed by the \emph{discrete-time Iterative Markovian Fitting (D-IMF) procedure}, whose convergence is established by extending the discrete-time existence theorem of (\citealp[Theorem 3.6]{gushchin2024adversarial}) to the discrete space and time setting. The exponential convergence rates of D-IMF can be found in \citep{sokolov2025exponentialconvergencerateiterative}. This constructive method uses the fact that the dynamic SB $q^*$ is both reciprocal ($q^* \in \reciprocal(\X^{N+2})$) and Markov ($q^* \in \markov(\X^{N+2})$). The D-IMF algorithm alternates between projections onto these two sets, starting from an initial process $q^0(x_0, x_1)q^{\text{ref}}(x_{\text{in}}|x_0,x_1)$, where $q^0(x_0, x_1)\in\Pi(p_0,p_1)$, e.g., $p_0(x_0)p_1(x_1)$, and converges to the SB $q^{*}$ in $\text{KL}$. Namely, %\alex{not clear from below that projections alternate}
\vspace{-1mm}
\begin{equation*}
    \begin{tikzpicture}[>=Stealth, baseline=(A.base)]
      \node (A) at (0,0) {$q^{2l (+2)}$};
      \node (B) at (6,0) {$q^{2l+1}$};
      \draw[->, transform canvas={yshift=0.7ex}, shorten >=2pt, shorten <=2pt] (A) to node[above] {$\proj_{\markov}$} (B);
      \draw[->, transform canvas={yshift=-0.7ex}, shorten >=2pt, shorten <=2pt] (B) to node[below] {$\proj_{\reciprocal}$} (A);
      \node[overlay, anchor=west] at ($(B.east)+(0.8em,0)$) {$l=0,1,\dots$};
    \end{tikzpicture}
    \vspace{-1mm}
\end{equation*}
where 
\begin{gather}
    \label{equation:reciprocal_projection}
    [\proj_{\reciprocal}(q)](x_0, x_{\text{in}}, x_1) 
    = q(x_0, x_1) q^{\text{ref}}(x_\text{in} | x_0, x_1), 
    \quad \forall q \in \process(\X^{N+2}), \\
    \label{equation:markovian_projection}
    \!\!\![\proj_{\markov}(q)](x_0, x_{\text{in}}, x_1) 
    = \!\!\!\!\argmin_{m \in \markov(\X^{N+2})} 
       \!\!\!\!\KL{q(x_0, x_{\text{in}}, x_1)}{m(x_0, x_{\text{in}}, x_1)}, 
       \quad \!\!\forall q \in \reciprocal(\X^{N+2}).
\end{gather}

\vspace{-2mm}
\textbf{Loss.} Fortunately, the reciprocal part \eqref{equation:reciprocal_projection} is straightforward via ancestral sampling. In turn, to make Markovian step \eqref{equation:markovian_projection} tractable, the authors parameterize the transitions of $m$, writing $m = q_\theta$, and minimize the following objective, defined up to an additive constant independent of $\theta$:
\vspace{-2mm}
\begin{multline}
    \label{equation:csbm_loss_forward}
    \mathcal{L}(\theta) = \mathbb{E}_{q(x_0,x_1)}\Bigg[\sum_{n=1}^{N}\mathbb{E}_{q^{\textup{ref}}(x_{t_{n-1}} | x_0, x_1)}    
    \Big[\textup{KL}\!\left(q^{\textup{ref}}(x_{t_{n}} | x_{t_{n-1}},x_1) \,\|\, q_\theta(x_{t_{n}} | x_{t_{n-1}})\right)\Big] 
    - \\[-2mm] - \mathbb{E}_{q^{\textup{ref}}(x_{t_N} | x_0, x_1)}\!\left[\log q_\theta(x_1 | x_{t_N})\right]\Bigg].
\end{multline}

\vspace{-3mm}
In practice, the D-IMF procedure is implemented in a bidirectional manner (see \citep[\S 3.2.5]{ksenofontov2025categorical}). That is, the forward and backward models are trained alternately at each Markovian step. Notably, the KL loss can be replaced by any divergence from the Bregman family (e.g., the mean squared error (MSE)), introducing an additional hyperparameter in our experimental setup. For details on this equivalence, see \citep[App.~C.1]{ksenofontov2025categorical}.

\textbf{Remark.} A continuous-time IMF was introduced in the Discrete Diffusion Schr\"odinger Bridge Matching \citep[DDSBM]{kim2024discrete} paper, which performs the Markovian projection \eqref{equation:markovian_projection} by matching the generator matrices of continuous-time Markov chains. As it reduces to the same loss and inference process due to the necessity to discretize time, we report results only for CSBM.

\vspace{-2mm}
\subsection{$\alpha$-Categorical Schr\"odinger Bridge Matching ($\alpha$-CSBM)}
\label{section:alpha_csbm}
\vspace{-1mm}
Running the IMF procedure bidirectionally is often beneficial, see \citep{kholkin2026diffusion}, but it doubles the computational burden, since two neural networks must be trained to model the forward and backward representations. To slightly mitigate this cost, recent work has proposed an online alternative to IMF, called $\alpha$-IMF \citep{de2024schr, peluchetti2025bm}.

In this approach, the exact projections in \eqref{equation:reciprocal_projection} and \eqref{equation:markovian_projection} are replaced by partial updates \citep[Eq.~9]{de2024schr}. Concretely, rather than running each projection to full convergence, one performs a single optimization step at each iteration $l$. Although each update is incomplete, the alternating steps still steer the learned distribution toward the double projection $\proj_{\reciprocal}(\proj_{\markov}(\cdot))$, as in IMF. Inspired by these advances in the continuous setting, we extend the same ideas to CSBM (\secref{section:csbm}), viewing the discrete version of $\alpha$-IMF as a heuristic analogue of the original procedure.

\textbf{Loss.} Since the method does not require each projection to fully converge, we can take a single optimization step for both representation directions simultaneously. This allows us to extend the CSBM bidirectional setup (\secref{section:csbm}) by jointly updating models using objective \eqref{equation:csbm_loss_forward} in both directions:

\vspace{-2mm}
\begin{equation}
    \mathcal{L}(\theta) = \tfrac{1}{2} \Big( 
        \KL{\overrightarrow{r_\mathrm{sg}}(x_0, x_{\text{in}}, x_1)}{\overleftarrow{q_\theta}(x_0, x_{\text{in}}, x_1)} + \KL{\overleftarrow{r_\mathrm{sg}}(x_0, x_{\text{in}}, x_1)}{\overrightarrow{q_\theta}(x_0, x_{\text{in}}, x_1)} 
    \Big),
\end{equation}

where $\rightarrow$ and $\leftarrow$ denote the direction of the Markov representation (typically implemented by conditioning a neural network on a direction variable), and $r_\mathrm{sg}$ denotes $\proj_{\reciprocal}(q_\theta)$ with stop-gradient.

\textbf{Limitation.} To avoid evaluating the full space of size $S^D$, transition probabilities $q_\theta(x_{t_n} | x_{t_{n-1}})$ are factorized across dimensions, reducing it to $D\times S$. However, this parametrization constitutes a key limitation of ($\alpha$-)CSBM, as it introduces approximation error.

\vspace{-2mm}
\subsection{Discrete Light Schr\"odinger Bridge (DLightSB)}
\label{section:dlight_sb}
\vspace{-1mm}
Now we take a different direction and propose new solvers for discrete-space EOT/SB problems. Specifically, we focus on methods that arise directly as a byproduct of our benchmark construction. We begin by introducing a new static SB solver, which we entitle DLightSB. In what follows, we adopt the same CP benchmark parameterization (\wasyparagraph\ref{subsection:benchmark_parameterization}) for the conditional distribution $q(x_1 | x_0)$. In particular, we treat the weights $\beta_k$ and the CP cores $r_k^d$ as learnable parameters and collect them in $\theta=\{\beta_k, r_k^d\}$, yielding the model $q_\theta(x_1 | x_0)$ with scalar-valued function $v_\theta(x_1)$.

\textbf{Loss.} To optimize the parameters $\theta$, we follow the approach of \citet{korotin2024light}. Concretely, we derive a discrete objective that is equivalent to the direct KL objective $\KL{q^*}{q_{\theta}}$ up to an additive constant. Importantly, this reformulation removes the dependence on the unknown optimal joint distribution $q^*$ and enables a feasible optimization procedure, as shown in the following proposition.
\begin{proposition}[Feasible Discrete Reformulation of the Direct KL Objective]
Under the parametrization \eqref{equation:conditional_closed_form} of $q_{\theta}(x_1 | x_0)$, the objective $\KL{q^*}{q_\theta}$ admits the following reformulation:

\vspace{-2mm}
    \begin{equation*}
        \KL{q^{*}}{q_{\theta}}=\mathcal{L}(\theta)-\mathcal{L}^*,\ \text{where}
    \end{equation*}
    \begin{equation}
        \mathcal{L}(\theta)=\mathbb{E}_{p_0(x_0)}\big[\log c_{\theta}(x_0)\big] -\mathbb{E}_{p_1(x_1)}\big[\log v_{\theta}(x_1)\big],
        \label{equation:dlight_sb_loss}
    \end{equation}
    
    and $\mathcal{L^*}\in \mathbb{R}$ is a constant value not depending on $\theta$, therefore, it can be omitted.
    \label{proposition:dlight_sb_loss}
\end{proposition}

Notably, the expectations in Proposition \ref{proposition:dlight_sb_loss} can be efficiently estimated via Monte Carlo sampling, and the resulting objective can be optimized using stochastic gradient descent with respect to $\theta$.

\vspace{-2mm}
\subsection{Discrete Light Schr\"odinger Bridge Matching (DLightSB-M)}
\label{section:dlight_sb_m}

\vspace{-2mm}
Finally, we introduce DLightSB-M, a dynamic variant of DLightSB, which also uses our benchmark construction. Inspired by \citep{gushchin2024light}, we introduce a \emph{discrete-space optimal projection} that recovers the SB through a single projection step. In particular, rather than projecting a reciprocal process $r \in \reciprocal(\X^{N+2})$ onto the Markov set $\markov(\X^{N+2})$ as in the D-IMF procedure \eqref{equation:markovian_projection}, we directly project $r$ onto the set of all SBs, defined as follows:

\vspace{-3mm}
\begin{gather}
\label{equation:sb_set}
    \gS(\X^{N+2}) \coloneqq 
    \Big\{
        q^{\mathrm{SB}} \in \process(\X^{N+2}) \ \text{such that} \ 
        \exists\, q^{\mathrm{SB}}_0, q^{\mathrm{SB}}_1 \in \process(\X),\nonumber \\
        q^{\mathrm{SB}} \!=\!
        \argmin_{q \in \Pi_N(q^{\mathrm{SB}}_0, q^{\mathrm{SB}}_1)}
        \!\KL{q}{q^{\mathrm{ref}}}
    \Big\}.
\end{gather}
\vspace{-3mm}

%\begin{gather}
%\label{equation:sb_set}
%    \gS(\X^{N+2}) \coloneqq 
%    \Big\{
%        q^{\mathrm{SB}} \in \process(\X^{N+2}) \ \text{such that} \ 
%        \exists\, q^{\mathrm{SB}}_0, q^{\mathrm{SB}}_1 \in \process(\X) \nonumber \\
%        q^{\mathrm{SB}} =
%        \argmin_{q \in \Pi_N(q^{\mathrm{SB}}_0, q^{\mathrm{SB}}_1)}
%        \KL{q}{q^{\mathrm{ref}}}
%    \Big\},
%\end{gather}
To justify this projection in discrete settings, we show that \citep[Theorem 3.1]{gushchin2024light} extends to an arbitrary Markov reference process $q^{\text{ref}}$.

\begin{proposition}[Discrete-Space Optimal Projection with an Arbitrary Reference Process]
    \label{proposition:optimal_projection}
    Let $r \in \reciprocal(\X^{N+2})$ be a reciprocal process defined with a reference process $q^{\textup{ref}}\in\markov(\X^{N+2})$ and a joint distribution $r(x_0, x_1) \in \Pi(p_0, p_1)$. Then, the optimal projection of $r$ onto the set of all SBs $\gS(\X^{N+2})$ is the SB $q^*$ between the desired marginals $p_0$ and $p_1$, i.e.,
    
    \vspace{-2mm}
    \begin{equation}
        \label{equation:optimal_projection}
        q^* = \argmin_{q^{\textup{SB}} \in \gS(\X^{N+2})} \KL{r}{q^{\textup{SB}}}.
    \end{equation}
\end{proposition}

\vspace{-2mm}
\textbf{Loss.} To complete the construction, we need to parameterize $q^{\text{SB}}$ so that the minimization is restricted to $\mathcal{S}(\mathcal{X}^{N+2})$. To address this we recall that $q^{\text{SB}}\!\in\!\markov(\X^{N+2})$ (\wasyparagraph\ref{section:csbm}) and it is fully determined by its forward transitions $q(x_{t_n} | x_{t_{n-1}})$. We therefore adopt the closed-form transition parameterization \eqref{equation:transition_closed_form}, with parameters $\theta\!=\!\{\beta_k, r_k^d\}$ as in DLightSB, yielding the transitions $q_\theta(x_{t_n} | x_{t_{n-1}})$ that define an SB. Finally, under this parameterization, the discrete-space optimal projection \eqref{equation:optimal_projection} can be optimized using the Markovian projection objective \eqref{equation:csbm_loss_forward} via stochastic gradient descent.

Importantly, since the reference process $q^{\text{ref}}$ factorizes across dimensions, the KL and cross-entropy (CE) terms in \eqref{equation:csbm_loss_forward} admit an equivalent reformulation as a sum over dimensions. Concretely, each term reduces to a sum over dimensions of KL or CE between $q^{\text{ref}}(x^d_{t_n} | x^d_{t_{n-1}}, x^d_1)$ and the model's $d$-th transition. The latter is obtained by summing out all dimensions except the $d$-th one (see \eqref{equation:marginal_sb_trainsition_distribution} in the proof of Proposition~\ref{proposition:sb_trainsition_distribution}), which is tractable under our parameterization.

\textbf{Limitation.} In high dimensions, the CP parameterization requires a large number of components $K$ to capture complex structure. This makes DLightSB(-M) methods computationally demanding.

\vspace{-3mm}
\section{Evaluation}
\label{section:evaluation}
\vspace{-2mm}
In this section, we first introduce the evaluation metrics and describe the baselines (\secref{subsection:evaluation_metrics}). We then evaluate the solvers from \wasyparagraph\ref{section:solvers_for_evaluation} on our benchmark setup (\secref{subsection:benchmark_pairs_construction}) and report the results in \secref{subsection:results}.

\vspace{-2mm}
\subsection{Metrics and Baselines}
\label{subsection:evaluation_metrics}
\vspace{-1mm}

Evaluating generative models on discrete data is challenging because widely used metrics, such as generative perplexity for text or FID for images \citep{heusel2017gans}, are domain-specific. Following tabular-data evaluation work \citep{zhang2024mixedtypetabulardatasynthesis, shi2025tabdiff}, we adopt the \emph{Shape Score} and \emph{Trend Score} metrics. Further, we introduce an EOT/SB-specific metric given by the \emph{Trajectory KL}.

\textbf{Shape and Trend Score.} These metrics quantify how well a solver matches the benchmark at the level of dimension-wise and pair-wise marginals. For each coordinate $d$ (Shape) and each pair $(d_m,d_n)$ (Trend), we compute an adjusted total variation score and then average over dimensions:
\begin{minipage}{0.43\linewidth}
    \vspace{-2mm}
    \begin{gather*}
        \text{SSM}^d = \\[-2mm] 1 - \frac{1}{2}\sum_{x^d=0}^{S-1} | \tilde{q}^*(x^d) - \tilde{q}_\theta(x^d) |;
    \end{gather*}
\end{minipage}
\hfill
\begin{minipage}{0.52\linewidth} 
    \vspace{-2mm}
    \begin{gather*}
        \text{TSM}^{d_m,d_n} = \\[-2mm] 1\!-\!\frac{1}{2}\sum_{x^{d_m}=0}^{S-1}\sum_{x^{d_n}=0}^{S-1} | \tilde{q}^*(x^{d_m}, x^{d_n}) - \tilde{q}_\theta(x^{d_m}, x^{d_n}) |,
    \end{gather*}
\end{minipage}

where the tilde denotes empirical distribution, i.e., $\tilde{q}(x)\!=\!\tfrac{1}{|I|}\!\sum_{i \in I} \delta_{x^{(i)}}$.

\textbf{Conditional Metrics.} In our evaluation, we primarily report conditional variants of these metrics. These are computed by averaging SSM and TSM over sampled $x_0 \sim p_0$ and multiple generated $x_1$ for each $x_0$. This approach provides a direct measure of the fidelity of the learned conditional distribution $q_\theta(x_1|x_0)$ and quantifies how well the EOT/SB solver solves the underlying problem.

\textbf{Trajectory KL divergence.} Additionally, we utilize the dynamic SB and its Markov property to compute the forward and reverse KL divergences between processes, namely $\KL{q^*(x_0, x_\text{in}, x_1)}{q_\theta(x_0, x_\text{in}, x_1)}$ and $\KL{q_\theta(x_0, x_\text{in}, x_1)}{q^*(x_0, x_\text{in}, x_1)}$, respectively, i.e.,
\begin{equation*}
    \KL{q^*(x_0, x_\text{in}, x_1)}{q_\theta(x_0, x_\text{in}, x_1)}
    =
    \sum_{n=1}^{N+1}\sum_{d=1}^D
    \KL{q^*(x^d_{t_n} | x_{t_{n-1}})}{q_\theta(x^d_{t_n} | x_{t_{n-1}})}.
\end{equation*}
An analogous decomposition holds for the reverse KL divergence. This metric relies on factorization across dimensions specifically for methods that model only marginals (e.g., CSBM, \secref{section:csbm}). In contrast, for methods that do not impose such a restriction (e.g., DLightSB, \secref{section:dlight_sb}), the required transitions are obtained by marginalizing out all coordinates except the $d$-th dimension, see \eqref{equation:marginal_sb_trainsition_distribution}.

\textbf{Baselines.} To provide a simple point of reference for our benchmark, we consider three simple baselines. The first, \textit{Independent}, ignores the joint distribution between $x_0$ and $x_1$ and samples $x_1$ directly from $p_1$. The second, \textit{Reference}, takes $x_1$ to be a sample from the reference process $q^{\text{ref}}$. The third, \textit{Feature-wise SB}, solves an SB problem separately for each dimension using the analytical D-IMF procedure \citep[Section 4.1]{ksenofontov2025categorical} with empirically estimated marginals $p_0$ and $p_1$. This baseline assumes factorization across dimensions and samples each $x_1^d$ independently.

\begin{figure}[!t]
   \centering
   \captionsetup[subfigure]{font=small}
   % --- first row ---
   \begin{minipage}{0.02\textwidth}
      \centering
       \vspace{-20mm}
       \rotatebox{90}{\small $q^{\text{gauss}}$/$\gamma\!=\!0.02$}
  \end{minipage}
   \begin{subfigure}[b]{0.156\linewidth}
       \includegraphics[width=0.995\linewidth]{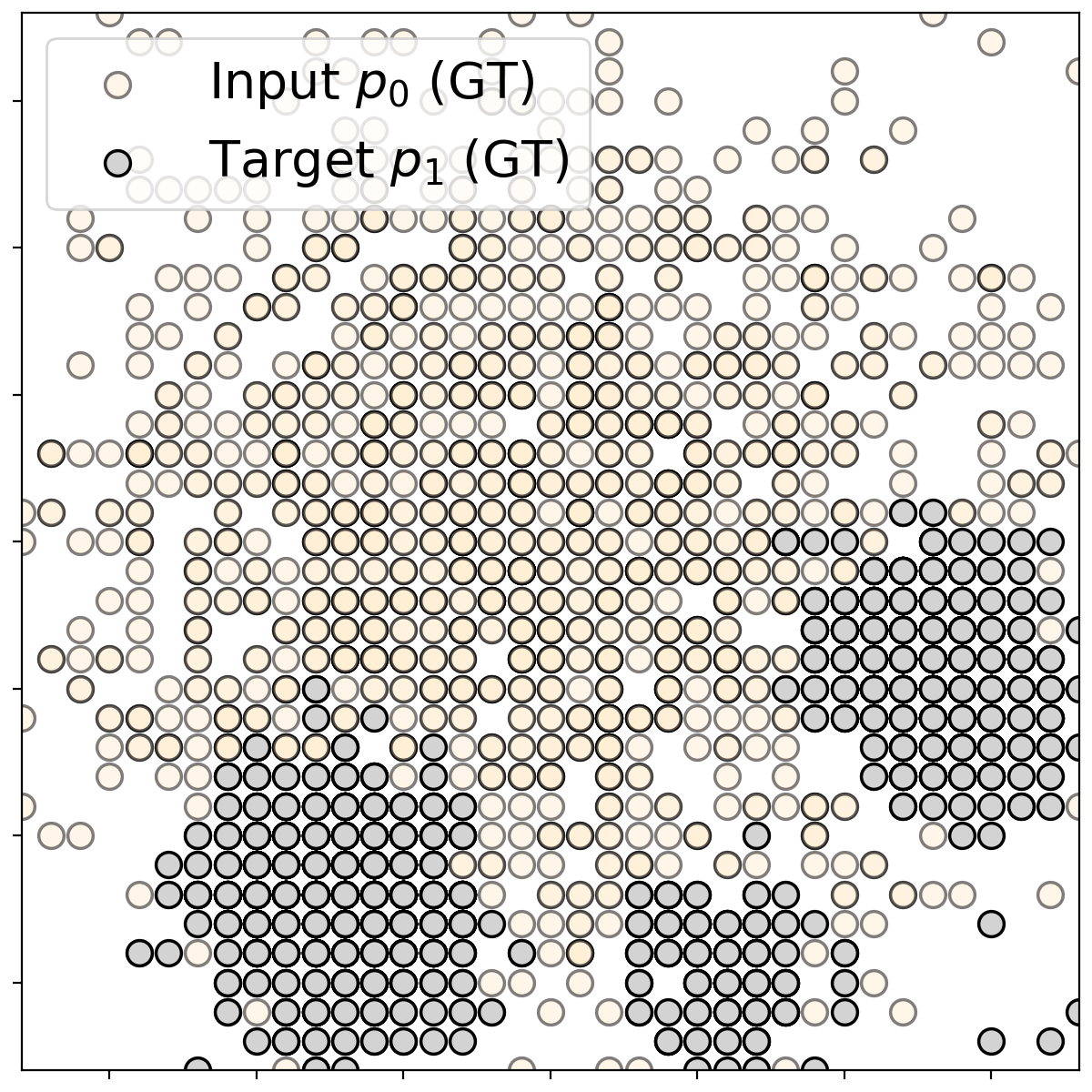}
       % \caption{Input/Target}
   \end{subfigure}
   \begin{subfigure}[b]{0.156\linewidth}
       \includegraphics[width=0.995\linewidth]{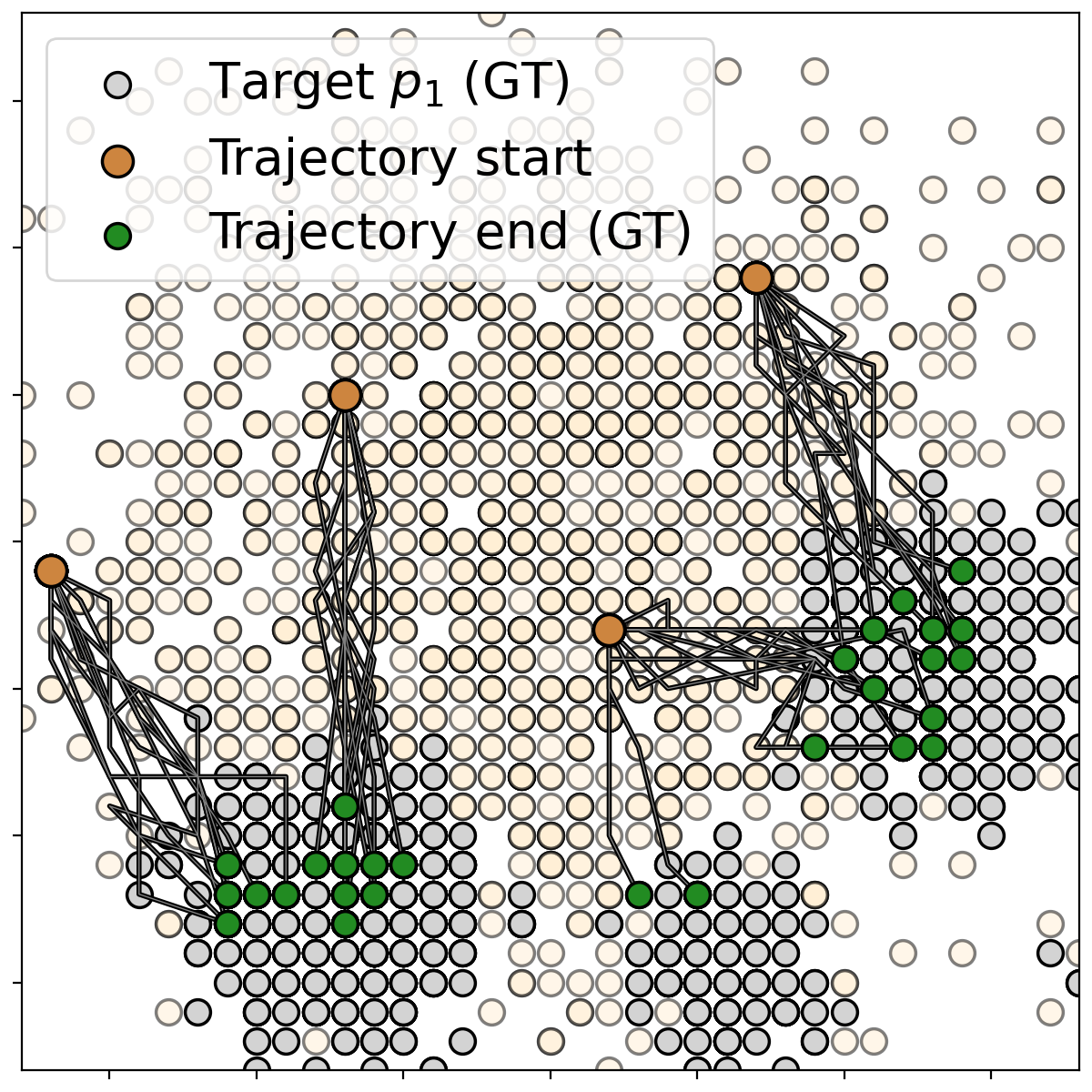}
       % \caption{Benchmark}
   \end{subfigure}
   \begin{subfigure}[b]{0.156\linewidth}
       \includegraphics[width=0.995\linewidth]{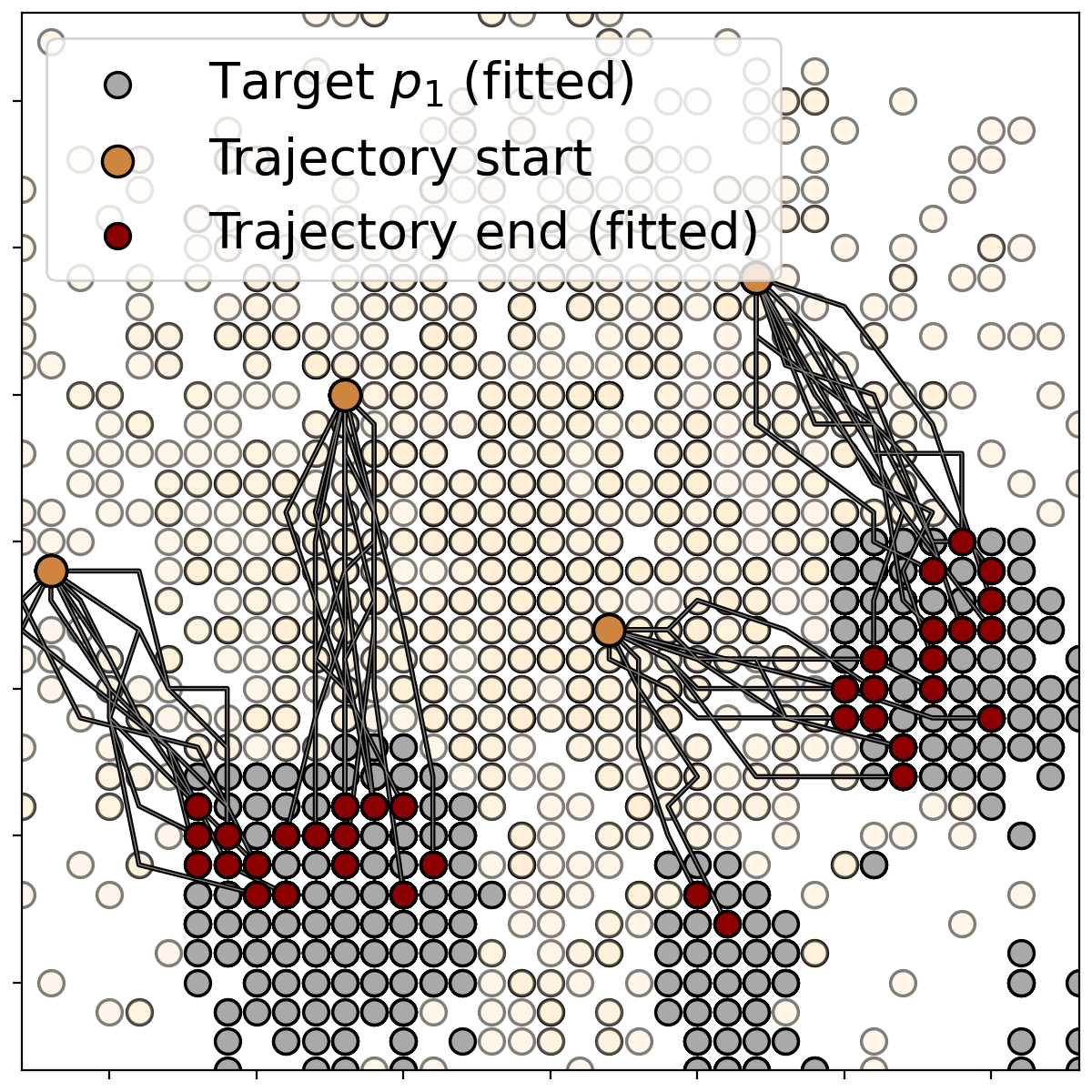}
       % \caption{CSBM}
   \end{subfigure}
   \begin{subfigure}[b]{0.156\linewidth}
       \includegraphics[width=0.995\linewidth]{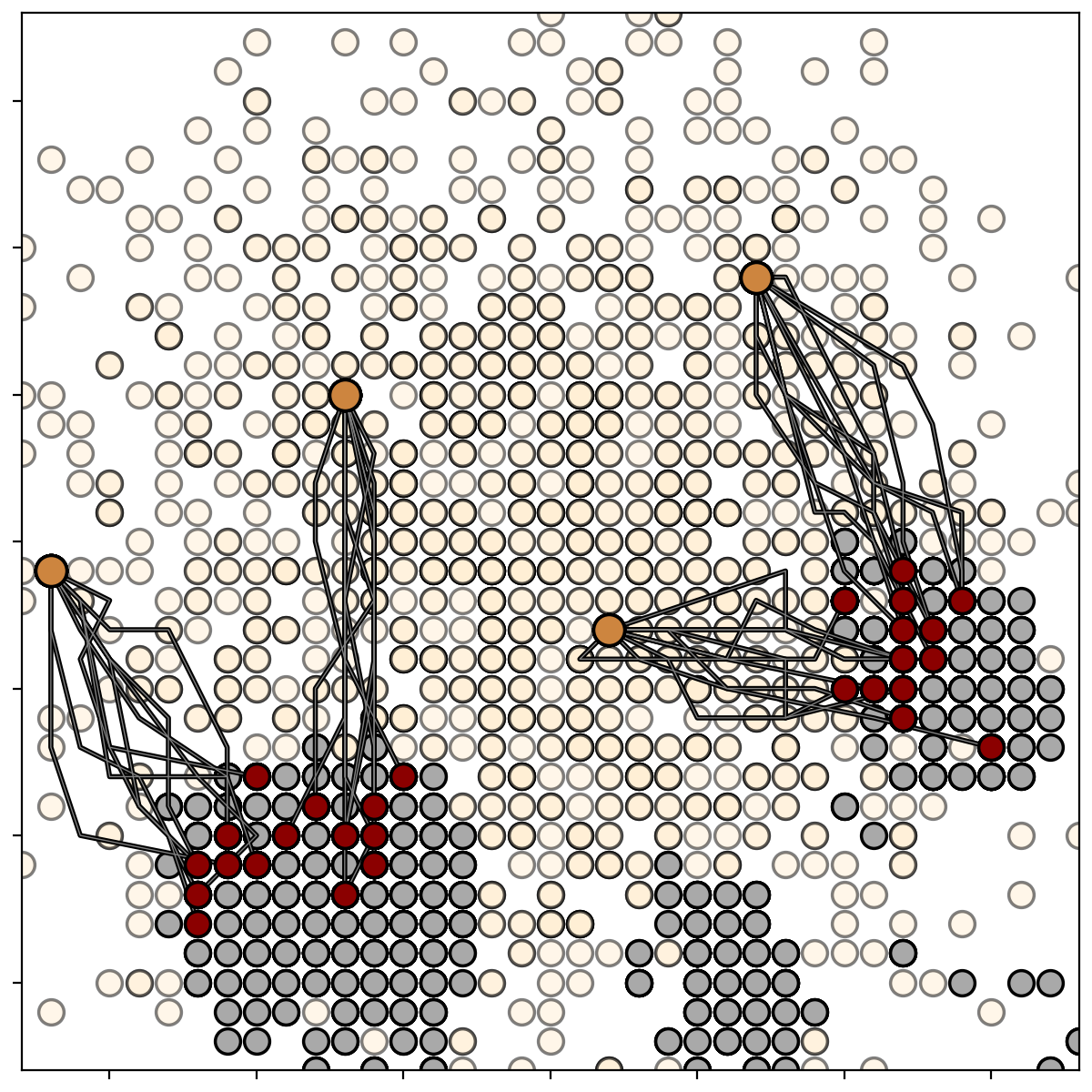}
       % \caption{$\alpha$-CSBM}
   \end{subfigure}
   \begin{subfigure}[b]{0.156\linewidth}
       \includegraphics[width=0.995\linewidth]{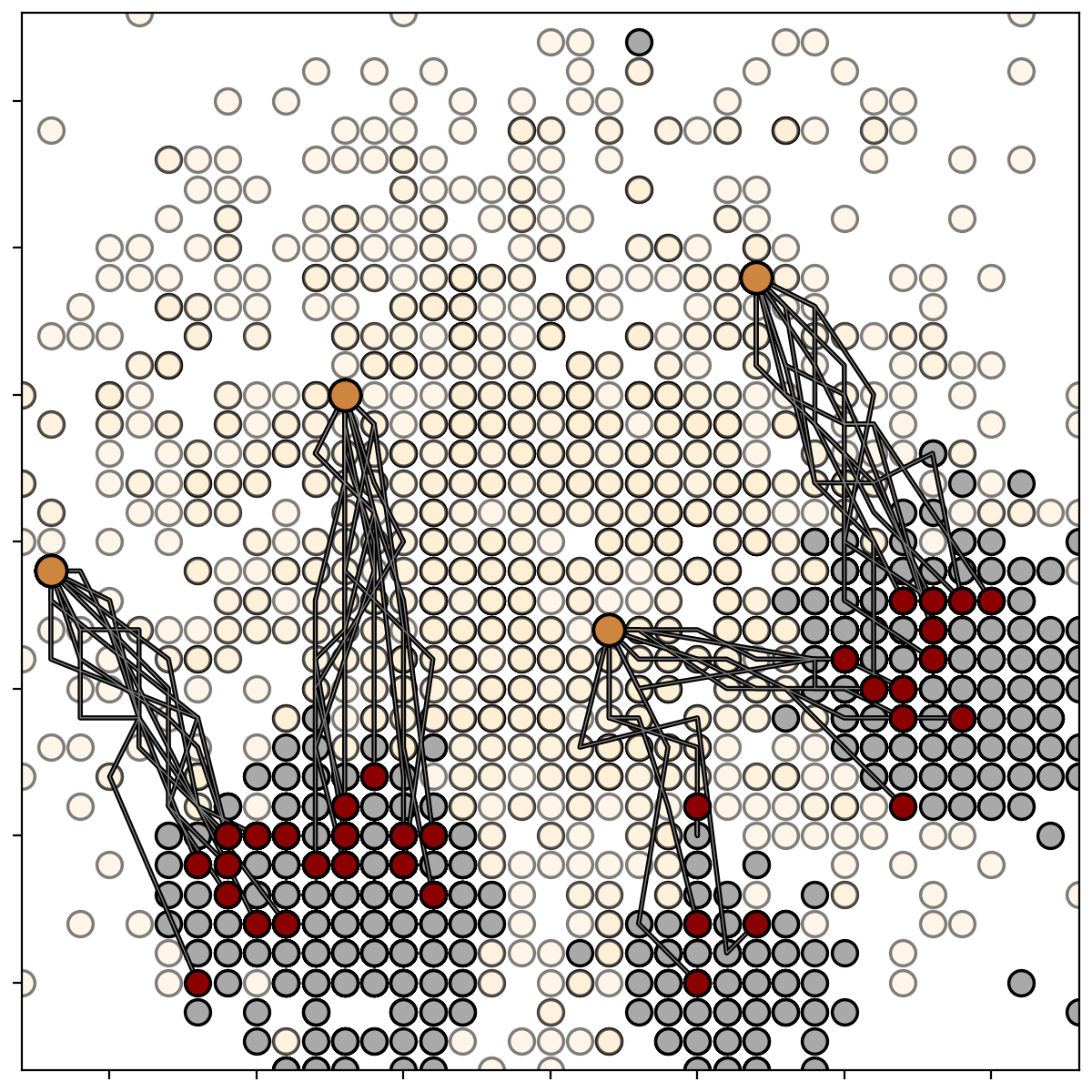}
       % \caption{DLightSB}
   \end{subfigure}
   \begin{subfigure}[b]{0.156\linewidth}
       \includegraphics[width=0.995\linewidth]{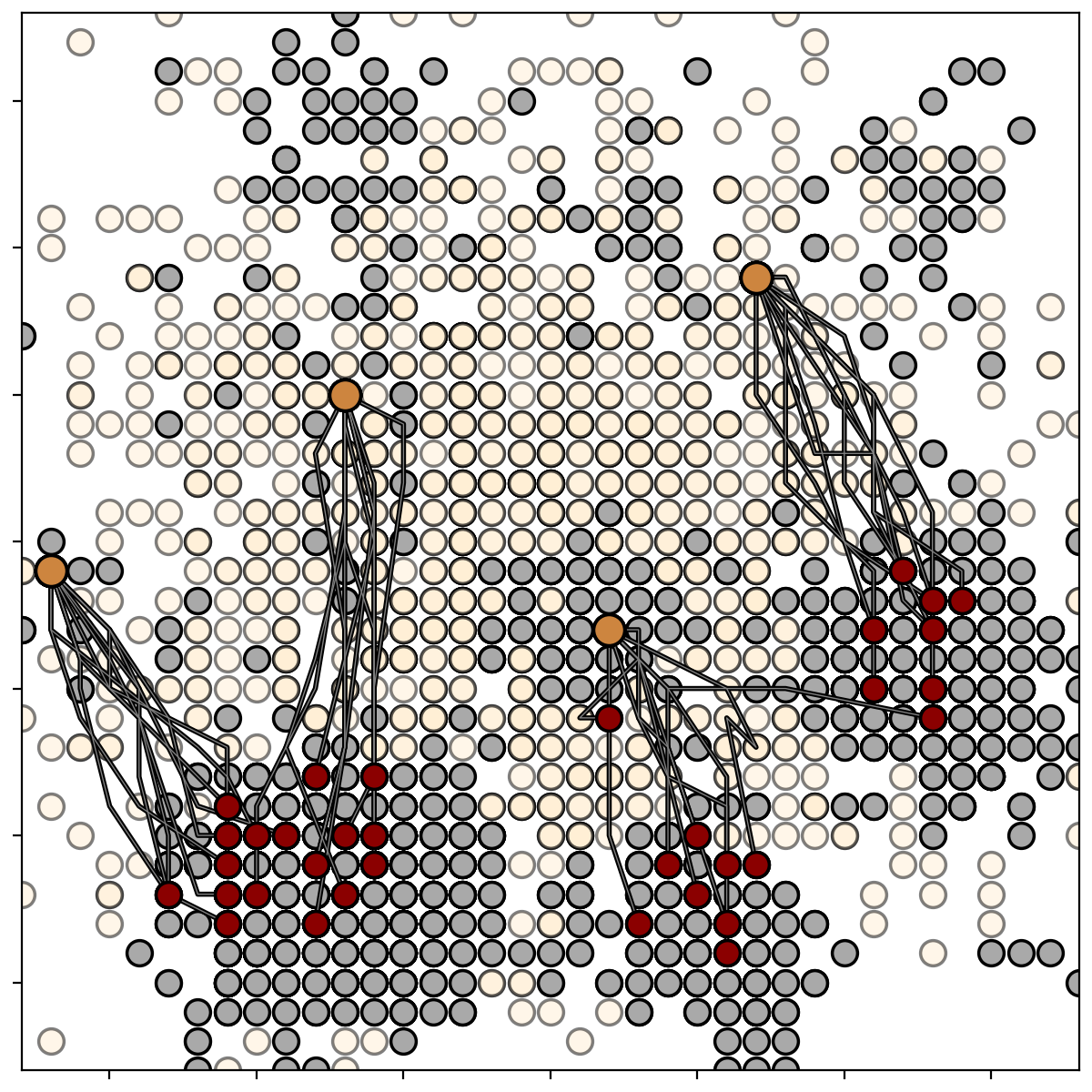}
       % \caption{DLightSB-M}
   \end{subfigure}
   \vspace{1mm} % space between rows

   % --- second row ---
    \begin{minipage}{0.02\textwidth}
      \centering
       \vspace{-32mm}
       \rotatebox{90}{\small $q^{\text{unif}}$/$\gamma\!=\!0.005$}
  \end{minipage}
   \begin{subfigure}[b]{0.156\linewidth}
       \includegraphics[width=0.995\linewidth]{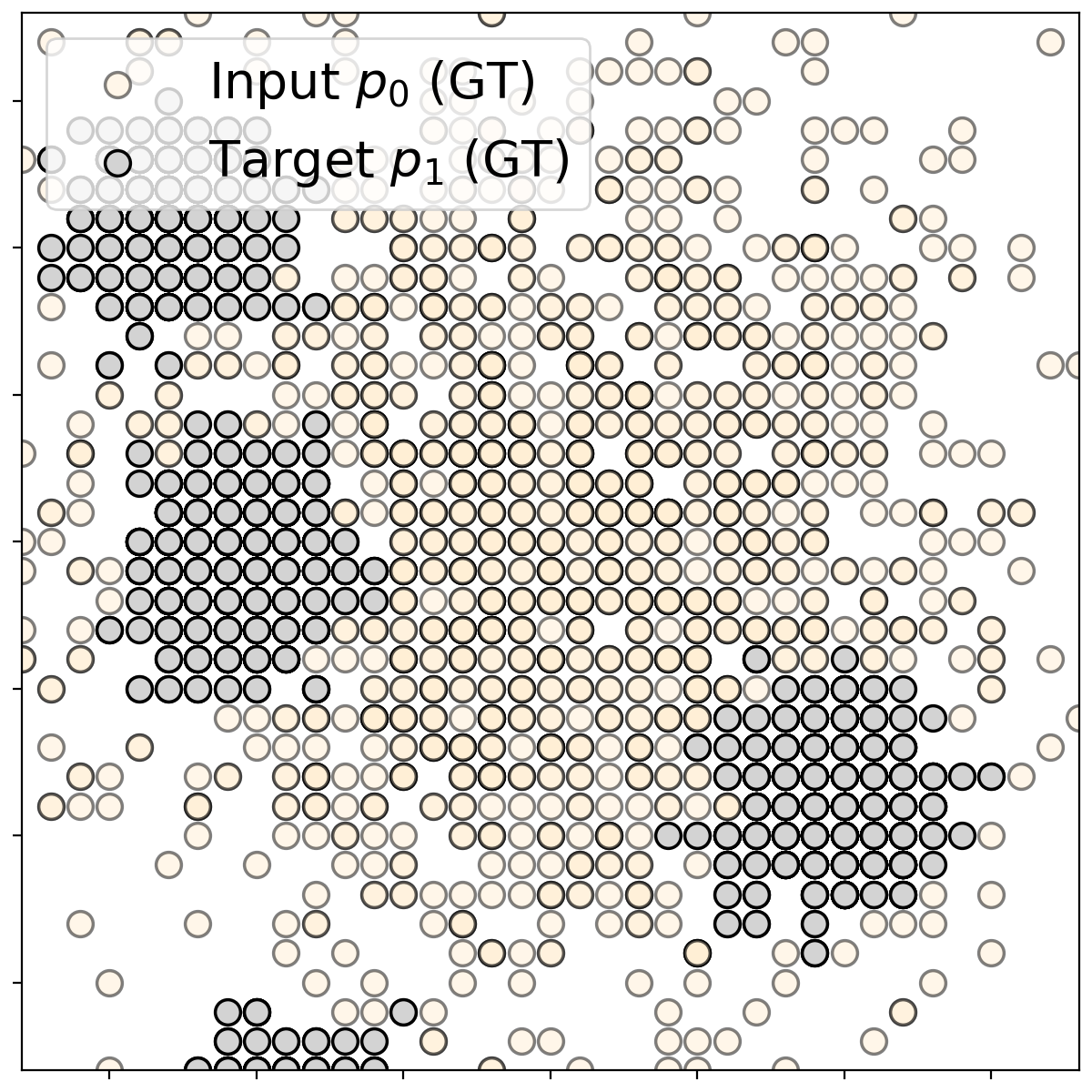}
       \caption{Input/Target}
   \end{subfigure}
   \begin{subfigure}[b]{0.156\linewidth}
       \includegraphics[width=0.995\linewidth]{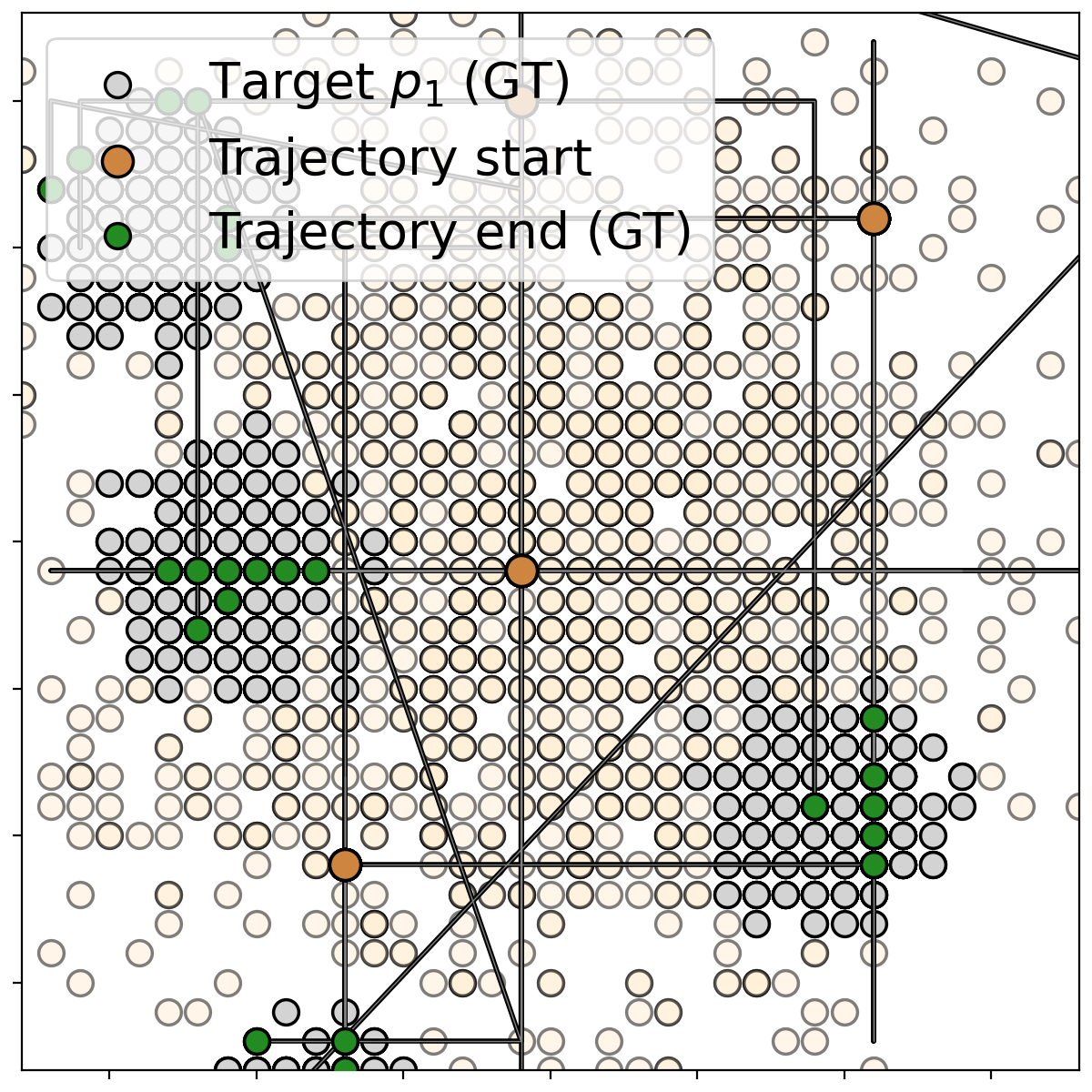}
       \caption{Benchmark}
       \label{subfigure:benchmark pairs}
   \end{subfigure}
   \begin{subfigure}[b]{0.156\linewidth}
       \includegraphics[width=0.995\linewidth]{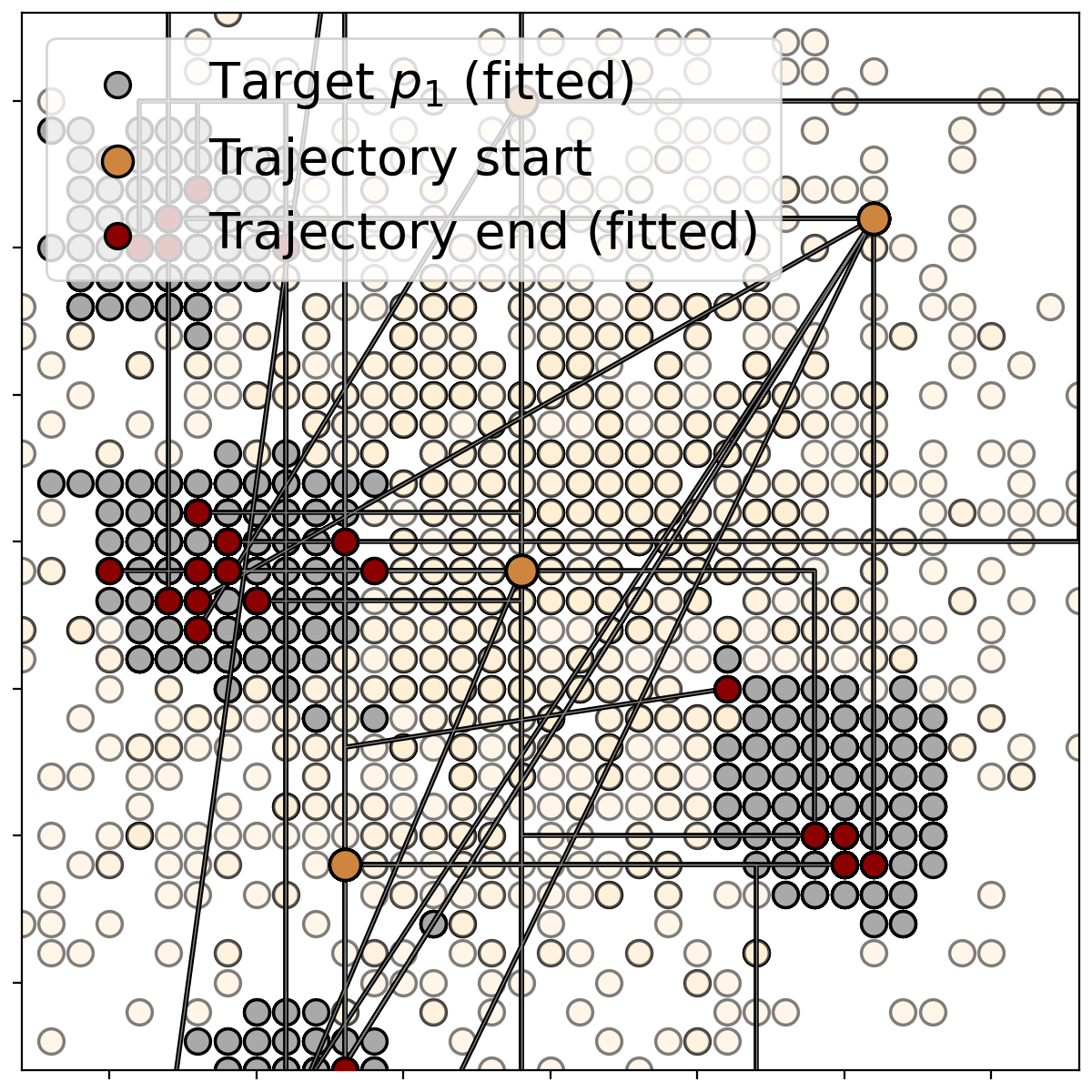}
       \caption{CSBM}
   \end{subfigure}
   \begin{subfigure}[b]{0.156\linewidth}
       \includegraphics[width=0.995\linewidth]{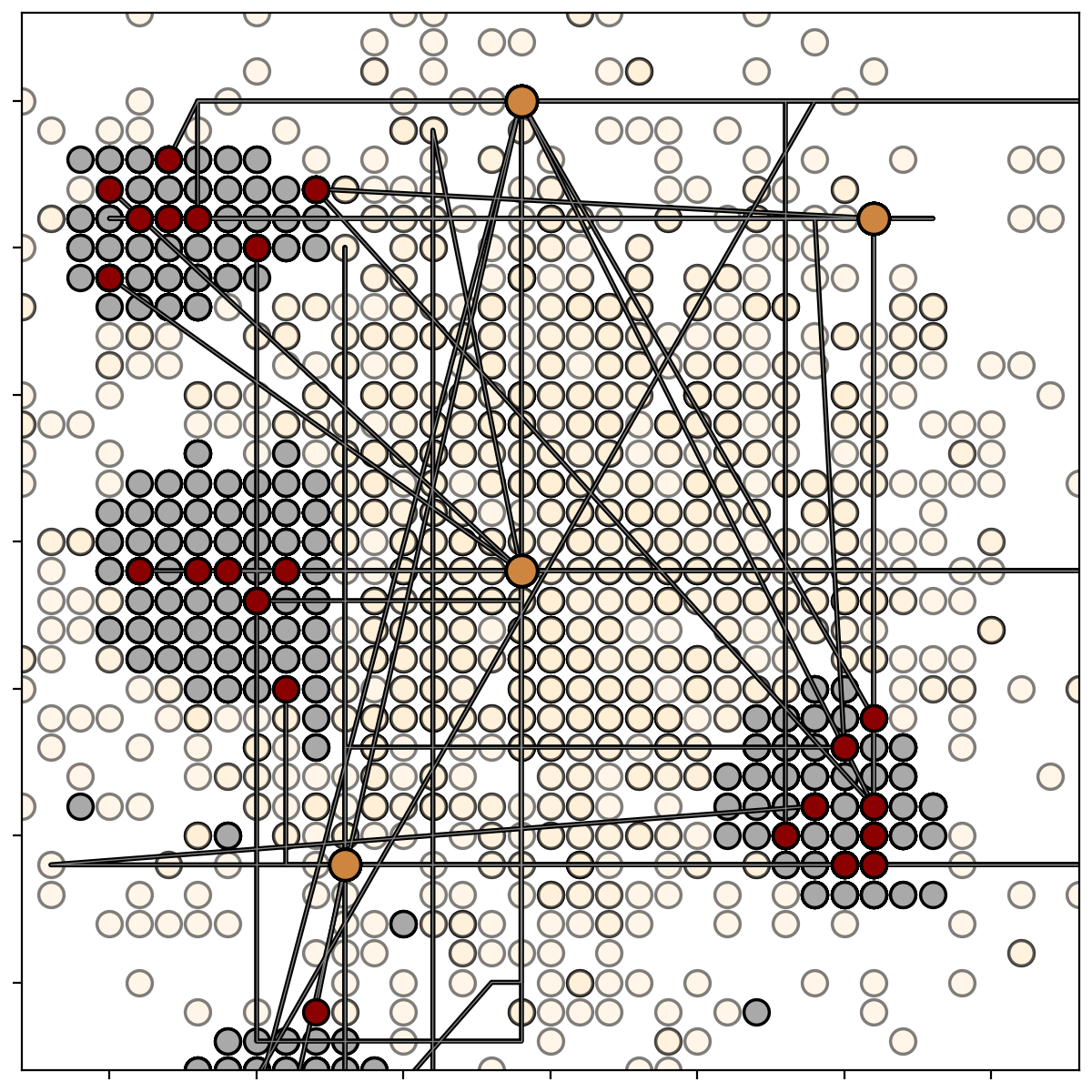}
       \caption{$\alpha$-CSBM}
   \end{subfigure}
   \begin{subfigure}[b]{0.156\linewidth}
       \includegraphics[width=0.995\linewidth]{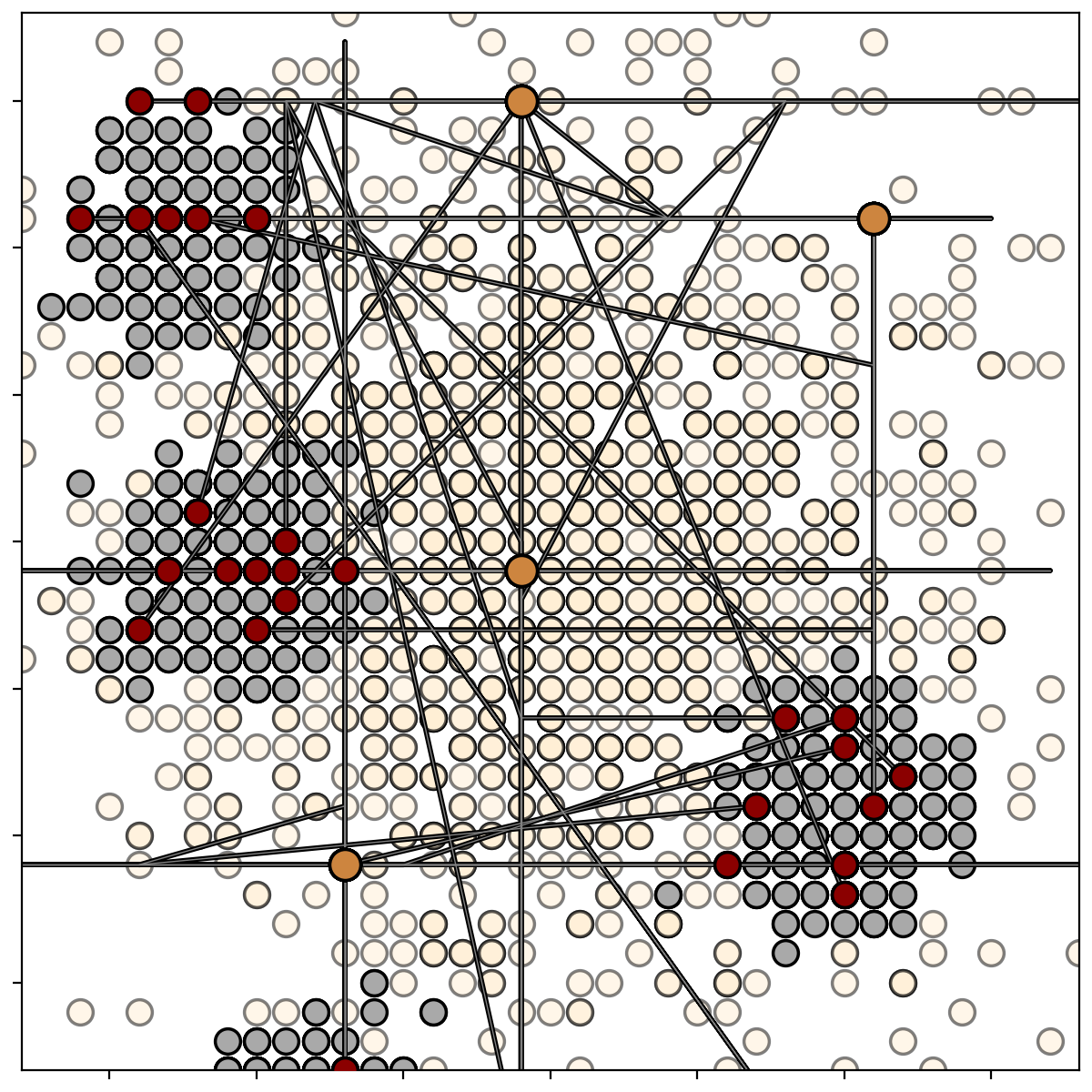}
       \caption{DLightSB}
   \end{subfigure}
   \begin{subfigure}[b]{0.156\linewidth}
       \includegraphics[width=0.995\linewidth]{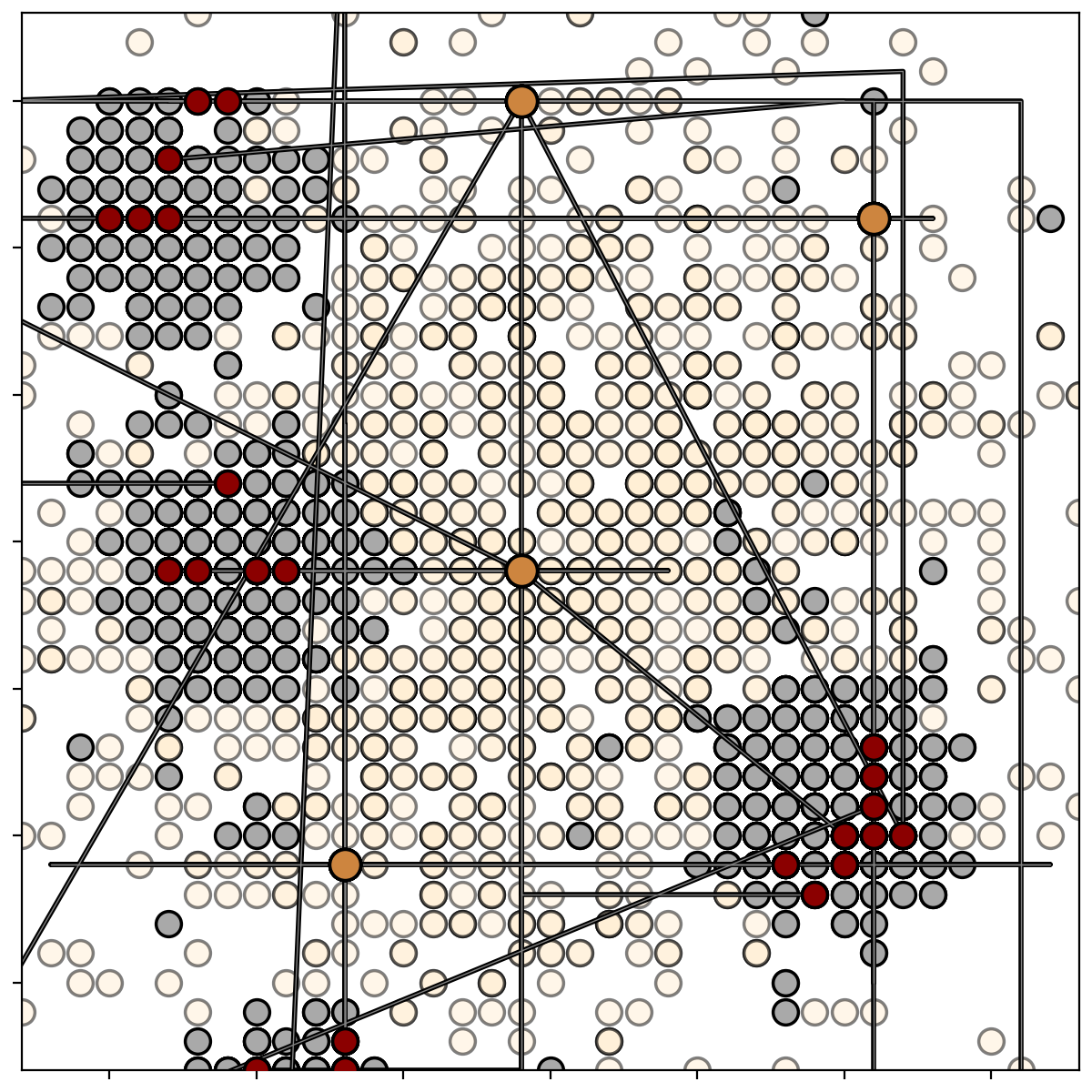}
       \caption{DLightSB-M}
   \end{subfigure}

   \vspace{-3mm}
   
   \caption{\centering Samples from all considered methods on our high-dimensional Gaussian mixture benchmark. 
   \textbf{Top row}: $q^{\text{unif}}$ ($\gamma = 0.005$). 
   \textbf{Bottom row}: Gaussian benchmark ($\gamma = 0.02$). 
   CSBM, $\alpha$-CSBM, and DLightSB-M were trained with KL loss ($N\!+\!1=64$).}
   \label{figure:hd_benchmarks_main}
   \vspace{-4mm}
\end{figure}

\vspace{-2mm}
\subsection{Results}
\label{subsection:results}
\vspace{-2mm}
\textbf{Benchmark Usage Protocol.} To simulate the generative discrete-space EOT/SB task discussed in \secref{subsection:background_problem_setting}, we use our benchmark pair constructor differently for training and testing. For \emph{training}, we sample $x_0^\textup{train} \sim p_0$ and $x_1^\textup{train} \sim p_1$ independently. We further remove dataset-size limitations that could affect a direct comparison of methods by infinitely sampling the training data, as enabled by our benchmark. For \emph{testing}, unconditional metrics and the Trajectory KL are computed on a fixed set of $20\ 000$ precomputed benchmark pairs $(x_0,x_1) \sim q^*(x_0, x_1)$. Conditional metrics are evaluated on $1\ 000$ $x_1$ samples for each of $157$ unique $x_0$, where the $x_0$ are drawn from this $20\ 000$ set and the corresponding $x_1$ are sampled from both the method and the benchmark. We also compare different training setups by varying $N$ across CSBM, $\alpha$-CSBM, and DLightSB-M. For the same set of methods, we experiment with two loss functions: KL and MSE.

\begin{table*}[!t]
    \centering
    \makebox[0.5\paperwidth][c]{%
    \begin{subtable}[t]{0.459\paperwidth}
        \centering
        \begin{adjustbox}{width=\linewidth}
        \begin{tabular}{rcccccccccccccc}
            \toprule
            & & & \multicolumn{4}{c}{$D\!=\!2$} & \multicolumn{4}{c}{$D\!=\!16$} & \multicolumn{4}{c}{$D\!=\!64$} \\
            \cmidrule(lr){4-7} \cmidrule(lr){8-11} \cmidrule(l){12-15}
            & & & \multicolumn{2}{c}{gaussian} & \multicolumn{2}{c}{uniform} & \multicolumn{2}{c}{gaussian} & \multicolumn{2}{c}{uniform} & \multicolumn{2}{c}{gaussian} & \multicolumn{2}{c}{uniform} \\
            \cmidrule(lr){4-5} \cmidrule(lr){6-7} \cmidrule(l){8-9} \cmidrule(lr){10-11} \cmidrule(lr){12-13} \cmidrule(l){14-15}
            Method & Loss & $N\!+\!1$ & {$0.02$} & {$0.05$} & {$0.005$} & {$0.01$} & {$0.02$} & {$0.05$} & {$0.005$} & {$0.01$} & {$0.02$} & {$0.05$} & {$0.005$} & {$0.01$} \\
            \midrule
            \multirow{1}{*}{\textit{Independent}} & -- & -- &  \scorecell{0.507} & \scorecell{0.828} & \scorecell{0.633} & \scorecell{0.642} & \scorecell{0.700} & \scorecell{0.742} & \scorecell{0.498} & \scorecell{0.571} & \scorecell{0.653} & \scorecell{0.659} & \scorecell{0.543} & \scorecell{0.611} \\
            \multirow{1}{*}{\textit{Reference}} & -- & -- &  \scorecell{0.167} & \scorecell{0.453} & \scorecell{0.393} & \scorecell{0.419} & \scorecell{0.294} & \scorecell{0.291} & \scorecell{0.296} & \scorecell{0.329} & \scorecell{0.406} & \scorecell{0.371} & \scorecell{0.358} & \scorecell{0.348} \\
            \multirow{1}{*}{\textit{Feature-wise SB}} & -- & -- &  \scorecell{0.712} & \scorecell{0.873} & \scorecell{0.875} & \scorecell{0.890} & \scorecell{0.808} & \scorecell{0.756} & \scorecell{0.612} & \scorecell{0.665} & \scorecell[\underline]{0.931} & \scorecell{0.661} & \scorecell{0.734} & \scorecell{0.701} \\
            \midrule
            \multirow{1}{*}{DLightSB} & \multirow{1}{*}{--} & -- & \scorecell[\textbf]{0.948} & \scorecell[\textbf]{0.917} & \scorecell[\textbf]{0.931} & \scorecell[\textbf]{0.933} & \scorecell{0.849} & \scorecell[\textbf]{0.934} & \scorecell[\textbf]{0.919} & \scorecell[\textbf]{0.933} & \scorecell[\textbf]{0.943} & \scorecell[\textbf]{0.907} & \scorecell[\textbf]{0.931} & \scorecell[\textbf]{0.923} \\
            \midrule
            \multirow{4}{*}{CSBM} & \multirow{2}{*}{KL} & 16 & \scorecell{0.719} & \scorecell{0.713} & \scorecell{0.872} & \scorecell{0.884} & \scorecell{0.786} & \scorecell{0.744} & \scorecell{0.643} & \scorecell{0.647} & \scorecell{0.858} & \scorecell{0.849} & \scorecell{0.654} & \scorecell{0.684} \\
             &  & 64 & \scorecell[\underline]{0.877} & \scorecell{0.866} & \scorecell{0.904} & \scorecell[\underline]{0.916} & \scorecell[\underline]{0.870} & \scorecell{0.822} & \scorecell{0.667} & \scorecell{0.711} & \scorecell{0.888} & \scorecell{0.852} & \scorecell{0.624} & \scorecell{0.681} \\
             & \multirow{2}{*}{MSE} & 16 & \scorecell{0.481} & \scorecell{0.712} & \scorecell{0.777} & \scorecell{0.816} & \scorecell{0.772} & \scorecell{0.707} & \scorecell{0.522} & \scorecell{0.546} & \scorecell{0.734} & \scorecell{0.804} & \scorecell{0.592} & \scorecell{0.615} \\
             &  & 64 & \scorecell{0.285} & \scorecell{0.822} & \scorecell{0.766} & \scorecell{0.765} & \scorecell{0.815} & \scorecell{0.810} & \scorecell{0.589} & \scorecell{0.643} & \scorecell{0.660} & \scorecell[\underline]{0.870} & \scorecell{0.622} & \scorecell{0.631} \\
            \midrule
            \multirow{4}{*}{$\alpha$-CSBM} & \multirow{2}{*}{KL} & 16 & \scorecell{0.663} & \scorecell{0.729} & \scorecell{0.894} & \scorecell{0.889} & \scorecell{0.722} & \scorecell{0.789} & \scorecell{0.662} & \scorecell{0.677} & \scorecell{0.872} & \scorecell{0.810} & \scorecell{0.626} & \scorecell{0.667} \\
             &  & 64 & \scorecell{0.770} & \scorecell{0.870} & \scorecell{0.897} & \scorecell{0.914} & \scorecell[\textbf]{0.886} & \scorecell{0.883} & \scorecell{0.659} & \scorecell{0.727} & \scorecell{0.892} & \scorecell{0.829} & \scorecell{0.627} & \scorecell{0.662} \\
             & \multirow{2}{*}{MSE} & 16 & \scorecell{0.527} & \scorecell{0.708} & \scorecell{0.846} & \scorecell{0.870} & \scorecell{0.715} & \scorecell{0.779} & \scorecell{0.571} & \scorecell{0.627} & \scorecell{0.740} & \scorecell{0.768} & \scorecell{0.603} & \scorecell{0.622} \\
             &  & 64 & \scorecell[\underline]{0.880} & \scorecell{0.858} & \scorecell{0.836} & \scorecell{0.876} & \scorecell{0.700} & \scorecell{0.844} & \scorecell{0.596} & \scorecell{0.657} & \scorecell{0.653} & \scorecell{0.844} & \scorecell{0.611} & \scorecell{0.642} \\
            \midrule
            \multirow{4}{*}{DLightSB-M} & \multirow{2}{*}{KL} & 16 & \scorecell{0.862} & \scorecell[\underline]{0.914} & \scorecell[\underline]{0.924} & \scorecell[\textbf]{0.929} & \scorecell{0.811} & \scorecell{0.909} & \scorecell[\textbf]{0.921} & \scorecell[\underline]{0.920} & \scorecell{0.834} & \scorecell{0.758} & \scorecell{0.684} & \scorecell{0.732} \\
             &  & 64 & \scorecell{0.846} & \scorecell[\underline]{0.908} & \scorecell[\underline]{0.917} & \scorecell[\underline]{0.922} & \scorecell[\underline]{0.870} & \scorecell[\underline]{0.921} & \scorecell[\textbf]{0.918} & \scorecell[\underline]{0.921} & \scorecell{0.830} & \scorecell{0.853} & \scorecell{0.697} & \scorecell[\underline]{0.841} \\
             & \multirow{2}{*}{MSE} & 16 & \scorecell{0.709} & \scorecell[\underline]{0.907} & \scorecell{0.824} & \scorecell{0.900} & \scorecell{0.619} & \scorecell[\underline]{0.922} & \scorecell[\underline]{0.868} & \scorecell{0.877} & \scorecell{0.683} & \scorecell{0.829} & \scorecell[\underline]{0.776} & \scorecell{0.604} \\
             &  & 64 & \scorecell{0.685} & \scorecell[\underline]{0.905} & \scorecell{0.790} & \scorecell{0.887} & \scorecell{0.582} & \scorecell[\underline]{0.915} & \scorecell[\underline]{0.874} & \scorecell{0.849} & \scorecell{0.574} & \scorecell{0.829} & \scorecell{0.530} & \scorecell{0.759} \\
            \bottomrule
        \end{tabular}
        \end{adjustbox}
        \caption{ \centering Conditional Shape Score ($\uparrow$)}
        \label{table:conditional_shape_score}
    \end{subtable}
    \begin{subtable}[t]{0.336\paperwidth}
        \centering
        \begin{adjustbox}{width=\linewidth}
        \begin{tabular}{rcccccccccccccc}
            \toprule
            \multicolumn{4}{c}{$D\!=\!2$} & \multicolumn{4}{c}{$D\!=\!16$} & \multicolumn{4}{c}{$D\!=\!64$} \\
            \cmidrule(lr){1-4} \cmidrule(lr){5-8} \cmidrule(l){9-12}
            \multicolumn{2}{c}{gaussian} & \multicolumn{2}{c}{uniform} & \multicolumn{2}{c}{gaussian} & \multicolumn{2}{c}{uniform} & \multicolumn{2}{c}{gaussian} & \multicolumn{2}{c}{uniform} \\
            \cmidrule(lr){1-4} \cmidrule(lr){3-4} \cmidrule(l){5-6} \cmidrule(lr){7-8} \cmidrule(lr){9-10} \cmidrule(l){11-12}
            {$0.02$} & {$0.05$} & {$0.005$} & {$0.01$} & {$0.02$} & {$0.05$} & {$0.005$} & {$0.01$} & {$0.02$} & {$0.05$} & {$0.005$} & {$0.01$} \\
            \midrule
            \scorecell{0.471} & \scorecell{0.779} & \scorecell{0.521} & \scorecell{0.551} & \scorecell{0.569} & \scorecell{0.676} & \scorecell{0.369} & \scorecell{0.458} & \scorecell{0.478} & \scorecell{0.509} & \scorecell{0.351} & \scorecell{0.432} \\
            \scorecell{0.085} & \scorecell{0.277} & \scorecell{0.230} & \scorecell{0.273} & \scorecell{0.111} & \scorecell{0.082} & \scorecell{0.090} & \scorecell{0.095} & \scorecell{0.195} & \scorecell{0.121} & \scorecell{0.136} & \scorecell{0.110} \\ 
            \scorecell{0.585} & \scorecell{0.630} & \scorecell{0.704} & \scorecell{0.738} & \scorecell{0.686} & \scorecell{0.532} & \scorecell{0.351} & \scorecell{0.373} & \scorecell{0.838} & \scorecell{0.404} & \scorecell{0.532} & \scorecell{0.456} \\        \midrule
            \textbf{\scorecell{0.908}} & \textbf{\scorecell{0.851}} & \textbf{\scorecell{0.873}} & \textbf{\scorecell{0.872}} & \scorecell{0.763} & \textbf{\scorecell{0.838}} & \textbf{\scorecell{0.839}} & \textbf{\scorecell{0.838}} & \textbf{\scorecell{0.854}} & \textbf{\scorecell{0.747}} & \textbf{\scorecell{0.817}} & \textbf{\scorecell{0.773}} \\
            \midrule
            \scorecell{0.643} & \scorecell{0.610} & \scorecell{0.801} & \scorecell{0.808} & \scorecell{0.674} & \scorecell{0.609} & \scorecell{0.517} & \scorecell{0.513} & \scorecell{0.756} & \scorecell{0.694} & \scorecell{0.517} & \scorecell{0.537} \\
            \underline{\scorecell{0.837}} & \scorecell{0.794} & \scorecell{0.840} & \scorecell{0.846} & \scorecell{0.801} & \scorecell{0.738} & \scorecell{0.572} & \scorecell{0.624} & \underline{\scorecell{0.788}} & \scorecell{0.698} & \scorecell{0.486} & \scorecell{0.536} \\
            \scorecell{0.444} & \scorecell{0.627} & \scorecell{0.692} & \scorecell{0.728} & \scorecell{0.658} & \scorecell{0.574} & \scorecell{0.392} & \scorecell{0.418} & \scorecell{0.586} & \scorecell{0.655} & \scorecell{0.446} & \scorecell{0.463} \\
            \scorecell{0.227} & \scorecell{0.748} & \scorecell{0.686} & \scorecell{0.673} & \scorecell{0.726} & \scorecell{0.716} & \scorecell{0.482} & \scorecell{0.551} & \scorecell{0.496} & \underline{\scorecell{0.714}} & \scorecell{0.479} & \scorecell{0.478} \\
            \midrule
            \scorecell{0.588} & \scorecell{0.616} & \scorecell{0.819} & \scorecell{0.814} & \scorecell{0.610} & \scorecell{0.659} & \scorecell{0.545} & \scorecell{0.557} & \scorecell{0.773} & \scorecell{0.659} & \scorecell{0.484} & \scorecell{0.515} \\
            \scorecell{0.734} & \scorecell{0.794} & \scorecell{0.834} & \scorecell{0.844} & \textbf{\scorecell{0.816}} & \scorecell{0.798} & \scorecell{0.563} & \scorecell{0.642} & \underline{\scorecell{0.793}} & \scorecell{0.677} & \scorecell{0.486} & \scorecell{0.514} \\
            \scorecell{0.470} & \scorecell{0.607} & \scorecell{0.767} & \scorecell{0.781} & \scorecell{0.604} & \scorecell{0.649} & \scorecell{0.449} & \scorecell{0.512} & \scorecell{0.592} & \scorecell{0.618} & \scorecell{0.448} & \scorecell{0.457} \\
            \underline{\scorecell{0.838}} & \scorecell{0.782} & \scorecell{0.758} & \scorecell{0.785} & \scorecell{0.587} & \scorecell{0.761} & \scorecell{0.486} & \scorecell{0.561} & \scorecell{0.488} & \scorecell{0.687} & \scorecell{0.466} & \scorecell{0.486} \\
            \midrule
            \scorecell{0.817} & \textbf{\scorecell{0.848}} & \underline{\scorecell{0.860}} & \underline{\scorecell{0.864}} & \scorecell{0.702} & \scorecell{0.820} & \underline{\scorecell{0.831}} & \underline{\scorecell{0.820}} & \scorecell{0.731} & \scorecell{0.622} & \scorecell{0.468} & \scorecell{0.500} \\
            \scorecell{0.798} & \underline{\scorecell{0.843}} & \scorecell{0.853} & \underline{\scorecell{0.856}} & \underline{\scorecell{0.791}} & \underline{\scorecell{0.825}} & \underline{\scorecell{0.825}} & \underline{\scorecell{0.816}} & \scorecell{0.727} & \underline{\scorecell{0.710}} & \scorecell{0.493} & \underline{\scorecell{0.674}} \\
            \scorecell{0.624} & \underline{\scorecell{0.839}} & \scorecell{0.755} & \scorecell{0.838} & \scorecell{0.451} & \scorecell{0.824} & \scorecell{0.783} & \scorecell{0.785} & \scorecell{0.532} & \scorecell{0.683} & \underline{\scorecell{0.645}} & \scorecell{0.328} \\
            \scorecell{0.595} & \scorecell{0.834} & \scorecell{0.716} & \scorecell{0.822} & \scorecell{0.404} & \scorecell{0.814} & \scorecell{0.786} & \scorecell{0.763} & \scorecell{0.392} & \scorecell{0.678} & \scorecell{0.327} & \scorecell{0.600} \\
            \bottomrule
        \end{tabular}
        \end{adjustbox}
        \caption{\centering Conditional Trend Score ($\uparrow$)}
        \label{table:conditional_trend_score}
    \end{subtable}}
    \vspace{-2mm}
    \caption{\centering Conditional metrics on our high-dimensional Gaussian mixture benchmark. The best-performing method is highlighted in bold, and the second is underlined. Color code: vermillion for $<0.5$, orange for $[0.5,0.75)$, yellow for $[0.75,0.85)$, and bluish-green for $\geq 0.85$.}
    \vspace{-5mm}
\end{table*}

In the main text, we report only the conditional metrics, as they more accurately reflect the performance of the EOT/SB solvers. In Appendix \ref{appendix:additional_experiments_main} we provide experiments to \underline{validate conditional} \underline{metrics} against the unconditional ones. \underline{Further experimental details} are provided in Appendix~\ref{appendix:experiment_details}.

\textbf{High-Dimensional Gaussian Mixtures.} \label{section:hd_experiments} Here, we report results on the high-dimensional Gaussian mixture benchmark constructed as in \secref{subsection:benchmark_pairs_construction} using the methods from \secref{section:solvers_for_evaluation}. Visual results are shown in Figure \ref{figure:hd_benchmarks_main} for $q^\textup{gauss}$ ($\gamma\!=\!0.02$) and $q^\textup{unif}$ ($\gamma\!=\!0.005$), while \underline{additional plots} provided in Appendix~\ref{appendix:additional_experiments_main}. The qualitative results are presented in Tables \ref{table:conditional_shape_score}, \ref{table:conditional_trend_score}, \ref{table:forward_kl}, \ref{table:reverse_kl}, \ref{table:unconditional_shape_score} and \ref{table:unconditional_trend_score}.

First, across all benchmarks, each baseline underperforms for reasons specific to its construction. In particular, \textit{Reference} consistently degrades, indicating that the scalar value function $v^*$ in \eqref{equation:conditional_closed_form} strongly reweights the reference process and yields a challenging $q^*$ to approximate. The \textit{Independent} baseline deteriorates when the stochasticity parameter is small, while \textit{Feature-wise SB} degrades as the dimension increases. Together, these failures highlight the target properties that our benchmarks are designed to test in EOT/SB solvers.

Second, DLightSB consistently achieves the strongest performance across all setups. We attribute this to the DLightSB solver being constructed according to the same principle underlying our benchmark. DLightSB-M, which incorporates the same inductive bias, attains comparable results with a slight drop in metrics, likely due to the additional variance introduced by the KL minimization loss. While this alignment could be seen as a limitation of the benchmark, we take the opposite view: DLightSB(-M) can be treated as an oracle-like method in this setting, since its inductive bias makes it less informative as a measure of general solver performance. For an analysis of the \underline{reverse benchmark} designed to probe this bias, see Appendix~\ref{appendix:reverse_benchmark}.

Finally, CSBM and $\alpha$-CSBM perform noticeably worse than DLightSB(-M). Notably, $\alpha$-CSBM attains comparable quality to CSBM while halving the computational cost, making it a more efficient alternative. Increasing $N$ generally improves the metrics. As for the loss function, KL consistently outperforms MSE, likely because MSE minimizes pointwise squared error and yields over-smoothed solutions that blur modes. This effect can be observed in Figure \ref{figure:g2_002_samples}.

\vspace{-4mm}
\section{Discussion}
\label{section:discussion}
\vspace{-4mm}

Our work addresses key gaps in discrete-space EOT/SB research by introducing the first standardized benchmark for these methods, along with new approaches. This contribution provides ground-truth data and consistent evaluation metrics for rigorous assessment of underlying solvers. The benchmark reveals fundamental limitations of both existing and proposed approaches: DLightSB(-M) faces severe memory constraints in high dimensions, while CSBM and $\alpha$-CSBM remain sensitive to parameter choices and require long training times. These findings underscore the need for more scalable architectures and stable training procedures, and guide future research in this area.

\textbf{Reproducibility.} We provide \underline{the experimental details} in Appendix \ref{appendix:experiment_details} and the code required to reproduce the conducted experiments is available in \href{[https://github.com/gregkseno/catsbench}{this repository} (see \texttt{README.md} for details).

\textbf{LLM Usage.} Large Language Models (LLMs) were used only to assist with rephrasing sentences and improving the clarity of the text. All scientific content, results, and interpretations in this paper were developed solely by the authors.

\vspace{-2mm}
\section*{Acknowledgments}
\vspace{-2mm}
The work was supported by the grant for research centers in the field of AI provided by the Ministry of Economic Development of the Russian Federation in accordance with the agreement 000000C313925P4F0002 and the agreement №139-10-2025-033.

\bibliography{iclr2026_conference}
\bibliographystyle{iclr2026_conference}

\newpage
\appendix

\section{Proofs}
\label{appendix:proofs}
\begin{proof}[Proof of Theorem \ref{theorem:benchmark_construction}] 

We start from the expression of the static SB problem in \eqref{equation:static_sb}
    
\vspace{-3mm}

\begin{equation}
    \begin{split}
        \min_{q \in \Pi(p_0, p_1)}&\KL{q(x_0, x_1)}{q^{\text{ref}}(x_0, x_1)}=\\&= \min_{q \in \Pi(p_0, p_1)}-H(q) - \sum_{x_0, x_1} q(x_0, x_1)\log q^{\text{ref}}(x_1| x_0)-\text{const}\\ &=\min_{q \in \Pi(p_0, p_1)}\sum_{x_0,x_1}q(x_0, x_1)\log q(x_0, x_1) - \sum_{x_0, x_1} q(x_0, x_1)\log q^{\text{ref}}(x_1| x_0)-\text{const}
    \end{split}
\end{equation}

\vspace{-2mm}

Noting that the joint distribution can be written as $q(x_0,x_1)=q(x_0)q(x_1| x_0)=p_0(x_0)q(x_1| x_0)$, and enforcing the marginal constraints $\sum_{x_0} p_0(x_0) q(x_1| x_0)=p_1(x_1)$ and $\sum_{x_1} q(x_1| x_0)=1$ (equivalently $q(x_0)=p_0(x_0)$), the corresponding Lagrangian can be formulated as

\vspace{-3mm}
\begin{equation}
\begin{split}
    \mathcal{C}(q)&=\sum_{x_0,x_1}p_0(x_0)q(x_1|x_0)\log 
    \big(p_0(x_0)q(x_1|x_0)\big) - \sum_{x_0, x_1} p_0(x_0)q(x_1|x_0)\log q^{\text{ref}}(x_1| x_0)+\\&+\sum_{x_1}\lambda(x_1)\left(\sum_{x_0}q(x_1|x_0)p_0(x_0)-p_1(x_1)\right)+\sum_{x_0}\tau(x_0)\left(\sum_{x_1}q(x_1|x_0)-p_0(x_0)\right)\\
    &=\underbrace{\sum_{x_0,x_1}p_0(x_0)q(x_1|x_0)\log 
    p_0(x_0)}_{=\sum_{x_0}p_0(x_0)\log 
    p_0(x_0)}+\sum_{x_0,x_1}p_0(x_0)q(x_1|x_0)\log 
     q(x_1|x_0) -\\-&\sum_{x_0, x_1} p_0(x_0)q(x_1|x_0)\log q^{\text{ref}}(x_1| x_0)+\sum_{x_1}\lambda(x_1)\left(\sum_{x_0}q(x_1|x_0)p_0(x_0)-p_1(x_1)\right)\\
     &+\sum_{x_0}\tau(x_0)\left(\sum_{x_1}q(x_1|x_0)-1\right)
\end{split}
\end{equation}

where $\lambda(x_1)$ and $\tau(x_0)$ denote the Lagrange multipliers associated with the marginal constraints on $x_1$ and $x_0$, respectively. Taking the pointwise partial derivative of $\mathcal{C}(q)$ with respect to $q(x_1| x_0)$

\vspace{-3mm}
\begin{equation}
    \frac{\partial \mathcal{C}}{\partial q(x_1|x_0)}=p_0(x_0)\left(\log q(x_1|x_0)+1\right)-p_0(x_0)\log q^{\text{ref}}(x_1|x_0)+\lambda(x_1)p_0(x_0)+\tau(x_0)=0
\end{equation}

Therefore, the process $q$ can be written as

\vspace{-3mm}

\begin{equation}
    q(x_1|x_0) = \exp(-\lambda(x_1)-1)q^{\text{ref}}(x_1|x_0)\exp\left(-\frac{\tau(x_0)}{p_0(x_0)}
    \right)
\end{equation}

\vspace{-2mm}
We can then define

\vspace{-3mm}

\[
v(x_1) := \exp({-1-\lambda(x_1)}), \qquad
\psi(x_0) := \exp\left({-\frac{\tau(x_0)}{p_0(x_0)}}\right).
\]

\vspace{-2mm}
Then  \(q(x_1|x_0) = \psi(x_0)  q^{\text{ref}}(x_1|x_0)v(x_1)\) and the normalization constraint \(\sum_{x_1} q(x_1|x_0) = 1\) forces  

\[
\psi(x_0) = \frac{1}{\displaystyle\sum_{x_1} v(x_1) q^{\text{ref}}(x_1|x_0)} =: \frac{1}{c(x_0)}.
\]

Thus the optimal conditional distribution $q^*(x_1|x_0)$ is uniquely defined by  

\begin{equation}
    q^*(x_1|x_0) = \frac{1}{c^*(x_0)}v^*(x_1)  q^{\text{ref}}(x_1|x_0),
\qquad\text{with } c^*(x_0) = \sum_{x_1} v^*(x_1) q^{\text{ref}}(x_1|x_0).
\label{equation:optimal_q_star}
\end{equation}

By construction \(q^*(x_0)=p_0(x_0)\), and  \mbox{\(p_1(x_1) := q^*(x_1) = \sum_{x_0} p_0(x_0) q^*(x_1|x_0)\)} is its second marginal, so \(q^*\in\Pi(p_0,p_1)\). Because the KL divergence is strictly convex and the feasible set \(\Pi(p_0,p_1)\) is convex, the first‑order conditions are sufficient for optimality. Hence \(q^*\) is the unique minimizer of \(\displaystyle\min_{q\in\Pi(p_0,p_1)}\KL{q}{q^{\text{ref}}}\). Therefore \(q^*\) in Equation \eqref{equation:optimal_q_star} together with the reference process \(q^{\text{ref}}\) defines the discrete‑space EOT/SB between \(p_0\) and \(p_1\).

\end{proof}

\begin{proof}[Proof of Proposition \ref{prop:tractable_parameterization}]
    Assuming the CP parameterization introduced in \eqref{equation:cp_decomposition}, and further assuming that the reference process factorizes across dimensions as $q^{\text{ref}}(x_1| x_0)=\prod_{d=1}^D q^{\text{ref}}(x_1^d| x_0^d)$, the normalized conditional distribution $q^*(x_1| x_0)$ in \eqref{equation:normalized_conditional} can be rewritten as
\begin{equation}
\begin{split}
    q^*(x_1| x_0)
    &=\frac{1}{c^*(x_0)}\left(\sum_{k=1}^K \beta_k \prod_{d=1}^D r_k^d[x_1^d]\right)\prod_{d=1}^D q^{\text{ref}}(x_1^d| x_0^d)\\
    &=\frac{1}{c^*(x_0)}\sum_{k=1}^K \beta_k \prod_{d=1}^D r_k^d[x_1^d]\,q^{\text{ref}}(x_1^d| x_0^d),
\end{split}
\end{equation}
where the reference factors can be merged with the rank-1 components because they are independent of the mixture index $k$ and factorize over dimensions. From here, it is possible to obtain the normalizing constant $c^*(x_0)$ by summing over all possible values of $x_1\in\gX=\mathbb{S}^D$, where $x_1^d\in \{0,\dots,S-1\}$. The normalizing constant can then be rewritten as
\begin{equation}
\begin{split}
    c^*(x_0) &= \sum_{x_1\in \mathbb{S}^D}\sum_{k=1}^K \beta_k \prod_{d=1}^D r_k^d[x_1^d]\,q^{\text{ref}}(x_1^d| x_0^d)\\
    &=\sum_{k=1}^K\beta_k \sum_{x_1\in \mathbb{S}^D}  \prod_{d=1}^D r_k^d[x_1^d]\,q^{\text{ref}}(x_1^d| x_0^d)\\
    &=\sum_{k=1}^K\beta_k  \prod_{d=1}^D\sum_{x_1^d=0}^{S-1} r_k^d[x_1^d]\,q^{\text{ref}}(x_1^d| x_0^d),
\end{split}
\end{equation}
where $\sum_{x_1\in \mathbb{S}^D}=\sum_{x_1^1=0}^{S-1}\sum_{x_1^2=0}^{S-1}\dots\sum_{x_1^D=0}^{S-1}$. The exchange between the product and the sum is valid here because the summation is separable across dimensions, i.e., each factor depends only on its corresponding coordinate $x_1^d$.
\end{proof}

\begin{proof}[Proof of Proposition \ref{proposition:dlight_sb_loss}]
We start from the standard KL minimization problem from the LightSB paper \citep{korotin2024light} and define it in discrete space.

\vspace{-4mm}
\begin{align*}
    \begin{split}
        \KL{q^*}{q}&=\sum_{x_0,x_1}q^*(x_0,x_1)\log \left(\frac{q^*(x_0,x_1)}{q(x_0,x_1)}\right)=\sum_{x_0, x_1}q^* \log q^*(x_0,x_1)-\sum_{x_0, x_1}q^* \log q(x_0,x_1)=\\
        &=-H(q^*)-\sum_{x_0, x_1}q^*(x_0,x_1)\log q(x_0,x_1)=-H(q^*)-\sum_{x_0, x_1}q^*(x_0,x_1)\log \left(q(x_0)q(x_1|x_0)\right)\\
        &= -H(q^*)-\sum_{x_0, x_1}q^*(x_0,x_1)\log \underbrace{q(x_0)}_{=p_0(x_0)}-\sum_{x_0, x_1}q^*(x_0,x_1)\log q(x_1|x_0)=\\
        &= -H(q^*)-\sum_{x_0}\log p_0(x_0)\underbrace{\sum_{x_1}q^*(x_0,x_1)}_{=q^*(x_0)=p_0(x_0)}-\sum_{x_0, x_1}q^*(x_0,x_1)\log q(x_1|x_0)
    \end{split}
\end{align*}

Now using \eqref{equation:normalized_conditional} on $q(x_1|x_0)$ we can get

\begin{equation*}
    \begin{split}
        \KL{q^*}{q}&=-H(q^*)-\sum_{x_0}\log p_0(x_0)p_0(x_0)-\sum_{x_0, x_1}q^*(x_0,x_1)\log \left(\frac{v^*(x_1)}{c^*(x_0)} q^{\text{ref}}(x_1| x_0)\right)=\\
        &=\underbrace{-H(q^*)-\sum_{x_0}\log p_0(x_0)p_0(x_0)-\sum_{x_0, x_1}q^*(x_0,x_1)\log  q^{\text{ref}}(x_1| x_0)}_{=-\mathcal{L^*}}-\\ &-\sum_{x_0, x_1}q^*(x_0,x_1)\log \left(\frac{v^*(x_1)}{c^*(x_0)}\right)=\\
        &=-\mathcal{L}^*+\sum_{x_0, x_1}q^*(x_0,x_1)\log c^*(x_0)-\sum_{x_0, x_1}q^*(x_0,x_1)\log v^*(x_1)=\\
        &=\sum_{x_0}p_0^*(x_0)\log c^*(x_0)-\sum_{x_1}q^*(x_1)\log v^*(x_1) -\mathcal{L}^*\\
        &=\mathbb{E}_{p_0(x_0)}\big[\log c_{\theta}(x_0)\big] -\mathbb{E}_{p_1(x_1)}\big[\log v_{\theta}(x_1)\big]-\mathcal{L}^*,
    \end{split}
\end{equation*}

That concludes the proof.
\end{proof}

\begin{proof}[Proof of Proposition \ref{proposition:optimal_projection}]
    \begin{align}
        \KL{r(x_0, x_{\text{in}}, x_1)}{q^{\text{SB}}(x_0, x_{\text{in}}, x_1)} = \nonumber \\
        = \KL{r(x_0, x_1)}{q^{\text{SB}}(x_0, x_1)} + \underbrace{\KL{q^{\text{ref}}(x_{\text{in}} | x_0, x_1)}{q^{\text{ref}}(x_{\text{in}}|x_0, x_1)}}_{=0} = \label{equation:op_proof_1} \\
        = \underbrace{\sum_{x_0, x_1}r(x_0, x_1)\log r(x_0,x_1)}_{=-H(r(x_0, x_1))} - \sum_{x_0, x_1}r(x_0, x_1)\log q^{\text{SB}}(x_0, x_1) = \nonumber \\
        = -H(r(x_0, x_1)) - \sum_{x_0,x_1}r(x_0, x_1) \log \frac{v^{\text{SB}}(x_1)q^{\text{ref}}(x_1 | x_0)}{c^{\text{SB}}(x_0)} = \label{equation:op_proof_2} \\
        = -H(r(x_0, x_1)) - \sum_{x_0,x_1}r(x_0, x_1) \log v^{\text{SB}}(x_1) - \nonumber \\ - \sum_{x_0,x_1}r(x_0, x_1)\log q^{\text{ref}}(x_1 | x_0) + \sum_{x_0,x_1}r(x_0, x_1)\log c^{\text{SB}}(x_0) = \nonumber \\
        = -H(r(x_0, x_1)) - \sum_{x_1}\log v^{\text{SB}}(x_1) \underbrace{r(x_1)}_{=p(x_1)=q^{*}(x_1)}\underbrace{\sum_{x_0}r(x_0 | x_1)}_{=1=\sum_{x_0}q^*(x_0 | x_1)} - \label{equation:op_proof_3} \\ - \sum_{x_0,x_1}r(x_0, x_1)\log q^{\text{ref}}(x_1 | x_0) + \sum_{x_0}\log c^{\text{SB}}(x_0) \underbrace{r(x_0)}_{=p(x_0) = q^*(x_0)}\underbrace{\sum_{x_1}r(x_1 | x_0)}_{=1 = \sum_{x_1}q^{*}(x_1 | x_0)} = \label{equation:op_proof_4} \\
        = \underbrace{-H(r(x_0, x_1)) - \sum_{x_0,x_1}r(x_0, x_1)\log q^{\text{ref}}(x_1 | x_0)}_{=C_1} - \sum_{x_0,x_1}q^*(x_0, x_1)\log \frac{v^{\text{SB}}(x_1)}{c^{\text{SB}}(x_0)} = \nonumber \\
        = C_1 - \sum_{x_0,x_1}q^*(x_0, x_1)\log \frac{v^{\text{SB}}(x_1)}{c^{\text{SB}}(x_0)} - \nonumber \\ -\underbrace{\sum_{x_0,x_1}q^*(x_0, x_1)\log q^{\text{ref}}(x_1 | x_0) + \sum_{x_0,x_1}q^*(x_0, x_1)\log q^{\text{ref}}(x_1 | x_0)}_{=0} = \label{equation:op_proof_5} \\
        = - \sum_{x_0,x_1}q^*(x_0, x_1)\log \frac{v^{\text{SB}}(x_1)q^{\text{ref}}(x_1 | x_0)}{c^{\text{SB}}(x_0)} + \underbrace{C_1 + \sum_{x_0,x_1}q^*(x_0, x_1)\log q^{\text{ref}}(x_1 | x_0)}_{=C_2} = \nonumber \\
        = C_2 - \sum_{x_0,x_1}q^*(x_0, x_1)\log q^{\text{SB}}(x_0, x_1) + \nonumber \\
        + \underbrace{\sum_{x_0, x_1}q^*(x_0, x_1)\log q^*(x_0,x_1) - \sum_{x_0, x_1}q^*(x_0, x_1)\log q^*(x_0,x_1)}_{=0} = \label{equation:op_proof_6} \\
        = \sum_{x_0,x_1}q^*(x_0, x_1)\log \frac{q^{*}(x_0, x_1)}{q^{\text{SB}}(x_0, x_1)} + \underbrace{C_2 - \sum_{x_0, x_1}q^*(x_0, x_1)\log q^*(x_0,x_1)}_{C_3} = \nonumber \\
        = \KL{q^*(x_0, x_1)}{q^{\text{SB}}(x_0, x_1)} + C_3 \nonumber
    \end{align}

    In \eqref{equation:op_proof_1}, we use the disintegration of the KL divergence to transition from the dynamic to the static formulation. In \eqref{equation:op_proof_2}, we apply our parameterization from \eqref{equation:normalized_conditional}. Next, in \eqref{equation:op_proof_3} and \eqref{equation:op_proof_4}, we use the properties of the reciprocal process $r$, which has the true marginals at $t = 0$ and $t = 1$. In \eqref{equation:op_proof_5}, we add a zero term to introduce $q^{\text{ref}}(x_1 | x_0)$ with the expectation taken over the optimal joint distribution $q^*(x_0, x_1)$. Finally, in \eqref{equation:op_proof_6}, we obtain the entropy term, completing the expression for the desired KL divergence. This establishes that optimizing $\KL{r(x_0,x_1)}{q^{\text{SB}}(x_0, x_1)}$ with respect to $q^{\text{SB}}$ is equivalent to optimizing $\KL{q^*(x_0,x_1)}{q^{\text{SB}}(x_0, x_1)}$ up to an additive constant.
\end{proof}

\vspace{-5mm}
\begin{proof}[Proof of Proposition \ref{proposition:sb_trainsition_distribution}]
    We first derive the transitional distributions of the SB by recalling its well-known characterization \citep[Thm.~2]{georgiou2015positive}:
    \begin{equation*}
        q^{*}\left(x_{t_{n}}| x_{t_{n-1}}\right) = q^{\text{ref}}\left(x_{t_{n}}| x_{t_{n-1}}\right)\frac{\phi^{*}_{t_{n}}(x_{t_{n}})}{\phi^{*}_{t_{n-1}}(x_{t_{n-1}})}, \qquad \phi^{*}_{t_{n}}(x_{t_n}) =\mathbb{E}_{q^{\mathrm{ref}}(x_1| x_{t_n})} \big[v^{*}(x_1)\big].
    \end{equation*}
    
    Using the CP parametrization of $v^{*}$ from \eqref{equation:cp_decomposition} and exploiting the factorization of $q^{\text{ref}}$, the scalar-valued functions $\phi_{t_n}$ can be written as:
    \begin{equation*}
        \phi^{*}_{t_n}(x_{t_n}) = \sum_{k=1}^K \beta_k \prod_{d=1}^D \mathbb{E}_{q^{\text{ref}}(x_1^d | x_{t_n}^d)}\big[r_k^d[x_1^d]\big] = \sum_{k=1}^K \beta_k \prod_{d=1}^D \underbrace{\sum_{x_1^d=0}^{S-1} q^{\text{ref}}(x_1^d| x_{t_n}^d)\, r_k^d[x_1^d]}_{u_{k,t_n}^d[x_{t_n}^d]},
    \end{equation*}
    where $u_{k,t_n}^d$ satisfy the following recursive relation:
    \begin{equation*}
        u_{k,t_n}^d[x_{t_n}^d] =\sum_{x_{t_{n+1}}^d=0}^{S-1} q^{\text{ref}}(x_{t_{n+1}}^d | x_{t_n}^d)\, u_{k,t_{n+1}}^d[x_{t_{n+1}}^d], \qquad u_{k,1}^d=r_k^d.
    \end{equation*}

    Thus, we obtain the following transition distributions:
    \begin{equation}
    \label{equation:characterization_proof_1}
        q^{*}(x_{t_{n}} | x_{t_{n-1}}) \propto  q^{\text{ref}}(x_{t_{n}} | x_{t_{n-1}}) \sum_{k=1}^K \beta_k \prod_{j=1}^D u_{k,t_{n}}^j[x_{t_{n}}^j].
    \end{equation}
    
    To obtain the $d$-th marginal transition distribution, we marginalize over $x_{t_{n}}^{-d} \coloneqq \{x_{t_{n}}^j\}_{j \neq d}$ as follows:
    \begin{multline*}
        q^{*}\big(x_{t_{n}}^d | x_{t_{n-1}}\big) \propto  \sum_{x_{t_{n}}^{-d}} \Bigg(\prod_{j=1}^Dq^{\text{ref}}(x_{t_{n}}^j | x^j_{t_{n-1}})\Bigg)\Bigg(\sum_{k=1}^K \beta_k \prod_{j=1}^D u_{k,t_{n}}^j[x_{t_{n}}^j]\Bigg) = \\ 
        = q^{\text{ref}}(x_{t_{n}}^d | x_{t_{n-1}}^d)\sum_{k=1}^K \beta_k u_{k,t_{n}}^d[x_{t_{n}}^d]\prod^D_{\substack{j=1\\j\neq d}}\underbrace{\sum_{x_{t_n}^j}q^{\text{ref}}(x_{t_n}^j | x_{t_{n-1}}^j)\,u_{k,t_n}^j[x_{t_n}^j]}_{u_{k,t_{n-1}}^j[x_{t_{n-1}}^j]\ \text{(by recursion)}}. \nonumber
    \end{multline*}
    
    Finally, we obtain the desired expression up to normalization:
    \begin{equation}
        \label{equation:marginal_sb_trainsition_distribution}
        q^{*}\big(x_{t_{n}}^d | x_{t_{n-1}}\big) \propto [Q^{\text{ref}}]_{x_{t_{n-1}}^d,x_{t_{n}}^d}\sum_{k=1}^K \beta_k u_{k,t_{n}}^d[x_{t_{n}}^d]\prod^D_{\substack{j=1\\j\neq d}}u_{k,t_{n-1}}^j[x_{t_{n-1}}^j].
    \end{equation}
\end{proof}

\section{Closed-form of the Conditional Distribution for the Uniform Reference Process}
\label{appendix:static_q_ref}
It is also useful to derive a closed-form expression for $q^{\text{ref}}(x^d_1 | x^d_0)$. Since this conditional distribution is given by the $(N\!+\!1)$-step transition of the reference chain, it can be obtained from the $(N\!+\!1)$-th power of the transition matrix $Q^{\text{ref}}$. In the uniform case, this yields:
\begin{equation}
    \overline{Q}^{\text{ref}}_{N+1} = \left(1-\gamma\frac{S}{S-1}\right)^{N+1}\mathbb{I}+\frac{1-\left(1-\gamma\frac{S}{S-1}\right)^{N+1}}{S}\mathbf{1}\mathbf{1}^\top,
    \label{equation:static_q_ref}
\end{equation}
where $\mathbf{1}\!=\![1,\dots, 1]^{\top}\!\!\in\!\mathbb{R}^S$ is a vector of ones and $\mathbb{I}$ is an $S\!\times \!S$ identity matrix.

    \begin{proof}[Proof of Equation \ref{equation:static_q_ref}]
    Let $ Q $ be the transition matrix in \eqref{equation:piref_uniform}, rewritten as
    \begin{align*}
    Q &= (1-\gamma) I + \frac{\gamma}{S-1} (\mathbf{1}\mathbf{1}^\top - I) \\
      &= \left(1 - \gamma \frac{S}{S-1}\right) I + \frac{\gamma}{S-1} \mathbf{1}\mathbf{1}^\top,
    \end{align*}
    where $ I $ is the identity matrix and $ \mathbf{1}\mathbf{1}^\top $ is the all-ones matrix. Let
    \[
    a = 1 - \gamma \frac{S}{S-1}, \quad b = \frac{\gamma}{S-1},
    \]
    so that $ Q = aI + b\mathbf{1}\mathbf{1}^\top $ and note that $ a + Sb = 1 $. We compute $ Q^{N+1} $ using the binomial expansion. Since $ I $ and $ \mathbf{1}\mathbf{1}^\top $ commute:
    \begin{align*}
    Q^n &= (aI + b\mathbf{1}\mathbf{1}^\top)^n \\
        &= \sum_{k=0}^n \binom{n}{k} a^{n-k} b^k (\mathbf{1}\mathbf{1}^\top)^k.
    \end{align*}
    
    Using $ (\mathbf{1}\mathbf{1}^\top)^k = S^{k-1} \mathbf{1}\mathbf{1}^\top $ for $ k \ge 1 $ and separating the $ k=0 $ term:
    \begin{align*}
    Q^n &= a^n I + \sum_{k=1}^n \binom{n}{k} a^{n-k} b^k S^{k-1} \mathbf{1}\mathbf{1}^\top \\
        &= a^n I + \frac{1}{S} \left( \sum_{k=1}^n \binom{n}{k} a^{n-k} (bS)^k \right) \mathbf{1}\mathbf{1}^\top.
    \end{align*}
    
    The binomial expansion gives:
    \[
    (a + bS)^n = \sum_{k=0}^n \binom{n}{k} a^{n-k} (bS)^k = a^n + \sum_{k=1}^n \binom{n}{k} a^{n-k} (bS)^k.
    \]
    
    Since $ a + bS = 1 $, we have $ (a + bS)^n = 1 $, so
    $
    \sum_{k=1}^n \binom{n}{k} a^{n-k} (bS)^k = 1 - a^n
    $. Thus,
    \[
    Q^n = a^n I + \frac{1 - a^n}{S} \mathbf{1}\mathbf{1}^\top.
    \]
    
    Substituting $ n = N+1 $ and $ a = 1 - \gamma \frac{S}{S-1} $ yields
    \[
    q^{\text{ref}}(x^d_1|x^d_0) = Q^{N+1} = \left(1-\gamma\frac{S}{S-1}\right)^{N+1} I + \frac{1-\left(1-\gamma\frac{S}{S-1}\right)^{N+1}}{S} \mathbf{1}\mathbf{1}^\top.
    \]
\end{proof}

From here it can be seen that $\overline{Q}^{\text{ref}}_{N+1}$ converges to $(1/S)\mathbf{1}\mathbf{1}^\top$ when $(N\!+\!1)\rightarrow\infty$, that is a uniform distribution over the number of categories $S$. In the case of the Gaussian reference process, the closed-form expression can also be obtained, but it is much more complex.

\section{Experiment Details}
\label{appendix:experiment_details}

This section provides detailed descriptions of all methods and their configurations.

%============DO NOT DELETE=================
%\paragraph{Shared Aspects.} Across all experiments, we use the AdamW optimizer with fixed \texttt{beta} values of $0.95$ and $0.99$. For the high-dimensional Gaussian benchmark (\secref{section:hd_experiments}), we use a batch size of $128$, whereas for the image denoising benchmark (\secref{section:image_denoising_experiments}), we train with a batch size of $32$. Notably, for diffusion-based methods, we fully sample the Markov chain, in contrast to \citet{austin2021structured}, which applies an \texttt{argmax} operation at the final timestep.
%===========================================

\paragraph{Shared Aspects.} Across all experiments, we use the AdamW optimizer with fixed \texttt{beta} values of $0.95$ and $0.99$. For the high-dimensional Gaussian benchmark (\secref{section:hd_experiments}). Notably, for diffusion-based methods, we fully sample the Markov chain, in contrast to \citet{austin2021structured}, which applies an \texttt{argmax} operation at the final timestep. 

\paragraph{CSBM and $\alpha$-CSBM.}  
For CSBM and $\alpha$-CSBM, we use the official implementation from \citet{ksenofontov2025categorical}:
\begin{center}
    \url{https://github.com/gregkseno/csbm}.
\end{center}

To stabilize training and improve final performance, we apply Exponential Moving Average (EMA) parameter updates with a decay rate of $0.999$, tuned consistently across all experiments. Unlike \citet{austin2021structured}, we omit the $L_{\text{simple}}$ loss during training. We employ a simple MLP with three hidden layers of size $[128, 128, 128]$ and ReLU activations. Time conditioning is implemented via an embedding layer of the same size as dimensions, $D$. Both methods are trained for $5$ D-IMF iterations, using $120\ 000$ gradient updates in the first iteration and $40\ 000$ in each subsequent iteration. For $\alpha$-CSBM, we use a learning rate of $10^{-3}$ and halve the batch size for training a single model, following \citet{de2024schr}. For CSBM, we use a learning rate of $10^{-4}$. The D-IMF procedure for both solvers is initialized using an independent joint distribution ($q^0(x_0, x_1) = p_0(x_0)p_1(x_1)$).

%===========DO NOT DELETE==============
%For the \textbf{image denoising benchmark} (\secref{section:image_denoising_experiments}), we follow \citet{austin2021structured} and use a PixelCNN++ \citep{salimans2016improved}-based U-Net \citep{ronneberger2015u} with a ResNet-like structure. The model operates at four feature-map resolutions, with two residual convolutional blocks per resolution and a channel multiplier of $(1, 2, 2, 2)$. A self-attention block is inserted at the $16 \times 16$ resolution level. Time conditioning is implemented by applying Transformer sinusoidal positional embeddings at each residual block. Both methods are trained for $4$ D-IMF iterations, with $200\ 000$ gradient updates in the first iteration and $20\ 000$ in each subsequent one. Training uses a learning rate of $10^{-3}$ with mini-batch joint distribution.
%===========DO NOT DELETE==============

\paragraph{DLightSB and DLightSB-M.} For all benchmark experiments, both methods use $K=1000$ components initialized from data samples and are trained for $100\ 000$ gradient updates. The learning rate is set to $10^{-2}$ for both, with DLightSB-M using independent joint distribution \mbox{($q^0(x_0, x_1) = p_0(x_0)p_1(x_1)$)}.

%\xavier{Gregory, please add more details about the use of independent joint distribution here}

\paragraph{Computational Resources and Training Time.} All high-dimensional Gaussian mixture benchmark experiments were conducted on 1 A100 GPU unless otherwise specified, with training times reported inclusive of evaluation. For $D=2$, training is relatively short: CSBM and $\alpha$-CSBM each complete within about $5$ hours, DLightSB-M within $4$ hours, and DLightSB in roughly $20$ minutes. For $D=64$, CSBM completes in under $14$ hours, $\alpha$-CSBM in under $9$ hours, DLightSB-M in just under 2 days (on 2 A100 GPUs), and DLightSB in under $7$ hours. 

\section{Additional Experiments}
\label{appendix:additional_experiments}

\subsection{Reverse benchmark}
\label{appendix:reverse_benchmark}
In this section, we try to overcome the inherited inductive bias of DLightSB(-M) solvers. By construction, the forward conditional distribution $q^*(x_1 | x_0)$ admits a CP decomposition, while the reverse distribution $q^*(x_0 | x_1)$ does not. As a result, when the benchmark is used in the reverse direction with the same marginals $p_0$ and $p_1$, DLightSB(-M) methods can no longer rely on the inductive bias that benefits them in the forward setup.

Unfortunately, in this setup, the true conditional distributions are not available, so we cannot compute conditional metrics. To overcome this restriction, we decided to compute the Classifier Two Sample Test \citep[C2ST]{lopez-paz2017revisiting} metric, ROC AUC of classifier between pairs $(x_0, x_1) \sim p_1(x_1)q^*(x_0 | x_1)$ and $(\hat{x}_0, x_1) \sim p_1(x_1)q_\theta(x_0 | x_1)$. As the classifier, we used two layer MLP with ReLU activations that takes as input the concatenation of one-hot vectors of $x_0$ and $x_1$. We present C2ST scores in the following table.

\begin{table*}[h]
    \resizebox{\textwidth}{!}{%
    \begin{tabular}{rcccccccccccccc}
        \toprule
        & & & \multicolumn{4}{c}{$D\!=\!2$} & \multicolumn{4}{c}{$D\!=\!16$} & \multicolumn{4}{c}{$D\!=\!64$} \\
        \cmidrule(lr){4-7} \cmidrule(lr){8-11} \cmidrule(l){12-15}
        & & & \multicolumn{2}{c}{gaussian} & \multicolumn{2}{c}{uniform} & \multicolumn{2}{c}{gaussian} & \multicolumn{2}{c}{uniform} & \multicolumn{2}{c}{gaussian} & \multicolumn{2}{c}{uniform} \\
        \cmidrule(lr){4-5} \cmidrule(lr){6-7} \cmidrule(l){8-9} \cmidrule(lr){10-11} \cmidrule(lr){12-13} \cmidrule(l){14-15}
        Method & Loss & $N+1$ & {$0.02$} & {$0.05$} & {$0.005$} & {$0.01$} & {$0.02$} & {$0.05$} & {$0.005$} & {$0.01$} & {$0.02$} & {$0.05$} & {$0.005$} & {$0.01$} \\
        \midrule
        \multirow{1}{*}{DLightSB} & -- & -- & \scorecell{0.926} & \scorecell{0.998} & \scorecell{0.996} & \scorecell{0.985} & \scorecell{0.961} & \scorecell{0.971} & \scorecell{0.993} & \scorecell{0.996} & \scorecell{0.972} & \scorecell{0.993} & \scorecell{0.985} & \scorecell{0.990} \\
        \midrule
        \multirow{4}{*}{CSBM} & \multirow{2}{*}{KL} & 16 & \scorecell{0.990} & \scorecell{0.991} & \scorecell{1.000} & \scorecell{0.996} & \scorecell{0.979} & \scorecell{0.990} & \scorecell{0.999} & \scorecell{0.988} & \scorecell{0.990} & \scorecell{0.990} & \scorecell{0.991} & \scorecell{0.997} \\
         &  & 64 & \scorecell{0.995} & \scorecell{1.000} & \scorecell{0.992} & \scorecell{0.998} & \scorecell{0.991} & \scorecell{0.982} & \scorecell{0.986} & \scorecell{0.981} & \scorecell{0.999} & \scorecell{0.999} & \scorecell{0.994} & \scorecell{0.999} \\
         & \multirow{2}{*}{MSE} & 16 & \scorecell{0.952} & \scorecell{0.996} & \scorecell{0.987} & \scorecell{0.997} & \scorecell{0.998} & \scorecell{0.976} & \scorecell{0.995} & \scorecell{0.985} & \scorecell{0.987} & \scorecell{0.997} & \scorecell{0.983} & \scorecell{0.999} \\
         &  & 64 & \scorecell{0.900} & \scorecell{0.990} & \scorecell{0.993} & \scorecell{0.981} & \scorecell{0.985} & \scorecell{0.992} & \scorecell{0.998} & \scorecell{0.973} & \scorecell{0.987} & \scorecell{0.997} & \scorecell{1.000} & \scorecell{0.999} \\
        \bottomrule
    \end{tabular}}
    \caption{\centering C2ST metric ($\uparrow$) on our high-dimensional Gaussian mixture benchmark. Color code: vermillion for $<0.5$, orange for $[0.5,0.75)$, yellow for $[0.75,0.85)$, and bluish-green for $\geq 0.85$.}
    \label{table:c2st_scores}
\end{table*}

As can be seen from Table~\ref{table:c2st_scores}, computed metric values are not informative. Across all methods, the metric values are nearly identical, indicating that the classifier is not capable of distinguishing generated samples from real ones. As a result, we decided to discard this metric.

\subsection{Additional Metrics and Plots}
\label{appendix:additional_experiments_main}

\begin{table*}
    \resizebox{\columnwidth}{!}{%
    \begin{tabular}{l c ccccccccccccc}
        \toprule
        \multicolumn{3}{c}{} & \multicolumn{4}{c}{$D=2$} & \multicolumn{4}{c}{$D=16$} & \multicolumn{4}{c}{$D=64$} \\
        \cmidrule(lr){4-7} \cmidrule(lr){8-11} \cmidrule(lr){12-15}
        \multicolumn{3}{c}{} & \multicolumn{2}{c}{gaussian} & \multicolumn{2}{c}{uniform} & \multicolumn{2}{c}{gaussian} & \multicolumn{2}{c}{uniform} & \multicolumn{2}{c}{gaussian} & \multicolumn{2}{c}{uniform} \\
        \cmidrule(lr){4-5} \cmidrule(lr){6-7} \cmidrule(lr){8-9} \cmidrule(lr){10-11} \cmidrule(lr){12-13} \cmidrule(lr){14-15}
        Method & Loss & $N+1$ & $0.02$ & $0.05$ & $0.005$ & $0.01$ & $0.02$ & $0.05$ & $0.005$ & $0.01$ & $0.02$ & $0.05$ & $0.005$ & $0.01$ \\
        \midrule 
        \multirow{1}{*}{\textit{Reference}} & -- & -- & 
        \scoreklcell{0.501} & \scoreklcell{0.637} & \scoreklcell{0.272} & \scoreklcell{0.509} & \underline{\scoreklcell{0.526}} & \scoreklcell{0.666} & \scoreklcell{0.301} & \scoreklcell{0.527} &  \scoreklcell{0.502} & \scoreklcell{0.637} & \scoreklcell{0.278} & \scoreklcell{0.498} \\
        \multirow{1}{*}{\textit{Feature-wise SB}} & -- & -- & 
        \underline{\scoreklcell{0.036}} & \underline{\scoreklcell{0.029}} & \underline{\scoreklcell{0.019}} & \underline{\scoreklcell{0.014}} & \textbf{\scoreklcell{0.037}} & \underline{\scoreklcell{0.036}} & \underline{\scoreklcell{0.058}} & \underline{\scoreklcell{0.057}} & \textbf{\scoreklcell{0.006}} & \textbf{\scoreklcell{0.042}} & \textbf{\scoreklcell{0.037}} & \textbf{\scoreklcell{0.04}} \\
        \midrule
        \multirow{1}{*}{DLightSB} & \multirow{1}{*}{--} & -- & 
        \textbf{\scoreklcell{0.0}} & \textbf{\scoreklcell{0.0}} & \textbf{\scoreklcell{0.0}} & \textbf{\scoreklcell{0.0}} & \scoreklcell{1.8} & \textbf{\scoreklcell{0.0}} & \textbf{\scoreklcell{0.0}} & \textbf{\scoreklcell{0.0}} & \underline{\scoreklcell{0.3}} & \underline{\scoreklcell{0.1}} & \underline{\scoreklcell{0.1}} & \underline{\scoreklcell{0.1}} \\
        \midrule
        \multirow{4}{*}{CSBM} & \multirow{2}{*}{KL} & 16 & \scoreklcell{8.0} & \scoreklcell{32.1} & \scoreklcell{4.7} & \scoreklcell{8.1} & \scoreklcell{18.4} & \scoreklcell{189.1} & \scoreklcell{70.8} & \scoreklcell{104.9} & \scoreklcell{16.8} & \scoreklcell{22.0} & \scoreklcell{15.2} & \scoreklcell{14.1} \\
         &  & 64 & \scoreklcell{1.5} & \scoreklcell{7.4} & \scoreklcell{1.0} & \scoreklcell{1.5} & \scoreklcell{2.8} & \scoreklcell{21.0} & \scoreklcell{7.9} & \scoreklcell{8.8} & \scoreklcell{22.8} & \scoreklcell{2.3} & \scoreklcell{4.0} & \scoreklcell{3.6} \\
         & \multirow{2}{*}{MSE} & 16 & \scoreklcell{14.0} & \scoreklcell{23.8} & \scoreklcell{7.9} & \scoreklcell{7.4} & \scoreklcell{28.7} & \scoreklcell{182.9} & \scoreklcell{108.0} & \scoreklcell{156.8} & \scoreklcell{47.5} & \scoreklcell{24.9} & \scoreklcell{36.8} & \scoreklcell{25.8} \\
         &  & 64 & \scoreklcell{9.5} & \scoreklcell{10.3} & \scoreklcell{3.2} & \scoreklcell{3.1} & \scoreklcell{6.7} & \scoreklcell{45.0} & \scoreklcell{25.2} & \scoreklcell{32.1} & \scoreklcell{59.3} & \scoreklcell{2.1} & \scoreklcell{14.8} & \scoreklcell{10.9} \\
        \midrule
        \multirow{4}{*}{$\alpha$-CSBM} & \multirow{2}{*}{KL} & 16 & \scoreklcell{6.4} & \scoreklcell{23.7} & \scoreklcell{4.2} & \scoreklcell{6.4} & \scoreklcell{13.3} & \scoreklcell{103.4} & \scoreklcell{61.1} & \scoreklcell{81.5} & \scoreklcell{26.1} & \scoreklcell{3.5} & \scoreklcell{4.9} & \scoreklcell{5.2} \\
         &  & 64 & \scoreklcell{1.3} & \scoreklcell{4.9} & \scoreklcell{0.7} & \scoreklcell{0.7} & \scoreklcell{3.5} & \scoreklcell{1.1} & \scoreklcell{2.1} & \scoreklcell{2.2} & \scoreklcell{24.5} & \scoreklcell{2.1} & \scoreklcell{4.5} & \scoreklcell{4.6} \\
         & \multirow{2}{*}{MSE} & 16 & \scoreklcell{7.6} & \scoreklcell{17.5} & \scoreklcell{5.6} & \scoreklcell{5.2} & \scoreklcell{16.1} & \scoreklcell{112.3} & \scoreklcell{52.8} & \scoreklcell{72.3} & \scoreklcell{46.2} & \scoreklcell{3.4} & \scoreklcell{17.9} & \scoreklcell{12.0} \\
         &  & 64 & \scoreklcell{1.0} & \scoreklcell{3.1} & \scoreklcell{0.9} & \scoreklcell{0.8} & \scoreklcell{9.4} & \scoreklcell{4.7} & \scoreklcell{5.6} & \scoreklcell{4.4} & \scoreklcell{58.8} & \scoreklcell{2.0} & \scoreklcell{15.5} & \scoreklcell{11.0} \\
        \midrule
        \multirow{4}{*}{DLightSB-M} & \multirow{2}{*}{KL} & 16 & \scoreklcell{0.1} & \textbf{\scoreklcell{0.0}} & \textbf{\scoreklcell{0.0}} & \textbf{\scoreklcell{0.0}} & \scoreklcell{3.2} & \scoreklcell{0.1} & \scoreklcell{0.2} & \scoreklcell{0.2} & \scoreklcell{6.2} & \scoreklcell{1.3} & \scoreklcell{26.4} & \scoreklcell{22.3} \\
         &  & 64 & \scoreklcell{0.2} & \textbf{\scoreklcell{0.0}} & \textbf{\scoreklcell{0.0}} & \textbf{\scoreklcell{0.0}} & \scoreklcell{1.2} & \scoreklcell{0.2} & \scoreklcell{0.4} & \scoreklcell{0.3} & \scoreklcell{6.6} & \scoreklcell{0.8} & \scoreklcell{22.9} & \scoreklcell{7.3} \\
         & \multirow{2}{*}{MSE} & 16 & \scoreklcell{0.8} & \textbf{\scoreklcell{0.0}} & \scoreklcell{0.2} & \textbf{\scoreklcell{0.0}} & \scoreklcell{22.0} & \scoreklcell{0.2} & \scoreklcell{0.5} & \scoreklcell{0.3} & \scoreklcell{26.3} & \scoreklcell{1.4} & \scoreklcell{8.8} & \scoreklcell{48.3} \\
         &  & 64 & \scoreklcell{1.0} & \textbf{\scoreklcell{0.0}} & \scoreklcell{0.2} & \scoreklcell{0.1} & \scoreklcell{24.8} & \scoreklcell{0.3} & \scoreklcell{0.6} & \scoreklcell{0.4} & \scoreklcell{58.8} & \scoreklcell{2.3} & \scoreklcell{41.0} & \scoreklcell{17.9} \\
        \bottomrule
    \end{tabular}}
    \caption{\centering Trajectory KL divergence ($\downarrow$) on the high-dimensional Gaussian mixture benchmark. The best-performing method is highlighted in bold, and the second is underlined. Color code: bluish-green for $<0.5$, yellow for $[0.5,2)$, orange for $[2,10)$, and vermillion for $\geq 10$.
}
\label{table:forward_kl}
\end{table*}

\begin{table*}
    \resizebox{\columnwidth}{!}{%
    \begin{tabular}{l c ccccccccccccc}
        \toprule
        \multicolumn{3}{c}{} & \multicolumn{4}{c}{$D=2$} & \multicolumn{4}{c}{$D=16$} & \multicolumn{4}{c}{$D=64$} \\
        \cmidrule(lr){4-7} \cmidrule(lr){8-11} \cmidrule(lr){12-15}
        \multicolumn{3}{c}{} & \multicolumn{2}{c}{gaussian} & \multicolumn{2}{c}{uniform} & \multicolumn{2}{c}{gaussian} & \multicolumn{2}{c}{uniform} & \multicolumn{2}{c}{gaussian} & \multicolumn{2}{c}{uniform} \\
        \cmidrule(lr){4-5} \cmidrule(lr){6-7} \cmidrule(lr){8-9} \cmidrule(lr){10-11} \cmidrule(lr){12-13} \cmidrule(lr){14-15}
        Method & Loss & $N+1$ & $0.02$ & $0.05$ & $0.005$ & $0.01$ & $0.02$ & $0.05$ & $0.005$ & $0.01$ & $0.02$ & $0.05$ & $0.005$ & $0.01$ \\
        \midrule
        \multirow{1}{*}{\textit{Reference}} & -- & -- & 
        \scoreklcell{2.1} & \scoreklcell{3.0} & \scoreklcell{0.8} & \scoreklcell{1.3} & \underline{\scoreklcell{2.2}} & \scoreklcell{4.6} & \scoreklcell{1.4} & \scoreklcell{1.5} & \scoreklcell{1.4} & \scoreklcell{3.5} & \scoreklcell{1.0} & \scoreklcell{1.2} \\
        \multirow{1}{*}{\textit{Feature-wise SB}} & -- & -- & 
        \underline{\scoreklcell{0.2}} & \scoreklcell{0.2} & \textbf{\scoreklcell{0.0}} & \textbf{\scoreklcell{0.0}} & \textbf{\scoreklcell{0.2}} & \underline{\scoreklcell{0.5}} & \underline{\scoreklcell{0.3}} & \underline{\scoreklcell{0.3}} & \textbf{\scoreklcell{0.0}} & \underline{\scoreklcell{0.4}} & \underline{\scoreklcell{0.3}} & \underline{\scoreklcell{0.4}} \\
        \midrule
        \multirow{1}{*}{DLightSB} & \multirow{1}{*}{--} & -- & 
        \textbf{\scoreklcell{0.0}} & \textbf{\scoreklcell{0.0}} & \textbf{\scoreklcell{0.0}} & \textbf{\scoreklcell{0.0}} & \scoreklcell{5.7} & \textbf{\scoreklcell{0.0}} & \textbf{\scoreklcell{0.0}} & \textbf{\scoreklcell{0.0}} & \underline{\scoreklcell{0.3}} & \textbf{\scoreklcell{0.1}} & \textbf{\scoreklcell{0.1}} & \textbf{\scoreklcell{0.1}} \\
        \midrule
        \multirow{4}{*}{CSBM} & \multirow{2}{*}{KL} & 16 & \scoreklcell{1.5} & \scoreklcell{3.7} & \scoreklcell{1.0} & \scoreklcell{1.1} & \scoreklcell{7.7} & \scoreklcell{22.1} & \scoreklcell{50.9} & \scoreklcell{71.4} & \scoreklcell{15.8} & \scoreklcell{13.0} & \scoreklcell{16.6} & \scoreklcell{20.2} \\
         &  & 64 & \scoreklcell{0.4} & \scoreklcell{1.8} & \scoreklcell{0.4} & \scoreklcell{0.4} & \scoreklcell{2.7} & \scoreklcell{7.1} & \scoreklcell{10.3} & \scoreklcell{10.8} & \scoreklcell{22.3} & \scoreklcell{2.5} & \scoreklcell{11.3} & \scoreklcell{10.8} \\
         & \multirow{2}{*}{MSE} & 16 & \scoreklcell{4.2} & \scoreklcell{2.5} & \scoreklcell{3.3} & \scoreklcell{4.3} & \scoreklcell{8.1} & \scoreklcell{19.5} & \scoreklcell{88.3} & \scoreklcell{122.3} & \scoreklcell{55.4} & \scoreklcell{12.2} & \scoreklcell{41.8} & \scoreklcell{34.7} \\
         &  & 64 & \scoreklcell{10.2} & \scoreklcell{1.8} & \scoreklcell{2.1} & \scoreklcell{2.2} & \scoreklcell{4.4} & \scoreklcell{9.1} & \scoreklcell{22.2} & \scoreklcell{24.4} & \scoreklcell{75.6} & \scoreklcell{2.6} & \scoreklcell{38.5} & \scoreklcell{39.0} \\
        \midrule
        \multirow{4}{*}{$\alpha$-CSBM} & \multirow{2}{*}{KL} & 16 & \scoreklcell{1.7} & \scoreklcell{3.3} & \scoreklcell{0.8} & \scoreklcell{0.9} & \scoreklcell{8.1} & \scoreklcell{17.5} & \scoreklcell{41.8} & \scoreklcell{53.8} & \scoreklcell{23.6} & \scoreklcell{5.0} & \scoreklcell{20.5} & \scoreklcell{27.3} \\
         &  & 64 & \scoreklcell{0.7} & \scoreklcell{1.4} & \scoreklcell{0.4} & \scoreklcell{0.3} & \scoreklcell{4.4} & \scoreklcell{1.0} & \scoreklcell{7.3} & \scoreklcell{10.7} & \scoreklcell{22.7} & \scoreklcell{3.0} & \scoreklcell{17.1} & \scoreklcell{16.2} \\
         & \multirow{2}{*}{MSE} & 16 & \scoreklcell{2.7} & \scoreklcell{3.0} & \scoreklcell{1.7} & \scoreklcell{1.1} & \scoreklcell{9.6} & \scoreklcell{18.5} & \scoreklcell{52.5} & \scoreklcell{75.4} & \scoreklcell{54.1} & \scoreklcell{5.6} & \scoreklcell{83.6} & \scoreklcell{89.5} \\
         &  & 64 & \scoreklcell{0.5} & \scoreklcell{1.1} & \scoreklcell{1.4} & \scoreklcell{1.3} & \scoreklcell{11.4} & \scoreklcell{2.8} & \scoreklcell{16.9} & \scoreklcell{20.0} & \scoreklcell{75.9} & \scoreklcell{3.3} & \scoreklcell{46.7} & \scoreklcell{48.3} \\
        \midrule
        \multirow{4}{*}{DLightSB-M} & \multirow{2}{*}{KL} & 16 & \scoreklcell{0.5} & \textbf{\scoreklcell{0.0}} & \underline{\scoreklcell{0.1}} & \underline{\scoreklcell{0.2}} & \scoreklcell{21.7} & \scoreklcell{1.0} & \scoreklcell{7.4} & \scoreklcell{5.9} & \scoreklcell{6.4} & \scoreklcell{2.3} & \scoreklcell{322.2} & \scoreklcell{329.7} \\
         &  & 64 & \scoreklcell{0.7} & \underline{\scoreklcell{0.1}} & \scoreklcell{0.3} & \scoreklcell{0.4} & \scoreklcell{8.2} & \scoreklcell{1.4} & \scoreklcell{12.3} & \scoreklcell{11.7} & \scoreklcell{6.8} & \scoreklcell{1.9} & \scoreklcell{292.6} & \scoreklcell{110.7} \\
         & \multirow{2}{*}{MSE} & 16 & \scoreklcell{1.1} & \scoreklcell{0.3} & \scoreklcell{0.3} & \underline{\scoreklcell{0.2}} & \scoreklcell{182.3} & \scoreklcell{4.5} & \scoreklcell{6.5} & \scoreklcell{6.6} & \scoreklcell{25.4} & \scoreklcell{15.8} & \scoreklcell{70.3} & \scoreklcell{405.0} \\
         &  & 64 & \scoreklcell{1.4} & \scoreklcell{0.5} & \scoreklcell{0.4} & \underline{\scoreklcell{0.2}} & \scoreklcell{156.4} & \scoreklcell{8.2} & \scoreklcell{8.4} & \scoreklcell{8.9} & \scoreklcell{52.4} & \scoreklcell{27.6} & \scoreklcell{313.8} & \scoreklcell{104.1} \\
        \bottomrule
    \end{tabular}}
    \caption{\centering Trajectory reverse KL divergence ($\downarrow$) on the high-dimensional Gaussian mixture benchmark. The best-performing method is highlighted in bold, and the second is underlined. Color code: bluish-green for $<0.5$, yellow for $[0.5,2)$, orange for $[2,10)$, and vermillion for $\geq 10$.}
\label{table:reverse_kl}
\end{table*}
\begin{table*}[h]
    \resizebox{\textwidth}{!}{%
    \begin{tabular}{rcccccccccccccc}
        \toprule
        & & & \multicolumn{4}{c}{$D\!=\!2$} & \multicolumn{4}{c}{$D\!=\!16$} & \multicolumn{4}{c}{$D\!=\!64$} \\
        \cmidrule(lr){4-7} \cmidrule(lr){8-11} \cmidrule(l){12-15}
        & & & \multicolumn{2}{c}{gaussian} & \multicolumn{2}{c}{uniform} & \multicolumn{2}{c}{gaussian} & \multicolumn{2}{c}{uniform} & \multicolumn{2}{c}{gaussian} & \multicolumn{2}{c}{uniform} \\
        \cmidrule(lr){4-5} \cmidrule(lr){6-7} \cmidrule(l){8-9} \cmidrule(lr){10-11} \cmidrule(lr){12-13} \cmidrule(l){14-15}
        Method & Loss & $N\!+\!1$ & {$0.02$} & {$0.05$} & {$0.005$} & {$0.01$} & {$0.02$} & {$0.05$} & {$0.005$} & {$0.01$} & {$0.02$} & {$0.05$} & {$0.005$} & {$0.01$} \\
        \midrule
        \multirow{1}{*}{\textit{Independent}} & -- & -- &
        \textbf{\scorecell{0.983}} & \textbf{\scorecell{0.985}} & \textbf{\scorecell{0.979}} & \textbf{\scorecell{0.984}} &
        \textbf{\scorecell{0.987}} & \textbf{\scorecell{0.985}} & \textbf{\scorecell{0.979}} & \textbf{\scorecell{0.982}} &
        \textbf{\scorecell{0.986}} & \textbf{\scorecell{0.981}} & \textbf{\scorecell{0.980}} & \textbf{\scorecell{0.980}} \\
        \multirow{1}{*}{\textit{Reference}} & -- & -- &
        \scorecell{0.324} & \scorecell{0.481} & \scorecell{0.411} & \scorecell{0.431} &
        \scorecell{0.406} & \scorecell{0.337} & \scorecell{0.410} & \scorecell{0.420} &
        \scorecell{0.560} & \scorecell{0.491} & \scorecell{0.476} & \scorecell{0.467} \\
        \multirow{1}{*}{\textit{Feature-wise SB}} & -- & -- &
        \scorecell{0.924} & \textbf{\scorecell{0.985}} & \textbf{\scorecell{0.979}} & \textbf{\scorecell{0.984}} &
        \underline{\scorecell{0.922}} & \underline{\scorecell{0.984}} & \textbf{\scorecell{0.982}} & \textbf{\scorecell{0.983}} &
        \underline{\scorecell{0.975}} & \textbf{\scorecell{0.981}} & \textbf{\scorecell{0.982}} & \textbf{\scorecell{0.981}} \\
        \midrule
        \multirow{1}{*}{DLightSB} & \multirow{1}{*}{--} & -- &
        \underline{\scorecell{0.974}} & \underline{\scorecell{0.971}} & \textbf{\scorecell{0.978}} & \textbf{\scorecell{0.979}} &
        \scorecell{0.891} & \scorecell{0.973} & \underline{\scorecell{0.955}} & \textbf{\scorecell{0.976}} &
        \scorecell{0.972} & \scorecell{0.952} & \underline{\scorecell{0.970}} & \underline{\scorecell{0.970}} \\
        \midrule
        \multirow{4}{*}{CSBM} & \multirow{2}{*}{KL} & 16 &
        \scorecell{0.768} & \scorecell{0.724} & \scorecell{0.915} & \scorecell{0.913} &
        \scorecell{0.824} & \scorecell{0.776} & \scorecell{0.853} & \scorecell{0.789} &
        \scorecell{0.907} & \scorecell{0.895} & \scorecell{0.840} & \scorecell{0.886} \\
         &  & 64 &
        \scorecell{0.914} & \scorecell{0.886} & \scorecell{0.955} & \underline{\scorecell{0.965}} &
        \scorecell{0.906} & \scorecell{0.845} & \scorecell{0.921} & \scorecell{0.931} &
        \scorecell{0.941} & \scorecell{0.941} & \scorecell{0.793} & \scorecell{0.883} \\
         & \multirow{2}{*}{MSE} & 16 &
        \scorecell{0.519} & \scorecell{0.721} & \scorecell{0.843} & \scorecell{0.858} &
        \scorecell{0.824} & \scorecell{0.743} & \scorecell{0.803} & \scorecell{0.744} &
        \scorecell{0.813} & \scorecell{0.886} & \scorecell{0.808} & \scorecell{0.810} \\
         &  & 64 &
        \scorecell{0.374} & \scorecell{0.836} & \scorecell{0.835} & \scorecell{0.810} &
        \scorecell{0.875} & \scorecell{0.831} & \scorecell{0.838} & \scorecell{0.900} &
        \scorecell{0.743} & \underline{\scorecell{0.955}} & \scorecell{0.834} & \scorecell{0.824} \\
        \midrule
        \multirow{4}{*}{$\alpha$-CSBM} & \multirow{2}{*}{KL} & 16 &
        \scorecell{0.738} & \scorecell{0.749} & \scorecell{0.930} & \scorecell{0.914} &
        \scorecell{0.788} & \scorecell{0.812} & \scorecell{0.852} & \scorecell{0.781} &
        \scorecell{0.928} & \scorecell{0.892} & \scorecell{0.807} & \scorecell{0.869} \\
         &  & 64 &
        \scorecell{0.820} & \scorecell{0.889} & \scorecell{0.956} & \underline{\scorecell{0.965}} &
        \underline{\scorecell{0.923}} & \scorecell{0.925} & \scorecell{0.930} & \scorecell{0.950} &
        \scorecell{0.947} & \scorecell{0.928} & \scorecell{0.803} & \scorecell{0.863} \\
         & \multirow{2}{*}{MSE} & 16 &
        \scorecell{0.692} & \scorecell{0.727} & \scorecell{0.892} & \scorecell{0.901} &
        \scorecell{0.794} & \scorecell{0.797} & \scorecell{0.847} & \scorecell{0.827} &
        \scorecell{0.816} & \scorecell{0.932} & \scorecell{0.830} & \scorecell{0.820} \\
         &  & 64 &
        \scorecell{0.917} & \scorecell{0.884} & \scorecell{0.886} & \scorecell{0.925} &
        \scorecell{0.797} & \scorecell{0.902} & \scorecell{0.872} & \scorecell{0.895} &
        \scorecell{0.738} & \scorecell{0.932} & \scorecell{0.820} & \scorecell{0.848} \\
        \midrule
        \multirow{4}{*}{DLightSB-M} & \multirow{2}{*}{KL} & 16 &
        \scorecell{0.901} & \scorecell{0.964} & \underline{\scorecell{0.966}} & \underline{\scorecell{0.970}} &
        \scorecell{0.861} & \scorecell{0.931} & \underline{\scorecell{0.957}} & \underline{\scorecell{0.957}} &
        \scorecell{0.893} & \scorecell{0.800} & \scorecell{0.828} & \scorecell{0.860} \\
         &  & 64 &
        \scorecell{0.883} & \scorecell{0.953} & \scorecell{0.951} & \scorecell{0.963} &
        \scorecell{0.899} & \scorecell{0.949} & \scorecell{0.952} & \underline{\scorecell{0.963}} &
        \scorecell{0.887} & \scorecell{0.899} & \scorecell{0.804} & \scorecell{0.926} \\
         & \multirow{2}{*}{MSE} & 16 &
        \scorecell{0.747} & \scorecell{0.949} & \scorecell{0.849} & \scorecell{0.928} &
        \scorecell{0.691} & \scorecell{0.956} & \scorecell{0.928} & \scorecell{0.918} &
        \scorecell{0.722} & \scorecell{0.881} & \scorecell{0.891} & \scorecell{0.787} \\
         &  & 64 &
        \scorecell{0.728} & \scorecell{0.949} & \scorecell{0.812} & \scorecell{0.916} &
        \scorecell{0.648} & \scorecell{0.947} & \scorecell{0.939} & \scorecell{0.885} &
        \scorecell{0.624} & \scorecell{0.885} & \scorecell{0.675} & \scorecell{0.832} \\
        \bottomrule
        \end{tabular}}
        \caption{ \centering Shape Score metric ($\uparrow$) on our high-dimensional Gaussian mixture benchmark. The best-performing method is highlighted in bold, and the second is underlined. Color code: vermillion for $<0.5$, orange for $[0.5,0.75)$, yellow for $[0.75,0.85)$, and bluish-green for $\geq 0.85$.}
    \label{table:unconditional_shape_score}
\end{table*}

\begin{table*}[h]
    \resizebox{\textwidth}{!}{%
    \begin{tabular}{rcccccccccccccc}
        \toprule
        & & & \multicolumn{4}{c}{$D\!=\!2$} & \multicolumn{4}{c}{$D\!=\!16$} & \multicolumn{4}{c}{$D\!=\!64$} \\
        \cmidrule(lr){4-7} \cmidrule(lr){8-11} \cmidrule(l){12-15}
        & & & \multicolumn{2}{c}{gaussian} & \multicolumn{2}{c}{uniform} & \multicolumn{2}{c}{gaussian} & \multicolumn{2}{c}{uniform} & \multicolumn{2}{c}{gaussian} & \multicolumn{2}{c}{uniform} \\
        \cmidrule(lr){4-5} \cmidrule(lr){6-7} \cmidrule(l){8-9} \cmidrule(lr){10-11} \cmidrule(lr){12-13} \cmidrule(l){14-15}
        Method & Loss & $N\!+\!1$ & {$0.02$} & {$0.05$} & {$0.005$} & {$0.01$} & {$0.02$} & {$0.05$} & {$0.005$} & {$0.01$} & {$0.02$} & {$0.05$} & {$0.005$} & {$0.01$} \\
        \midrule
        \multirow{1}{*}{\textit{Independent}} & -- & -- &
        \textbf{\scorecell{0.968}} & \textbf{\scorecell{0.969}} & \textbf{\scorecell{0.961}} & \textbf{\scorecell{0.964}} &
        \textbf{\scorecell{0.961}} & \textbf{\scorecell{0.957}} & \textbf{\scorecell{0.943}} & \textbf{\scorecell{0.948}} &
        \textbf{\scorecell{0.948}} & \textbf{\scorecell{0.923}} & \textbf{\scorecell{0.920}} & \textbf{\scorecell{0.920}} \\ 
        \multirow{1}{*}{\textit{Reference}} & -- & -- &
        \scorecell{0.173} & \scorecell{0.303} & \scorecell{0.247} & \scorecell{0.292} &
        \scorecell{0.194} & \scorecell{0.118} & \scorecell{0.137} & \scorecell{0.138} &
        \scorecell{0.369} & \scorecell{0.233} & \scorecell{0.232} & \scorecell{0.204}  \\
        \multirow{1}{*}{\textit{Feature-wise SB}} & -- & -- &
        \scorecell{0.713} & \scorecell{0.699} & \scorecell{0.781} & \scorecell{0.817} &
        \scorecell{0.806} & \scorecell{0.663} & \scorecell{0.550} & \scorecell{0.509} &
        \scorecell{0.930} & \scorecell{0.670} & \scorecell{0.733} & \scorecell{0.677} \\       
        \midrule
        \multirow{1}{*}{DLightSB} & \multirow{1}{*}{--} & 16 &
        \underline{\scorecell{0.959}} & \underline{\scorecell{0.955}} & \textbf{\scorecell{0.957}} & \textbf{\scorecell{0.958}} &
        \scorecell{0.836} & \underline{\scorecell{0.947}} & \underline{\scorecell{0.925}} & \underline{\scorecell{0.941}} &
        \underline{\scorecell{0.937}} & \underline{\scorecell{0.903}} & \underline{\scorecell{0.913}} & \underline{\scorecell{0.914}} \\
        \midrule
        \multirow{4}{*}{CSBM} & \multirow{2}{*}{KL} & 16 &
        \scorecell{0.696} & \scorecell{0.637} & \scorecell{0.873} & \scorecell{0.869} &
        \scorecell{0.732} & \scorecell{0.660} & \scorecell{0.752} & \scorecell{0.674} &
        \scorecell{0.850} & \scorecell{0.827} & \scorecell{0.781} & \scorecell{0.833} \\
         &  & 64 &
        \scorecell{0.886} & \scorecell{0.851} & \scorecell{0.925} & \scorecell{0.928} &
        \scorecell{0.874} & \scorecell{0.810} & \scorecell{0.873} & \scorecell{0.883} &
        \scorecell{0.900} & \scorecell{0.890} & \scorecell{0.729} & \scorecell{0.833} \\
         & \multirow{2}{*}{MSE} & 16 &
        \scorecell{0.498} & \scorecell{0.645} & \scorecell{0.767} & \scorecell{0.793} &
        \scorecell{0.728} & \scorecell{0.618} & \scorecell{0.686} & \scorecell{0.621} &
        \scorecell{0.722} & \scorecell{0.813} & \scorecell{0.743} & \scorecell{0.756} \\
         &  & 64 &
        \scorecell{0.323} & \scorecell{0.780} & \scorecell{0.761} & \scorecell{0.740} &
        \scorecell{0.813} & \scorecell{0.769} & \scorecell{0.777} & \scorecell{0.832} &
        \scorecell{0.634} & \underline{\scorecell{0.900}} & \scorecell{0.767} & \scorecell{0.766} \\
        \midrule
        \multirow{4}{*}{$\alpha$-CSBM} & \multirow{2}{*}{KL} & 16 &
        \scorecell{0.662} & \scorecell{0.655} & \scorecell{0.889} & \scorecell{0.872} &
        \scorecell{0.705} & \scorecell{0.709} & \scorecell{0.776} & \scorecell{0.696} &
        \scorecell{0.880} & \scorecell{0.827} & \scorecell{0.741} & \scorecell{0.813} \\
         &  & 64 &
        \scorecell{0.805} & \scorecell{0.849} & \scorecell{0.928} & \scorecell{0.934} &
        \underline{\scorecell{0.890}} & \scorecell{0.897} & \scorecell{0.883} & \scorecell{0.910} &
        \scorecell{0.907} & \scorecell{0.875} & \scorecell{0.736} & \scorecell{0.813} \\
         & \multirow{2}{*}{MSE} & 16 &
        \scorecell{0.634} & \scorecell{0.638} & \scorecell{0.834} & \scorecell{0.832} &
        \scorecell{0.713} & \scorecell{0.688} & \scorecell{0.758} & \scorecell{0.728} &
        \scorecell{0.725} & \scorecell{0.869} & \scorecell{0.752} & \scorecell{0.751} \\
         &  & 64 &
        \scorecell{0.887} & \scorecell{0.837} & \scorecell{0.829} & \scorecell{0.862} &
        \scorecell{0.720} & \scorecell{0.866} & \scorecell{0.810} & \scorecell{0.837} &
        \scorecell{0.628} & \scorecell{0.878} & \scorecell{0.748} & \scorecell{0.788} \\
        \midrule
        \multirow{4}{*}{DLightSB-M} & \multirow{2}{*}{KL} & 16 &
        \scorecell{0.870} & \scorecell{0.948} & \underline{\scorecell{0.940}} & \underline{\scorecell{0.949}} &
        \scorecell{0.781} & \scorecell{0.908} & \scorecell{0.912} & \scorecell{0.916} &
        \scorecell{0.831} & \scorecell{0.742} & \scorecell{0.640} & \scorecell{0.669} \\
         &  & 64 &
        \scorecell{0.846} & \scorecell{0.933} & \scorecell{0.925} & \scorecell{0.940} &
        \scorecell{0.842} & \scorecell{0.917} & \scorecell{0.906} & \scorecell{0.915} &
        \scorecell{0.825} & \scorecell{0.848} & \scorecell{0.643} & \scorecell{0.828} \\
         & \multirow{2}{*}{MSE} & 16 &
        \scorecell{0.667} & \scorecell{0.932} & \scorecell{0.804} & \scorecell{0.905} &
        \scorecell{0.541} & \scorecell{0.922} & \scorecell{0.878} & \scorecell{0.878} &
        \scorecell{0.597} & \scorecell{0.826} & \scorecell{0.791} & \scorecell{0.513} \\
         &  & 64 &
        \scorecell{0.645} & \scorecell{0.923} & \scorecell{0.758} & \scorecell{0.887} &
        \scorecell{0.486} & \scorecell{0.907} & \scorecell{0.882} & \scorecell{0.846} &
        \scorecell{0.461} & \scorecell{0.822} & \scorecell{0.483} & \scorecell{0.724} \\
        \bottomrule
    \end{tabular}}
    \caption{\centering  Trend Score ($\uparrow$) on our high-dimensional Gaussian mixture benchmark. The best-performing method is highlighted in bold, and the second is underlined. Color code: vermillion for $<0.5$, orange for $[0.5,0.75)$, yellow for $[0.75,0.85)$, and bluish-green for $\geq 0.85$.}
    \label{table:unconditional_trend_score}
\end{table*}
\begin{figure}[h]
    \centering
    \captionsetup[subfigure]{font=scriptsize}
    \begin{subfigure}[b]{0.19\linewidth}
        \centering
        \includegraphics[width=0.995\linewidth]{images/benchmark/pairs_d2_g002.png}
        \caption{Input and Target}
    \end{subfigure}
    \begin{subfigure}[b]{0.19\linewidth}
        \centering
        \includegraphics[width=0.995\linewidth]{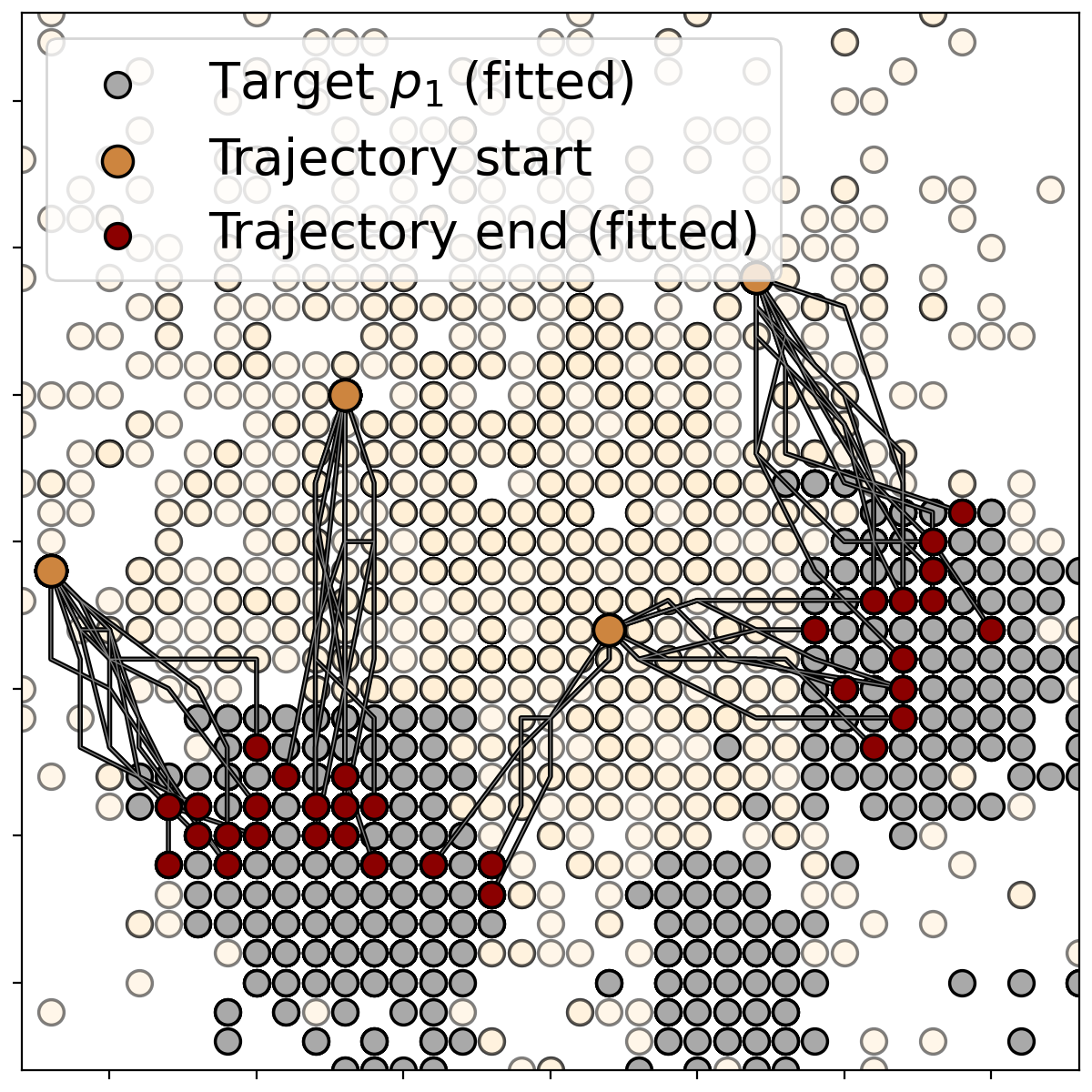}
        \caption{CSBM KL/16}
    \end{subfigure}
    \begin{subfigure}[b]{0.19\linewidth}
        \centering
        \includegraphics[width=0.995\linewidth]{images/csbm/d2_g002_t63_kl.png}
        \caption{CSBM KL/64}
    \end{subfigure}
    \begin{subfigure}[b]{0.19\linewidth}
        \centering
        \includegraphics[width=0.995\linewidth]{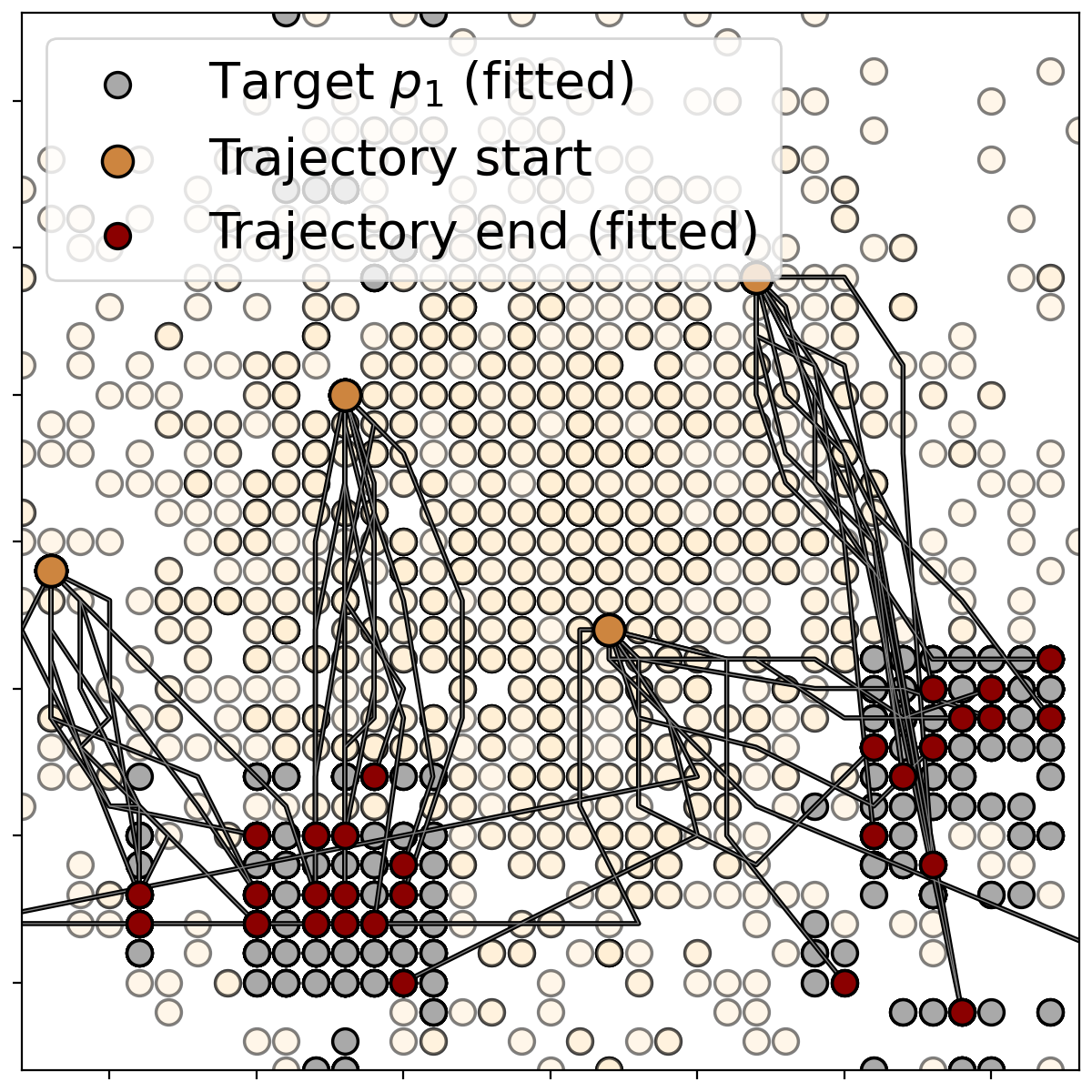}
        \caption{CSBM MSE/16}
    \end{subfigure}
    \begin{subfigure}[b]{0.19\linewidth}
        \centering
        \includegraphics[width=0.995\linewidth]{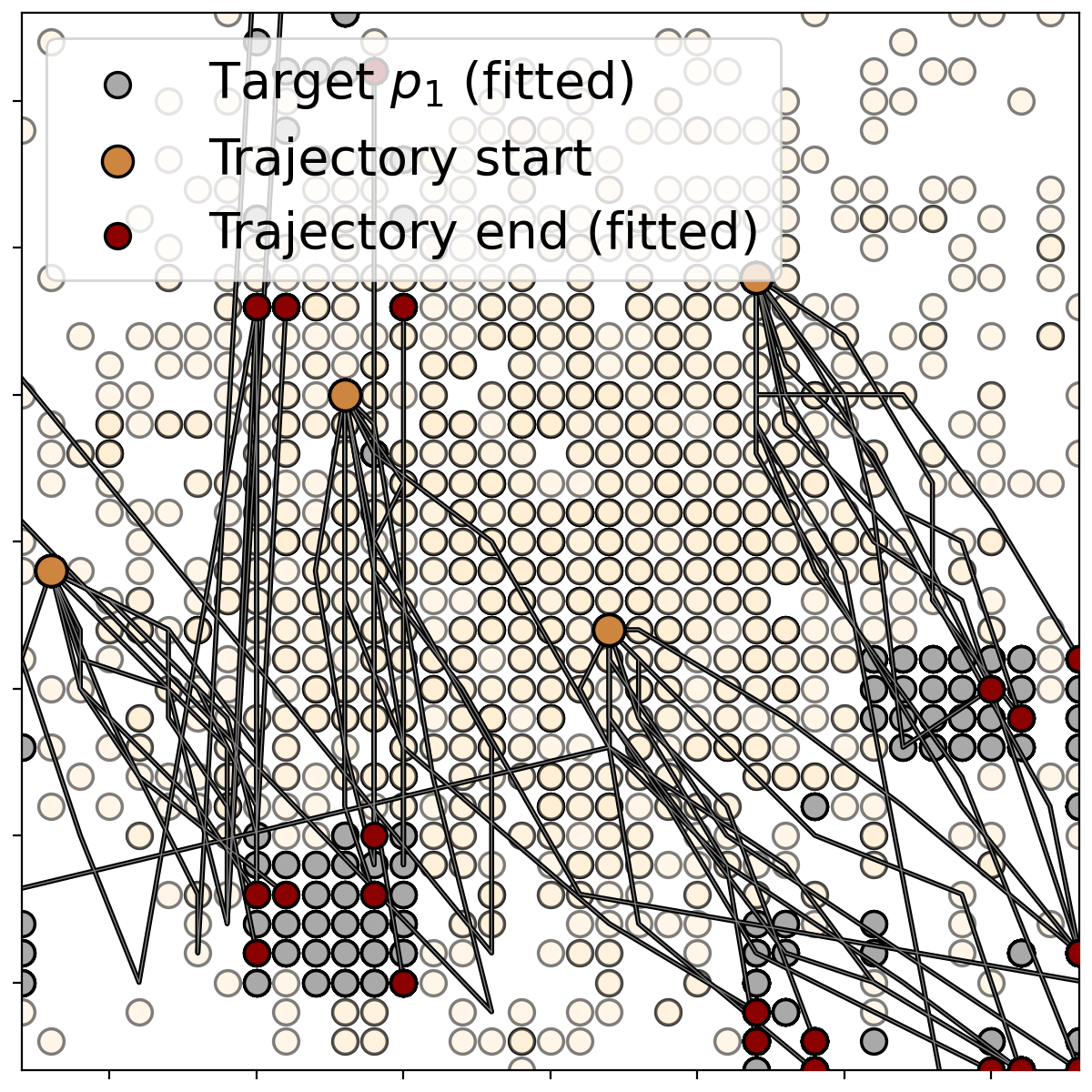}
        \caption{CSBM MSE/64}
    \end{subfigure}
    \\
    \begin{subfigure}[b]{0.19\linewidth}
        \centering
        \includegraphics[width=0.995\linewidth]{images/benchmark/d2_g002.png}
        \caption{Benchmark}
    \end{subfigure}
    \begin{subfigure}[b]{0.19\linewidth}
        \centering
        \includegraphics[width=0.995\linewidth]{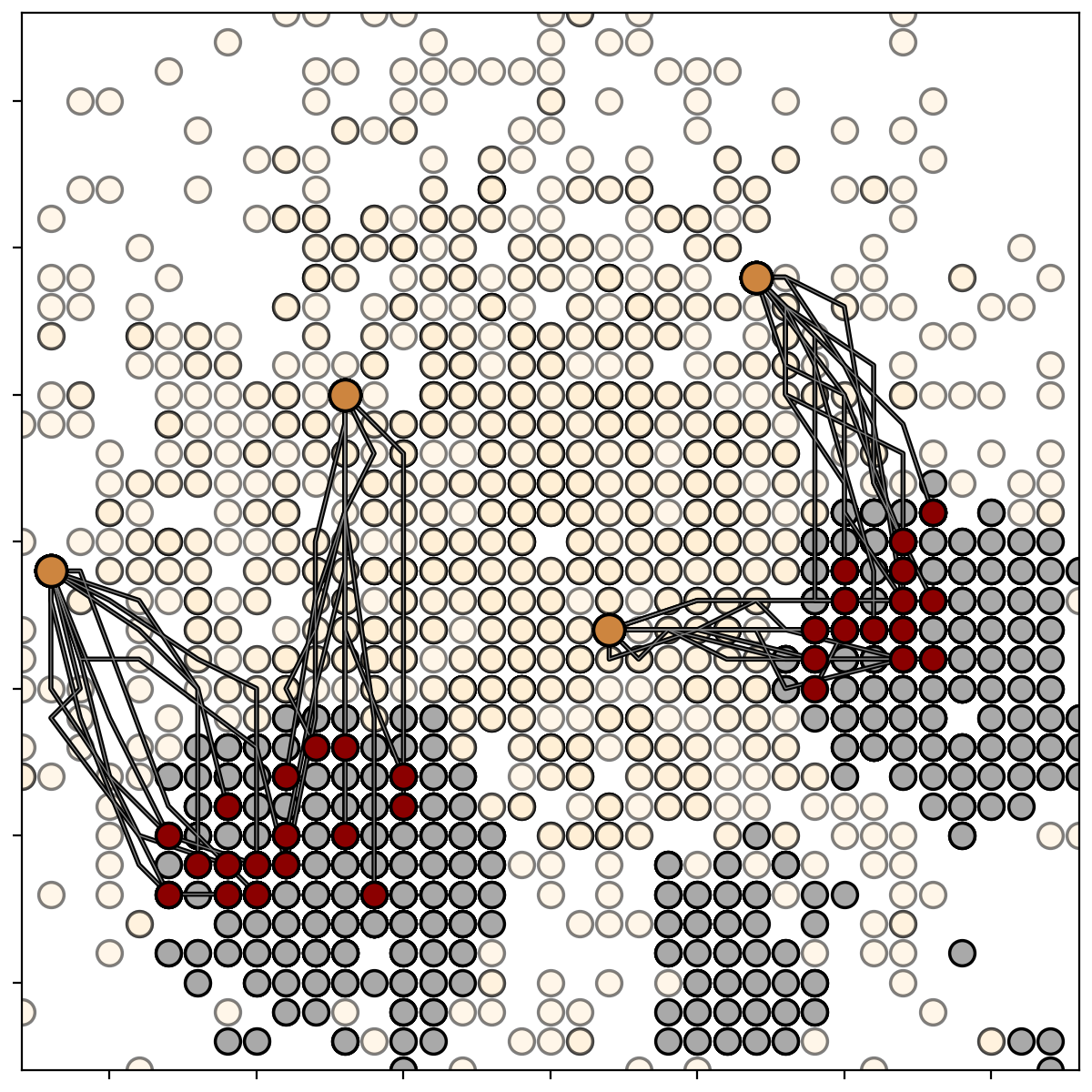}
        \caption{$\alpha$-CSBM KL/16}
    \end{subfigure}
    \begin{subfigure}[b]{0.19\linewidth}
        \centering
        \includegraphics[width=0.995\linewidth]{images/alpha_csbm/d2_g002_t63_kl.png}
        \caption{$\alpha$-CSBM KL/64}
    \end{subfigure}
    \begin{subfigure}[b]{0.19\linewidth}
        \centering
        \includegraphics[width=0.995\linewidth]{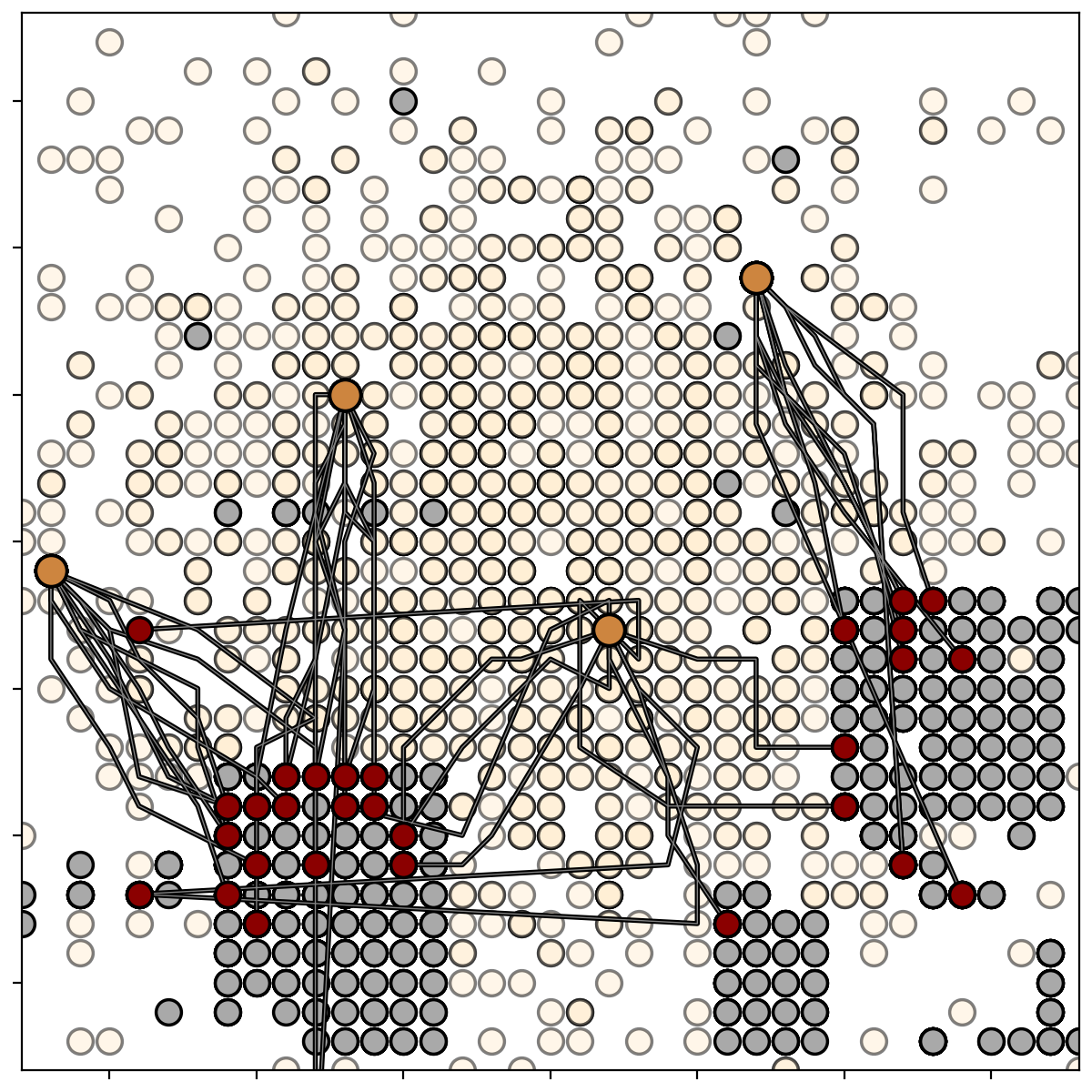}
        \caption{$\alpha$-CSBM MSE/16}
    \end{subfigure}
    \begin{subfigure}[b]{0.19\linewidth}
        \centering
        \includegraphics[width=0.995\linewidth]{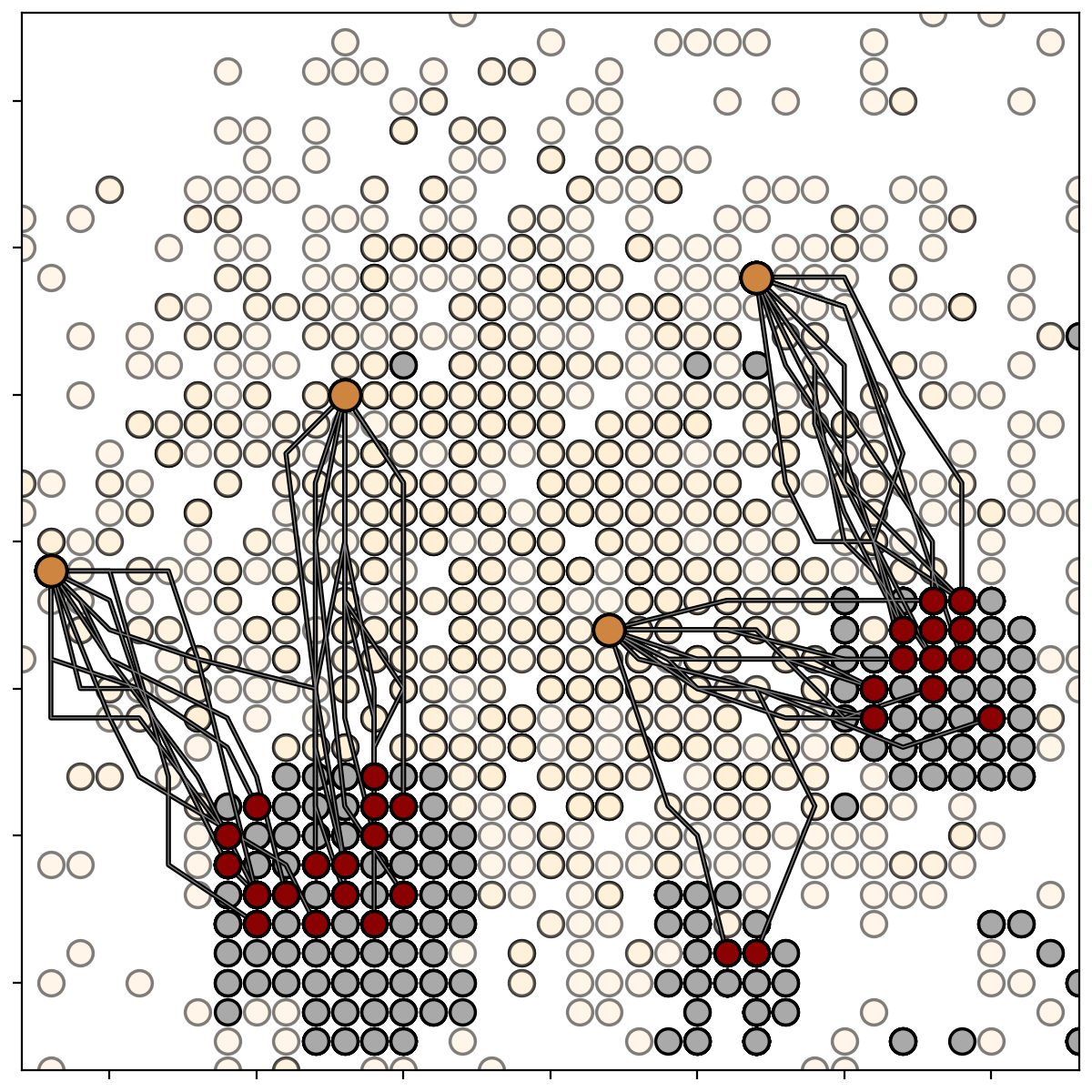}
        \caption{$\alpha$-CSBM MSE/64}
    \end{subfigure}
    \\
    \begin{subfigure}[b]{0.19\linewidth}
        \centering
        \includegraphics[width=0.995\linewidth]{images/dlight_sb/d2_g002.png}
        \caption{DLightSB}
    \end{subfigure}
    \begin{subfigure}[b]{0.19\linewidth}
        \centering
        \includegraphics[width=0.995\linewidth]{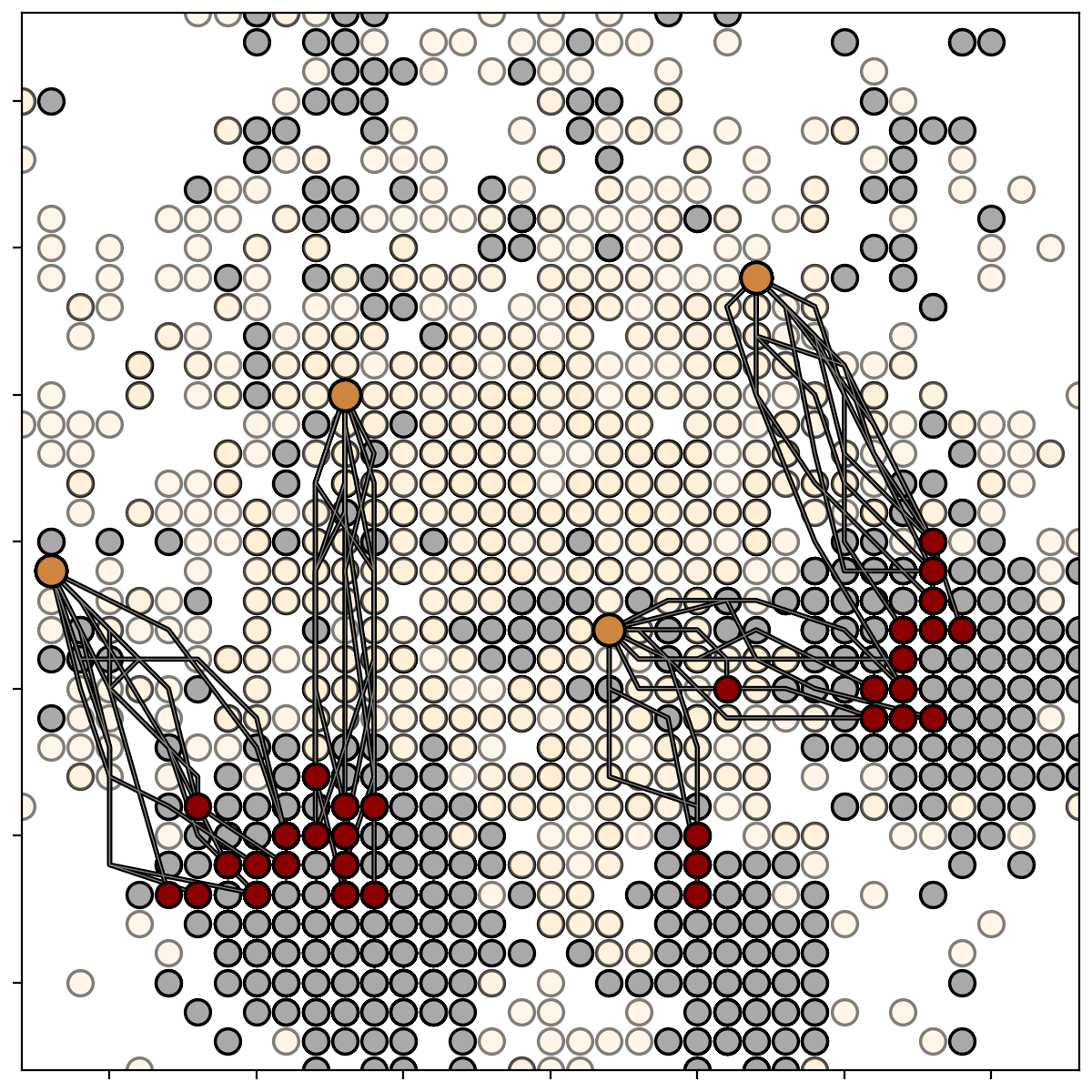}
        \caption{DLightSB-M KL/16}
    \end{subfigure}
    \begin{subfigure}[b]{0.19\linewidth}
        \centering
        \includegraphics[width=0.995\linewidth]{images/dlight_sb_m/d2_g002_t63_kl.png}
        \caption{DLightSB-M KL/64}
    \end{subfigure}
    \begin{subfigure}[b]{0.19\linewidth}
        \centering
        \includegraphics[width=0.995\linewidth]{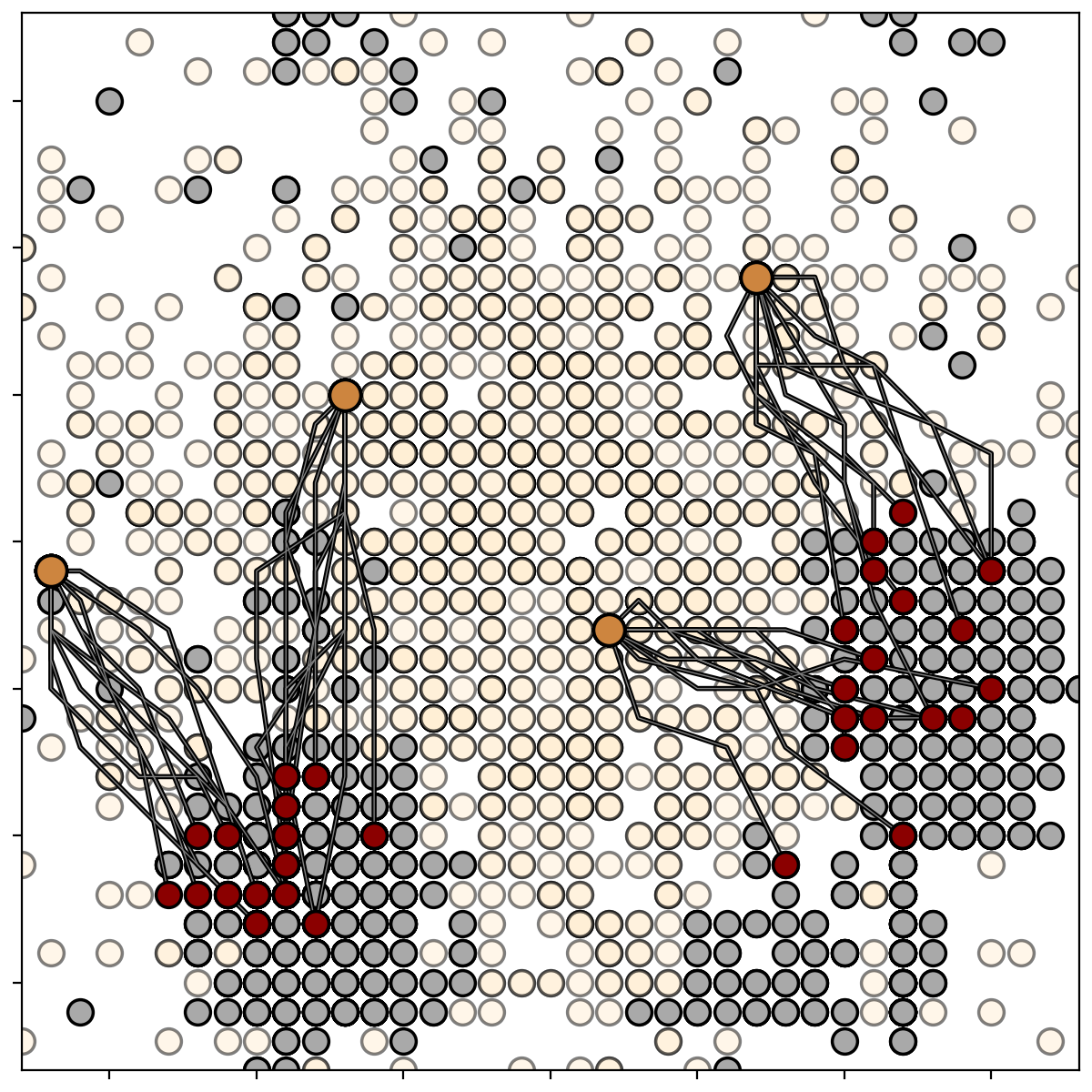}
        \caption{DLightSB-M MSE/16}
    \end{subfigure}
    \begin{subfigure}[b]{0.19\linewidth}
        \centering
        \includegraphics[width=0.995\linewidth]{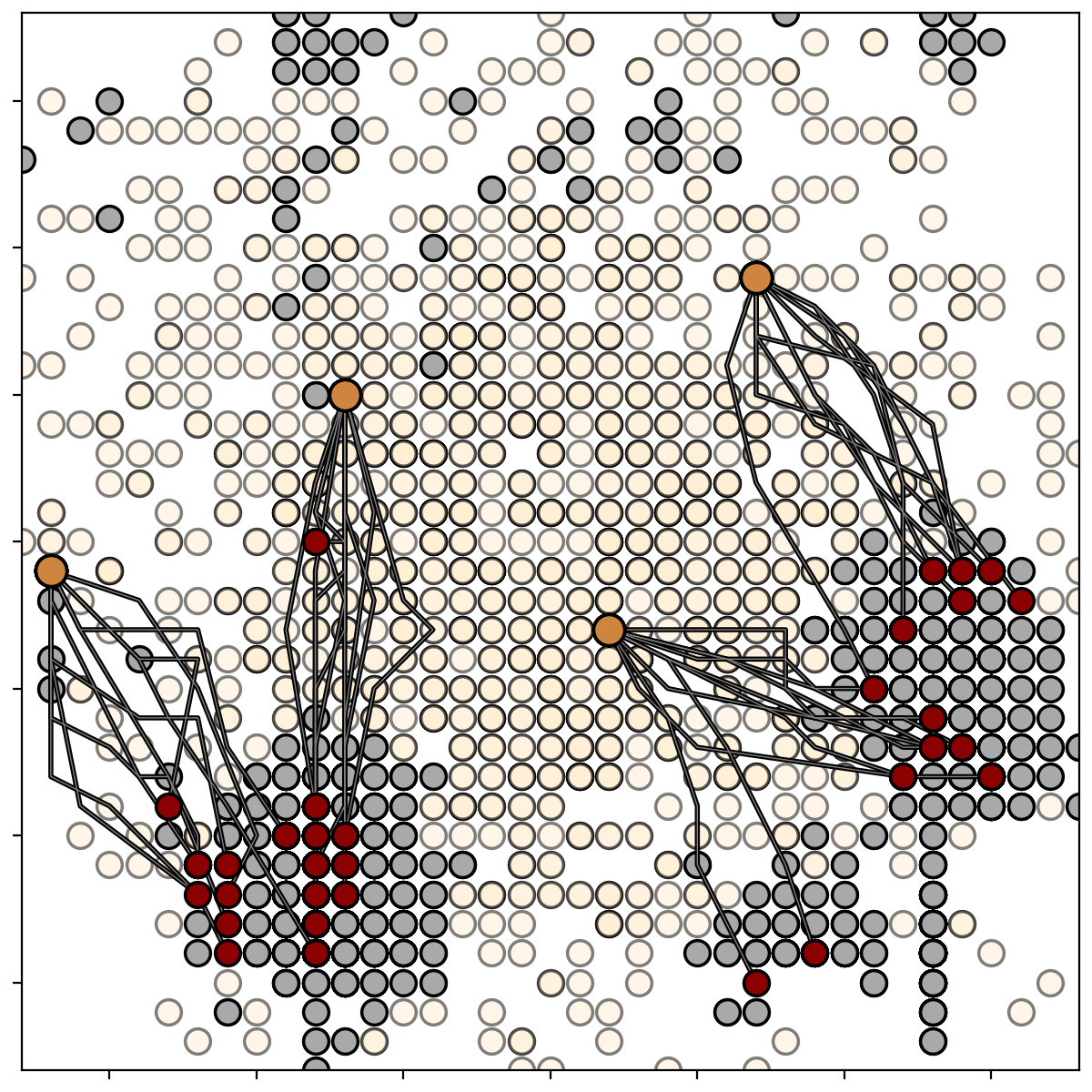}
        \caption{DLightSB-M MSE/64}
    \end{subfigure}
    \caption{ \centering Samples from all methods on the high-dimensional Gaussian mixture benchmark using the Gaussian reference process $q^{\text{gauss}}$ with $\gamma = 0.02$.}
    \label{figure:g2_002_samples}
\end{figure}

\begin{figure}[h]
    \centering
    \captionsetup[subfigure]{font=scriptsize}
    \begin{subfigure}[b]{0.19\linewidth}
        \centering
        \includegraphics[width=0.995\linewidth]{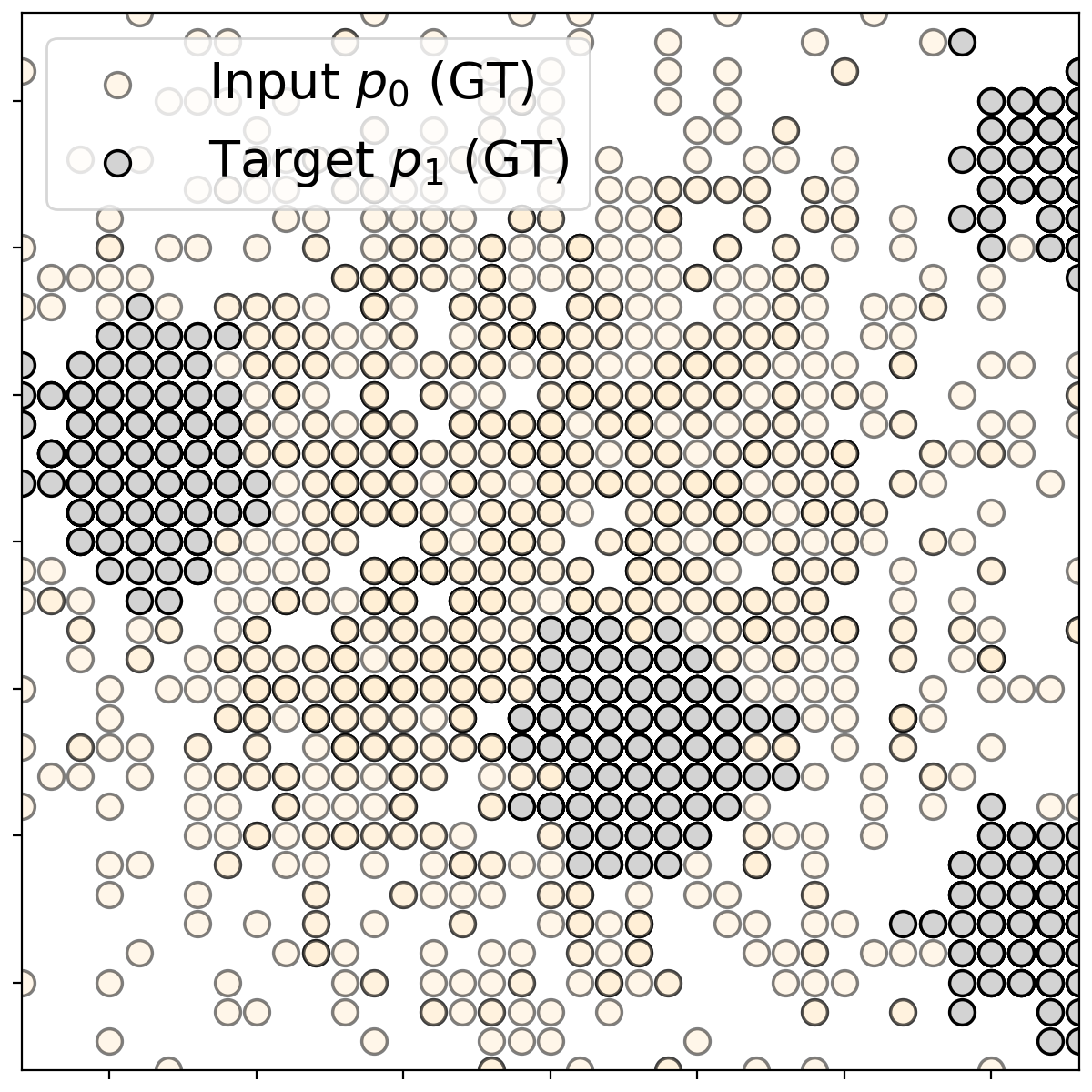}
        \caption{Input and Target}
    \end{subfigure}
    \begin{subfigure}[b]{0.19\linewidth}
        \centering
        \includegraphics[width=0.995\linewidth]{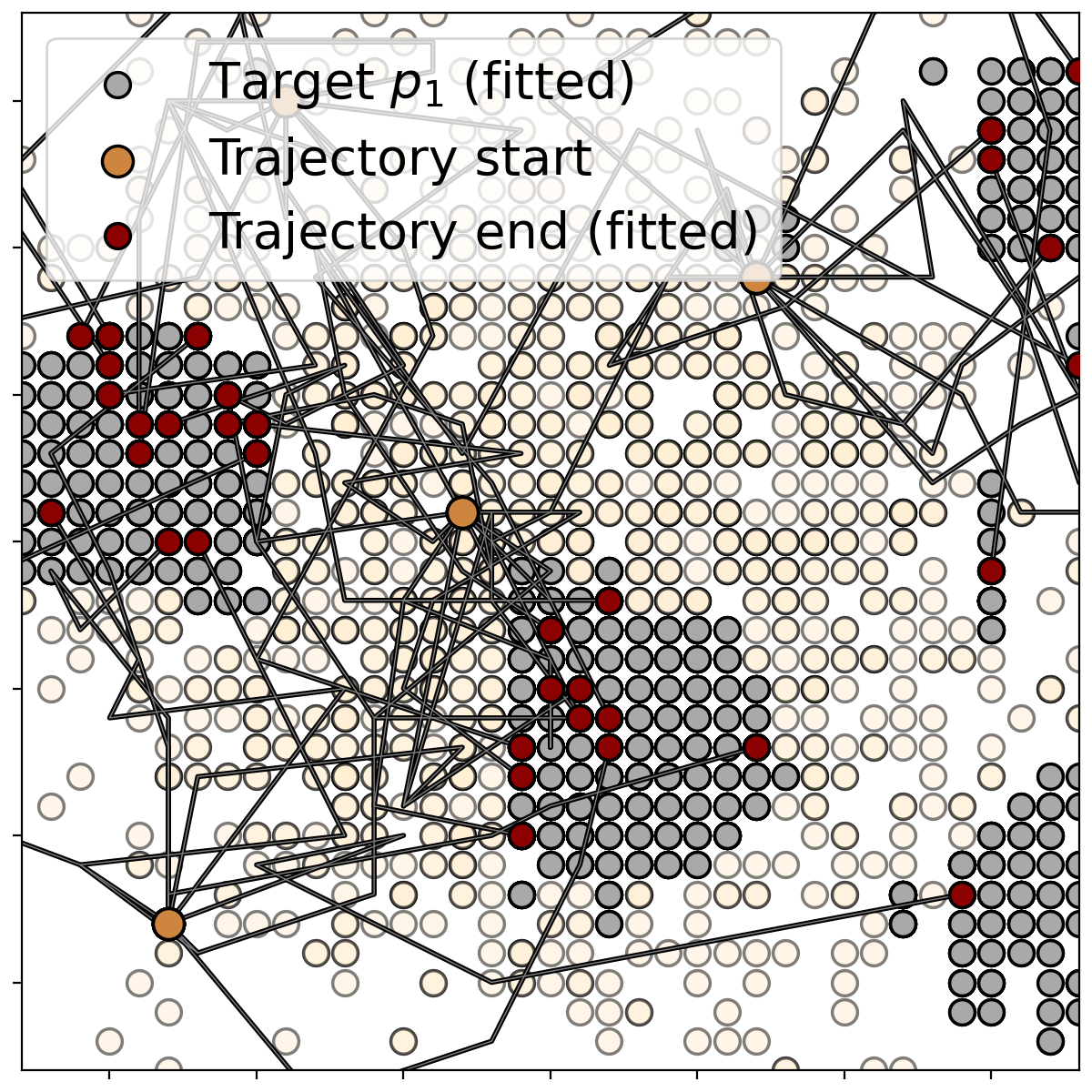}
        \caption{CSBM KL/16}
    \end{subfigure}
    \begin{subfigure}[b]{0.19\linewidth}
        \centering
        \includegraphics[width=0.995\linewidth]{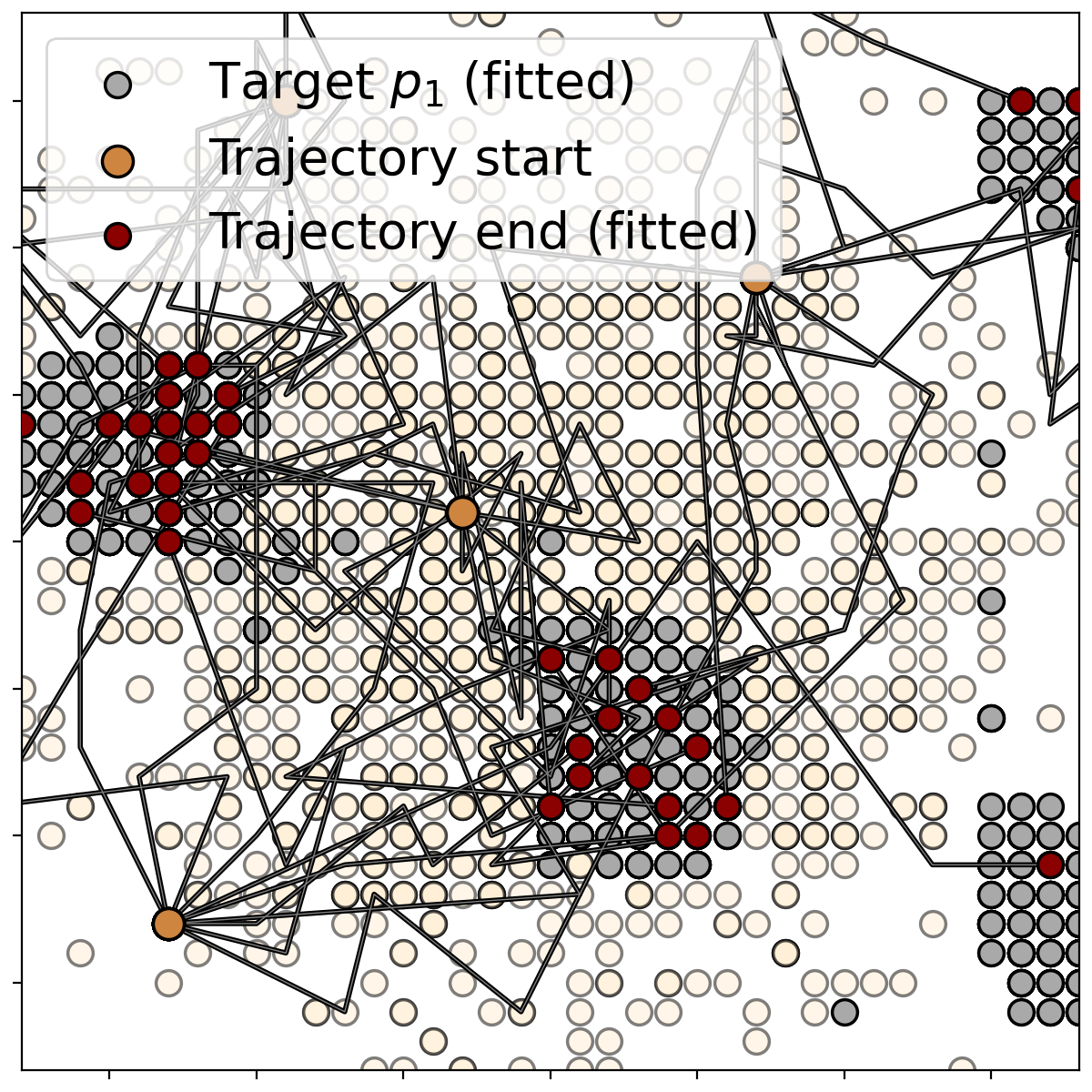}
        \caption{CSBM KL/64}
    \end{subfigure}
    \begin{subfigure}[b]{0.19\linewidth}
        \centering
        \includegraphics[width=0.995\linewidth]{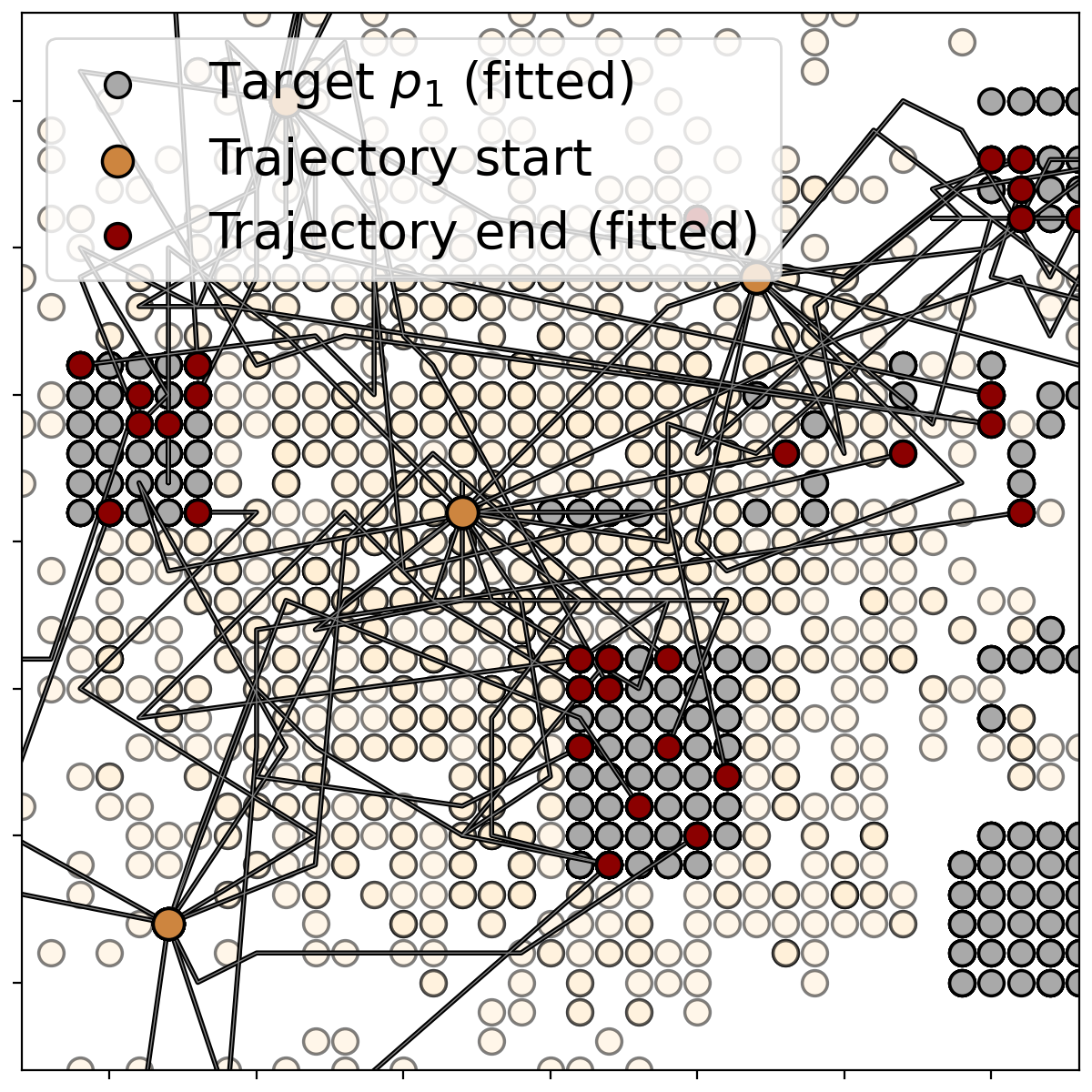}
        \caption{CSBM MSE/16}
    \end{subfigure}
    \begin{subfigure}[b]{0.19\linewidth}
        \centering
        \includegraphics[width=0.995\linewidth]{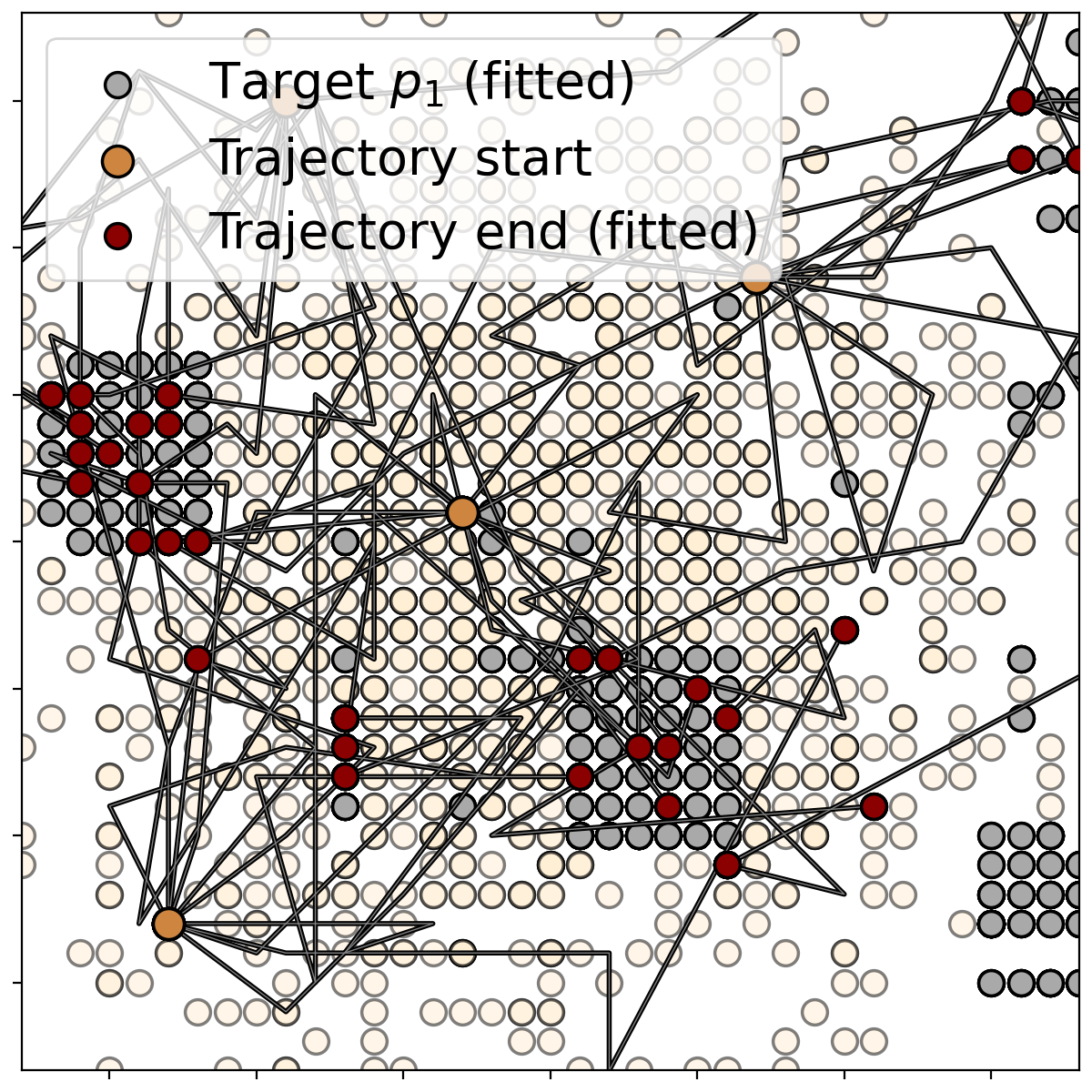}
        \caption{CSBM MSE/64}
    \end{subfigure}
    \\
    \begin{subfigure}[b]{0.19\linewidth}
        \centering
        \includegraphics[width=0.995\linewidth]{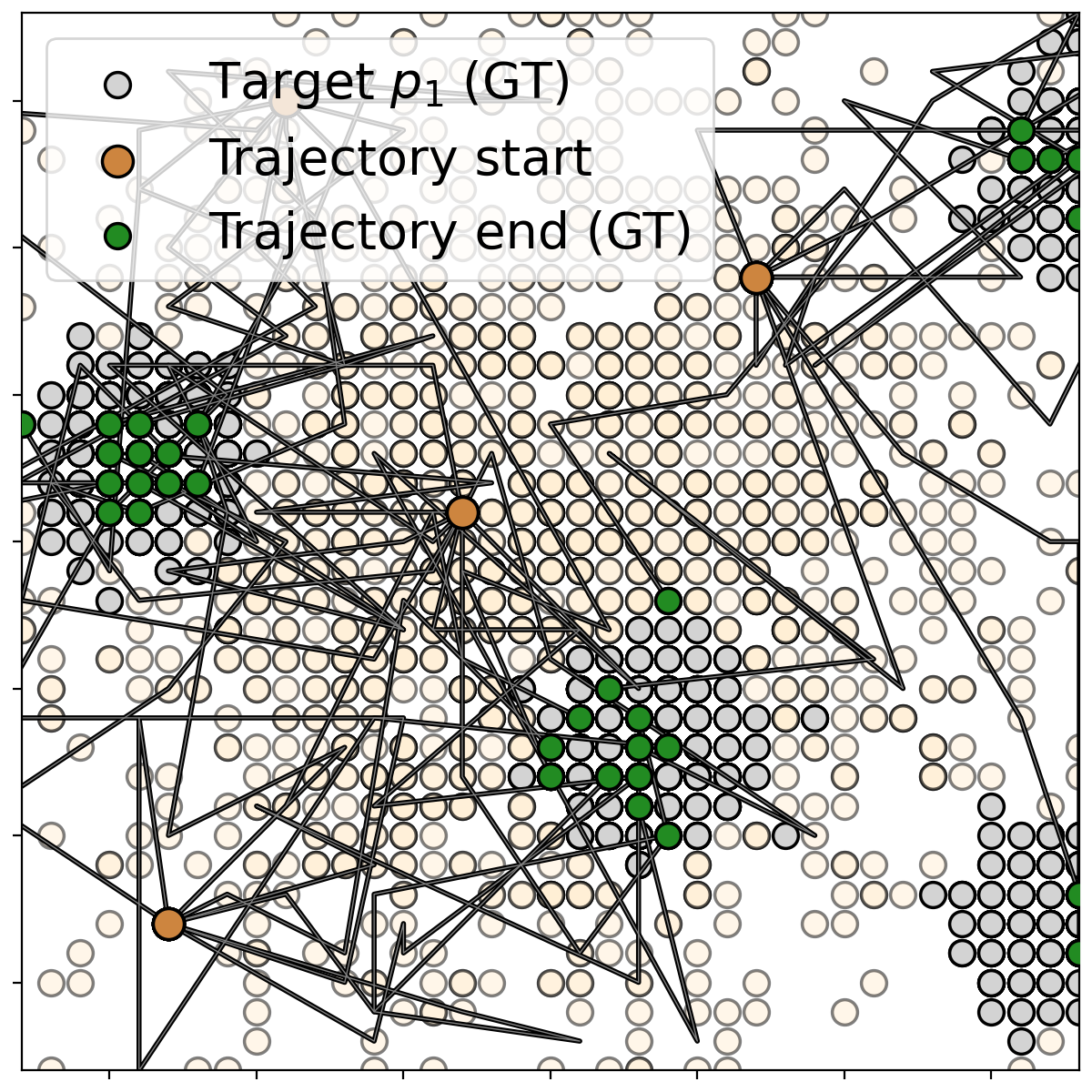}
        \caption{Benchmark}
    \end{subfigure}
    \begin{subfigure}[b]{0.19\linewidth}
        \centering
        \includegraphics[width=0.995\linewidth]{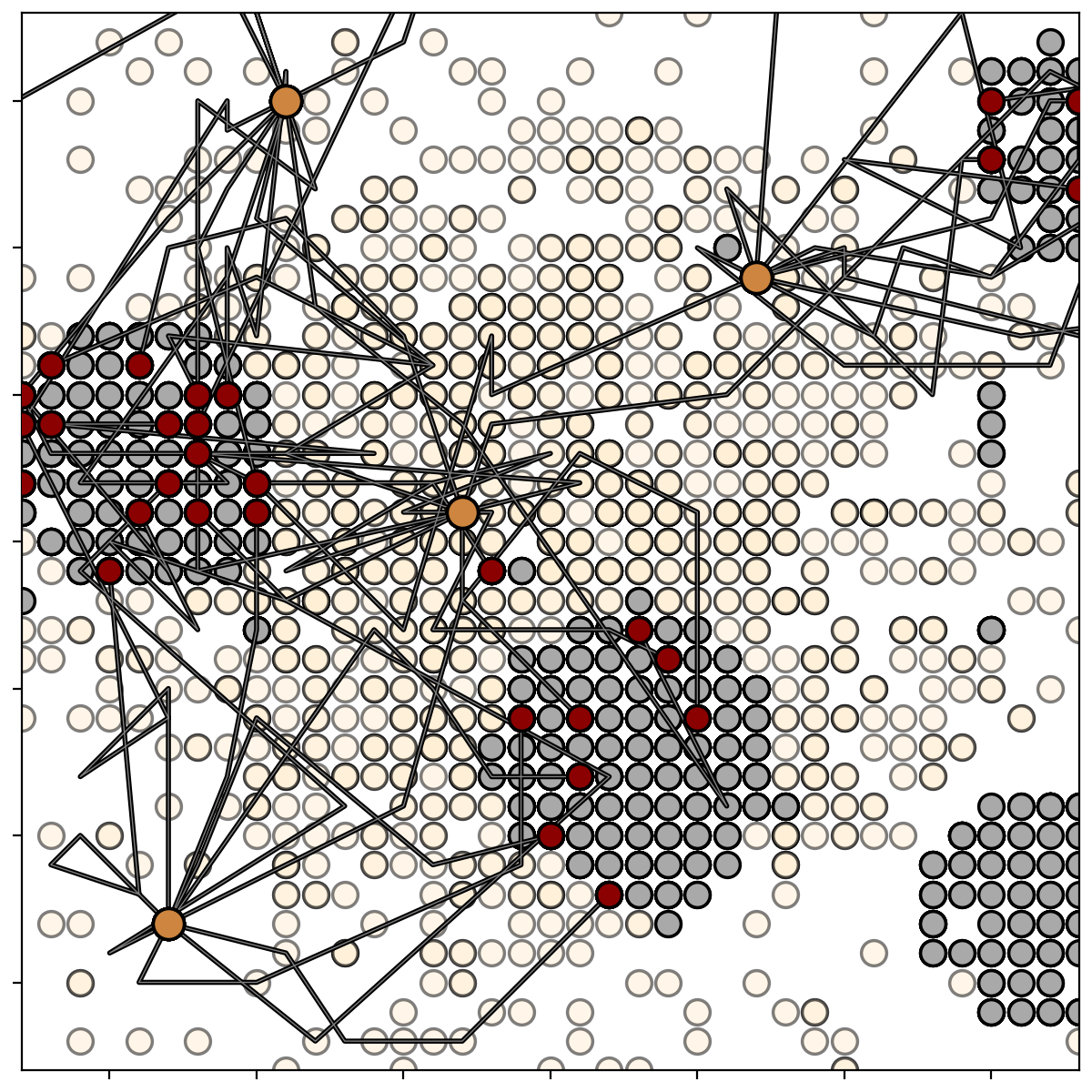}
        \caption{$\alpha$-CSBM KL/16}
    \end{subfigure}
    \begin{subfigure}[b]{0.19\linewidth}
        \centering
        \includegraphics[width=0.995\linewidth]{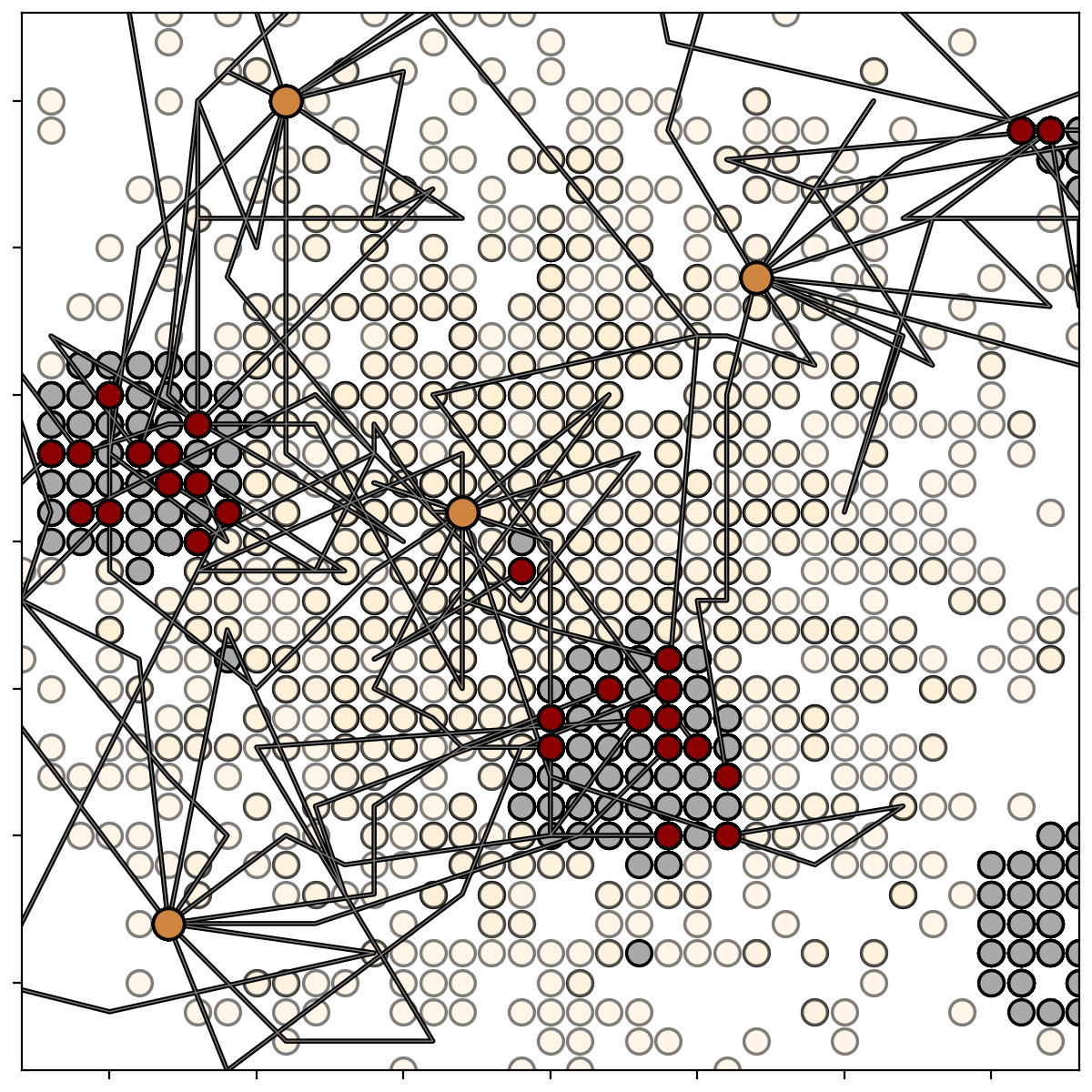}
        \caption{$\alpha$-CSBM KL/64}
    \end{subfigure}
    \begin{subfigure}[b]{0.19\linewidth}
        \centering
        \includegraphics[width=0.995\linewidth]{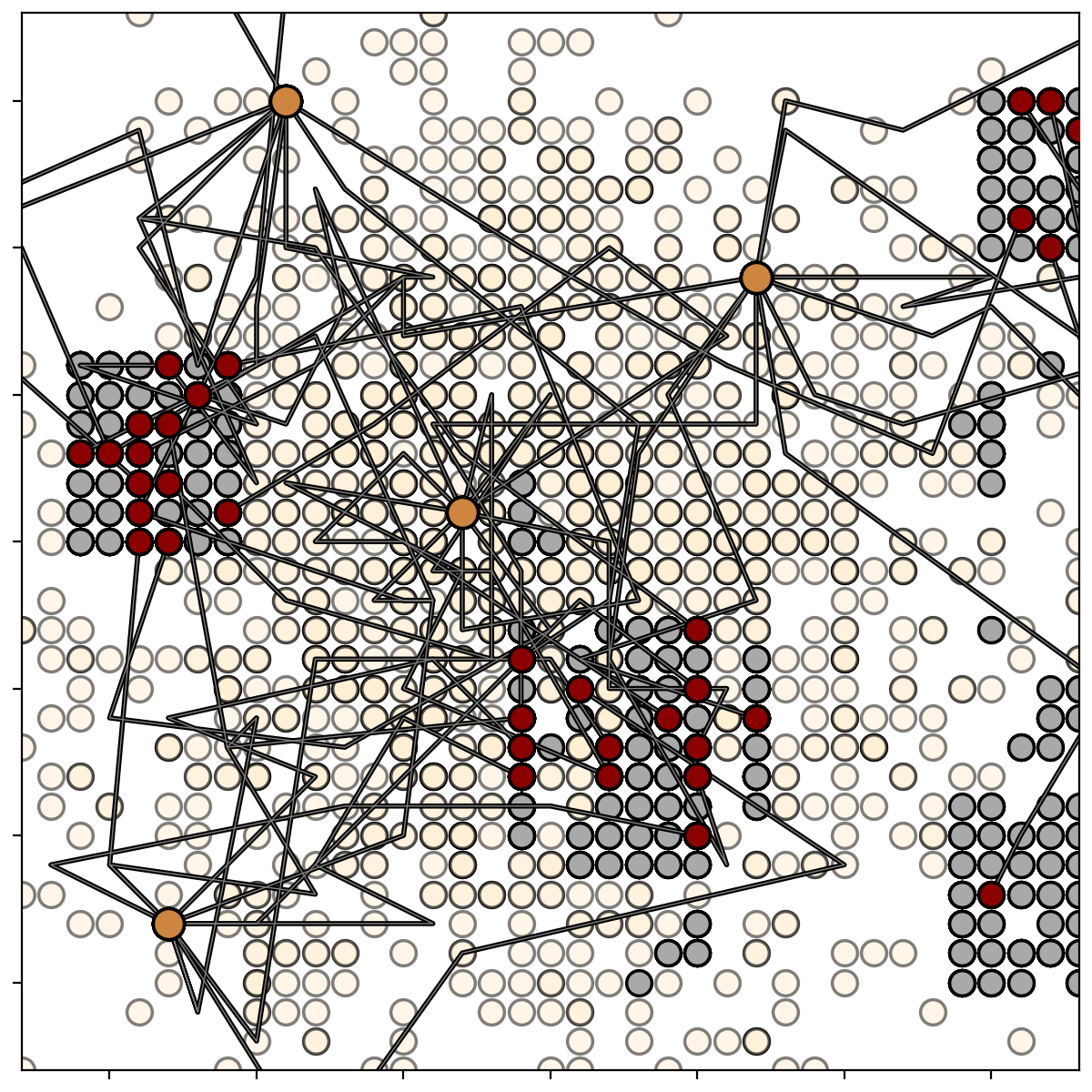}
        \caption{$\alpha$-CSBM MSE/16}
    \end{subfigure}
    \begin{subfigure}[b]{0.19\linewidth}
        \centering
        \includegraphics[width=0.995\linewidth]{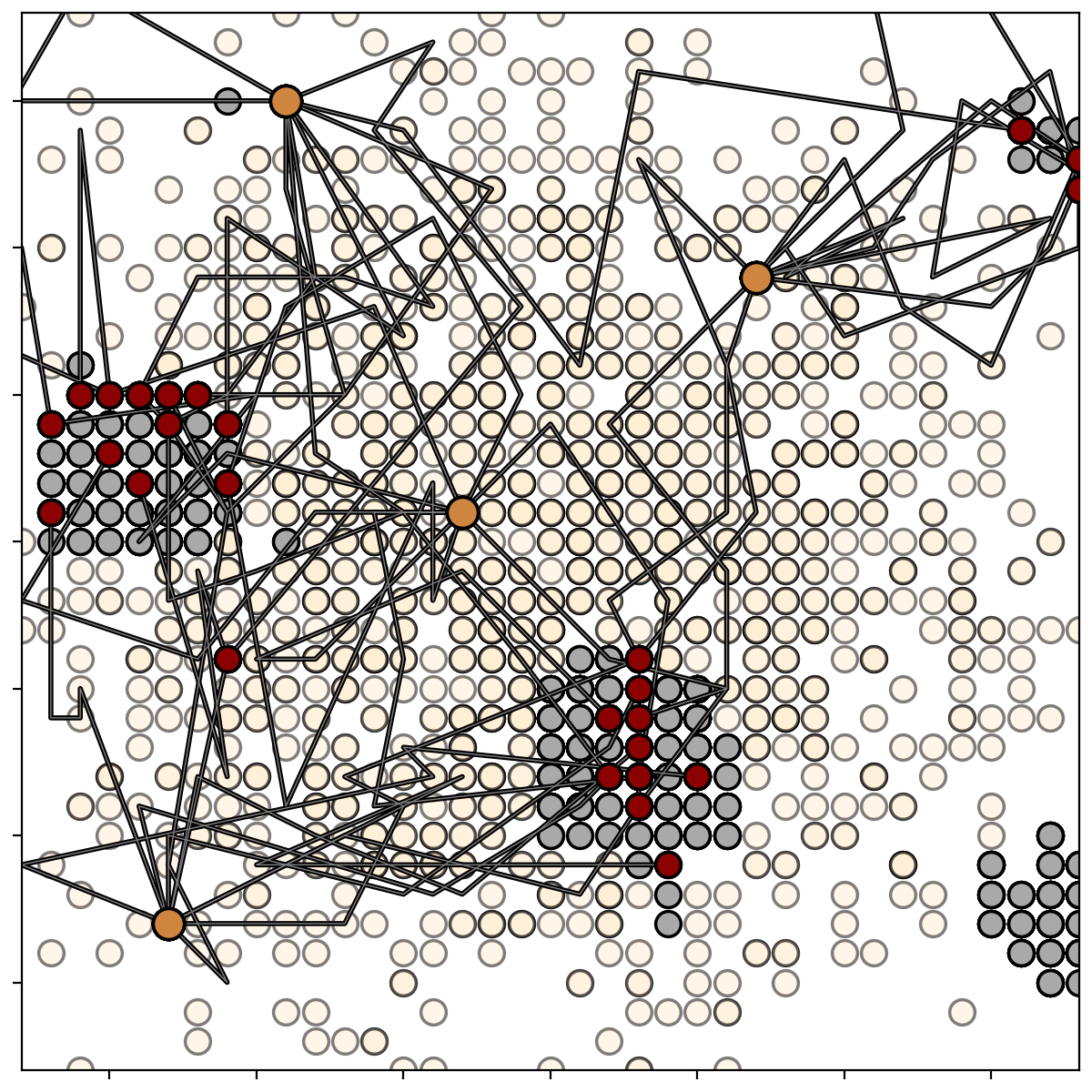}
        \caption{$\alpha$-CSBM MSE/64}
    \end{subfigure}
    \\
    \begin{subfigure}[b]{0.19\linewidth}
        \centering
        \includegraphics[width=0.995\linewidth]{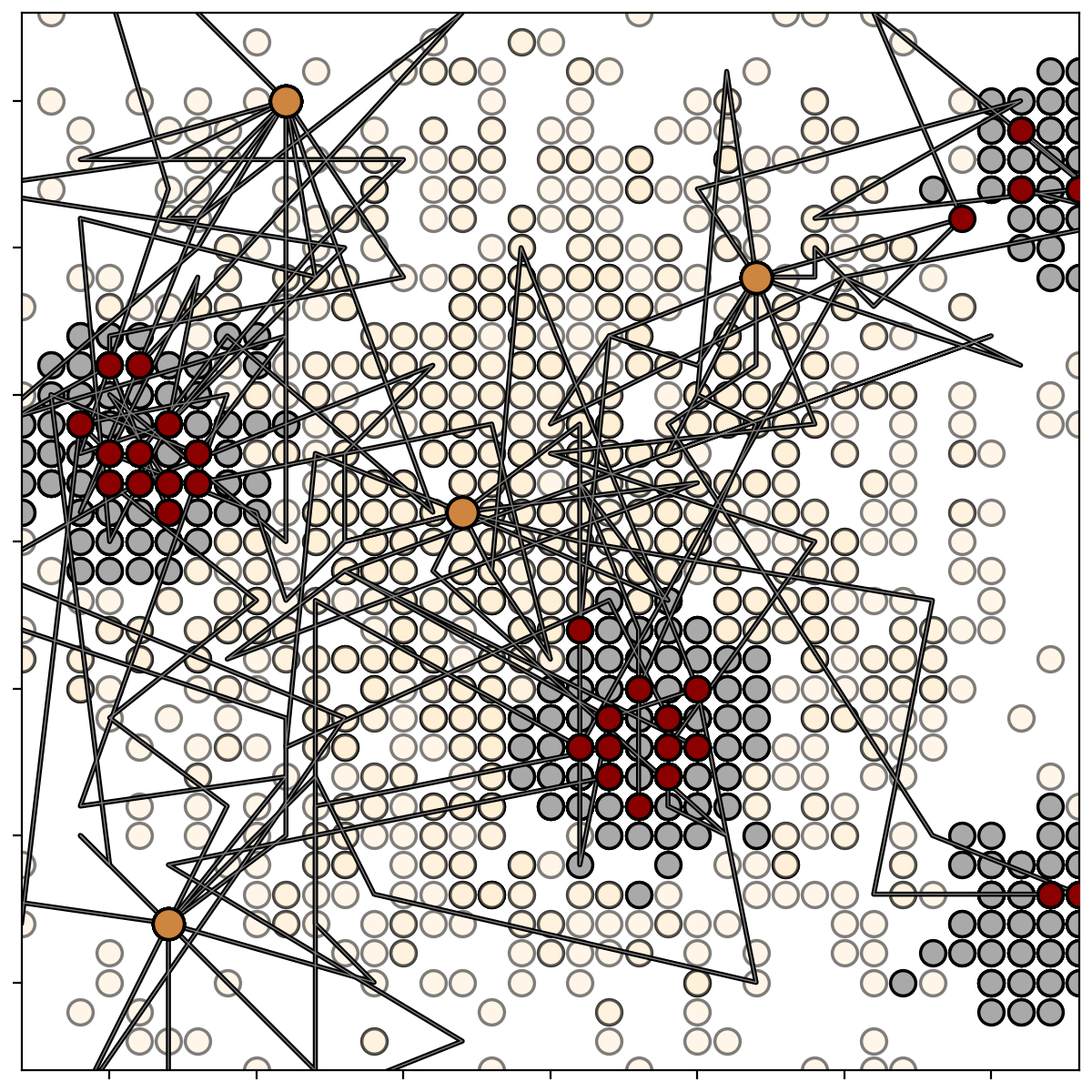}
        \caption{DLightSB}
    \end{subfigure}
    \begin{subfigure}[b]{0.19\linewidth}
        \centering
        \includegraphics[width=0.995\linewidth]{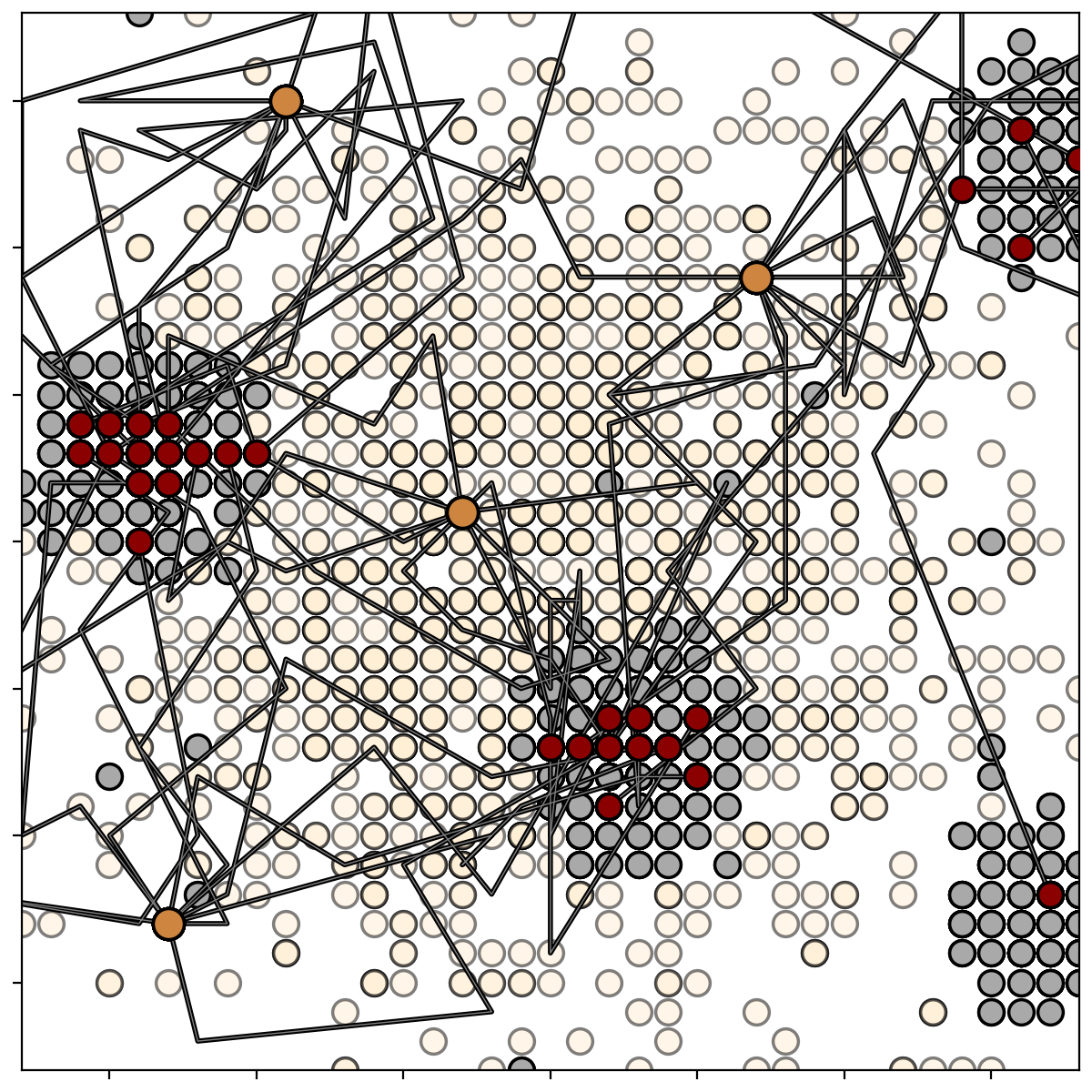}
        \caption{DLightSB-M KL/16}
    \end{subfigure}
    \begin{subfigure}[b]{0.19\linewidth}
        \centering
        \includegraphics[width=0.995\linewidth]{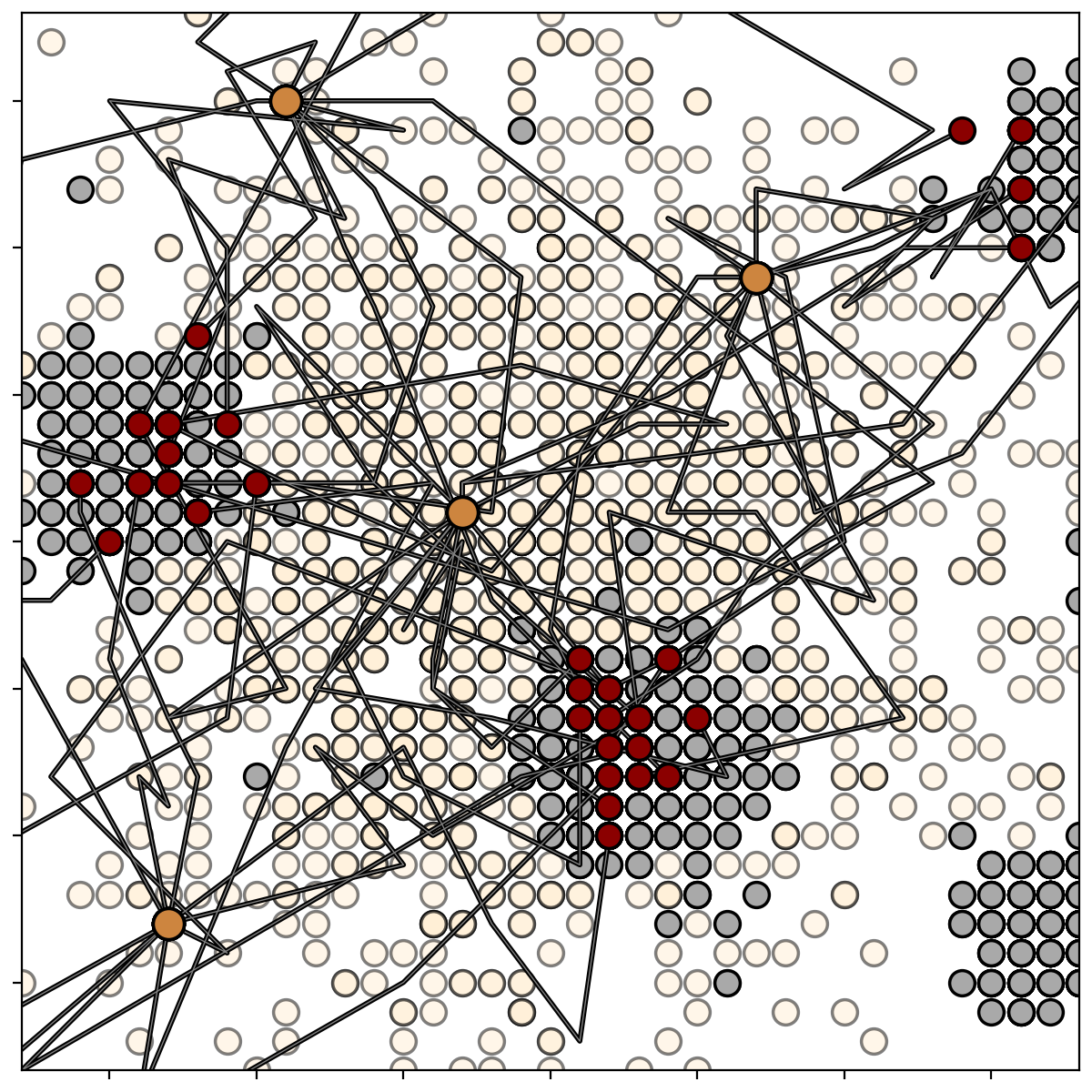}
        \caption{DLightSB-M KL/64}
    \end{subfigure}
    \begin{subfigure}[b]{0.19\linewidth}
        \centering
        \includegraphics[width=0.995\linewidth]{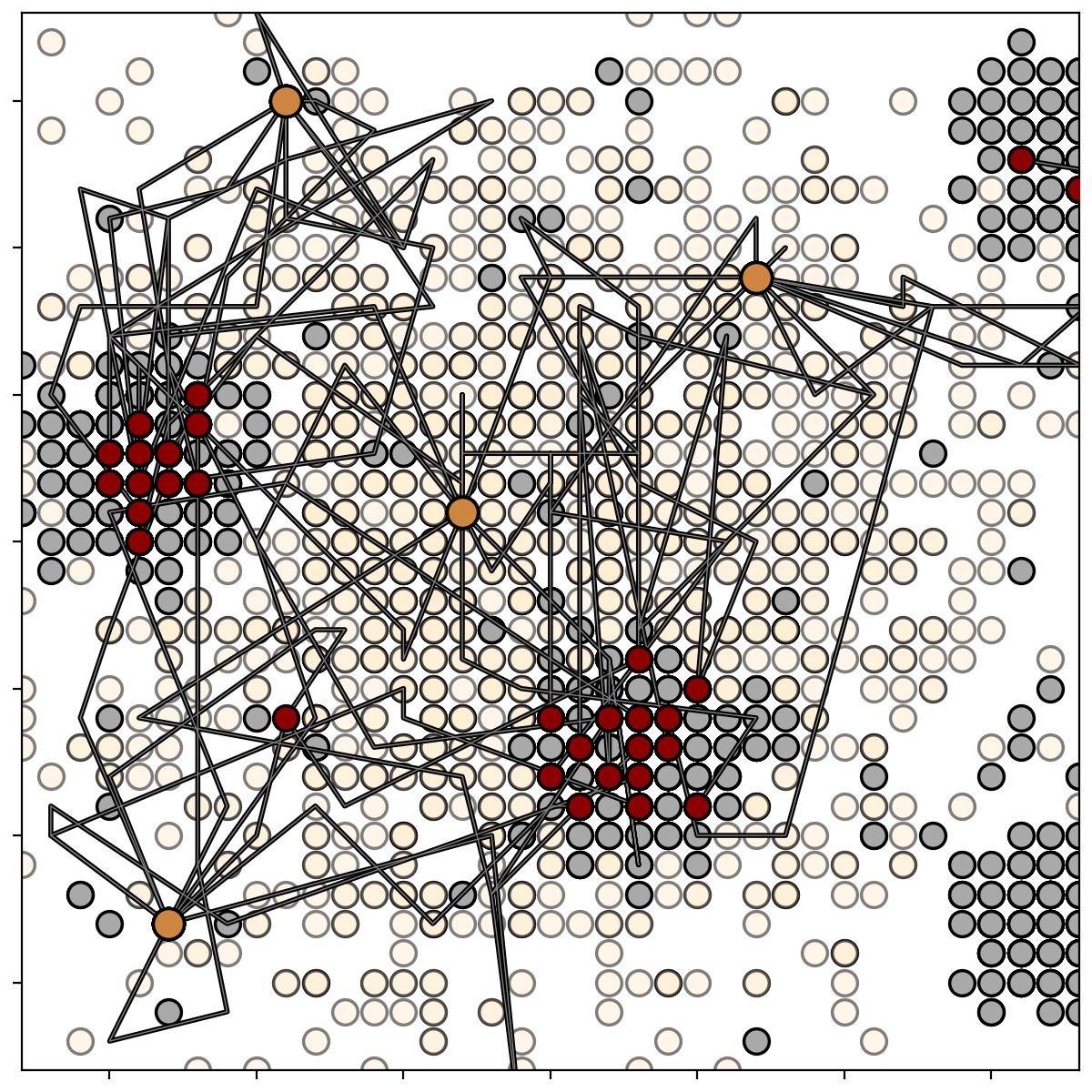}
        \caption{DLightSB-M MSE/16}
    \end{subfigure}
    \begin{subfigure}[b]{0.19\linewidth}
        \centering
        \includegraphics[width=0.995\linewidth]{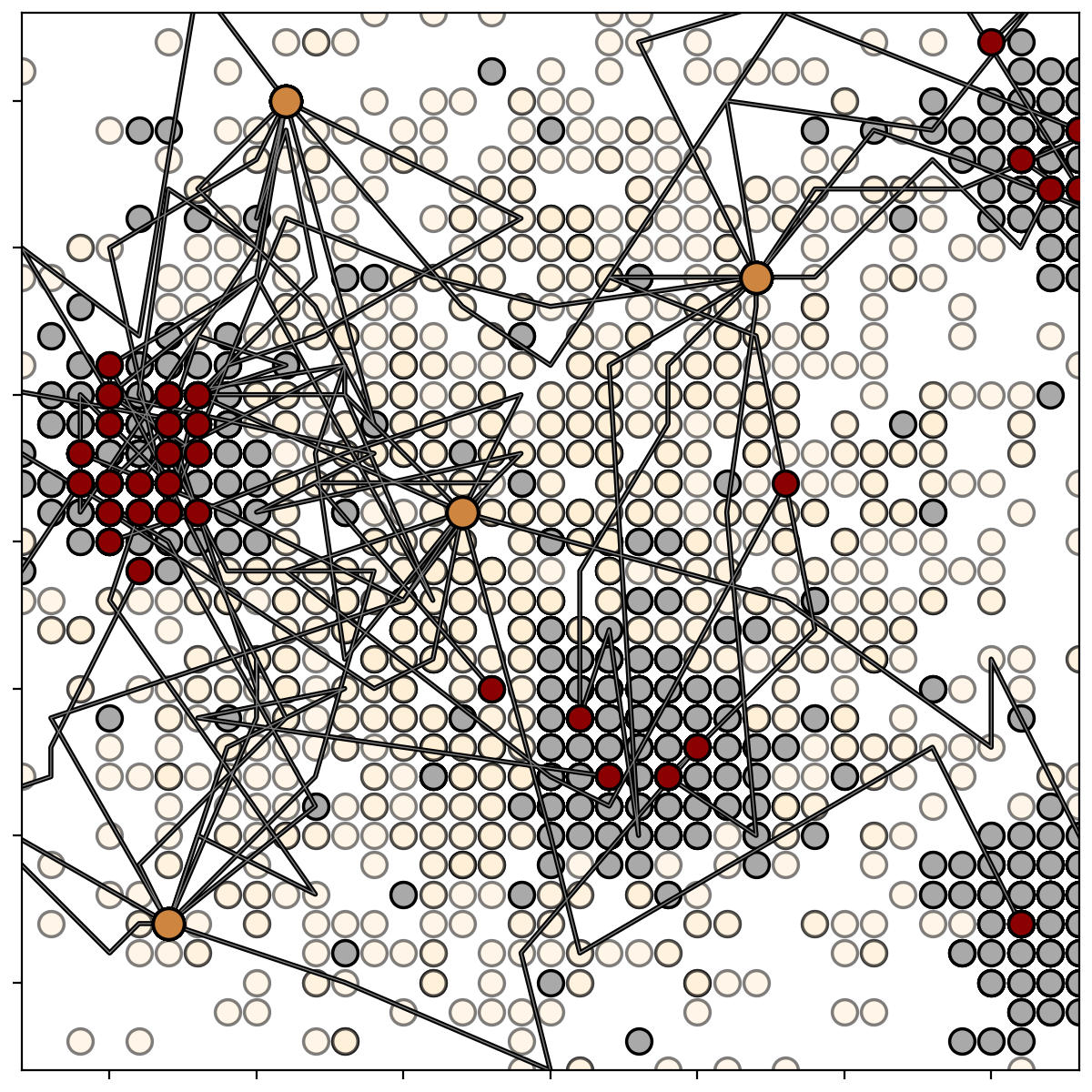}
        \caption{DLightSB-M MSE/64}
    \end{subfigure}
    \caption{ \centering Samples from all methods on the high-dimensional Gaussian mixture benchmark using the Gaussian reference process $q^{\text{gauss}}$ with $\gamma = 0.05$.}
    \label{figure:d2_g005_samples}
\end{figure}

\begin{figure}[h]
    \centering
    \captionsetup[subfigure]{font=scriptsize}
    \begin{subfigure}[b]{0.19\linewidth}
        \centering
        \includegraphics[width=0.995\linewidth]{images/benchmark/pairs_d2_u0005.png}
        \caption{Input and Target}
    \end{subfigure}
    \begin{subfigure}[b]{0.19\linewidth}
        \centering
        \includegraphics[width=0.995\linewidth]{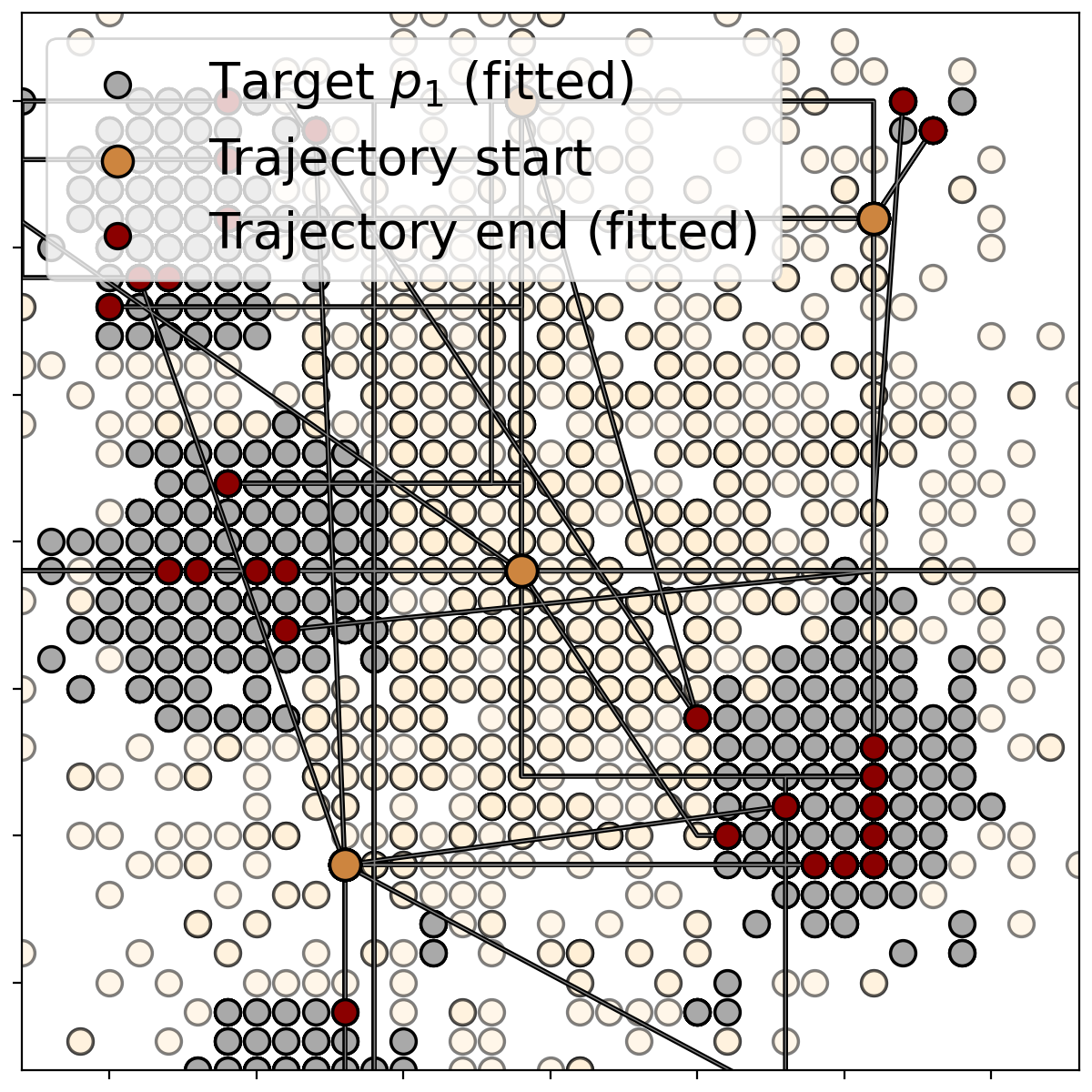}
        \caption{CSBM KL/16}
    \end{subfigure}
    \begin{subfigure}[b]{0.19\linewidth}
        \centering
        \includegraphics[width=0.995\linewidth]{images/csbm/d2_u0005_t63_kl.png}
        \caption{CSBM KL/64}
    \end{subfigure}
    \begin{subfigure}[b]{0.19\linewidth}
        \centering
        \includegraphics[width=0.995\linewidth]{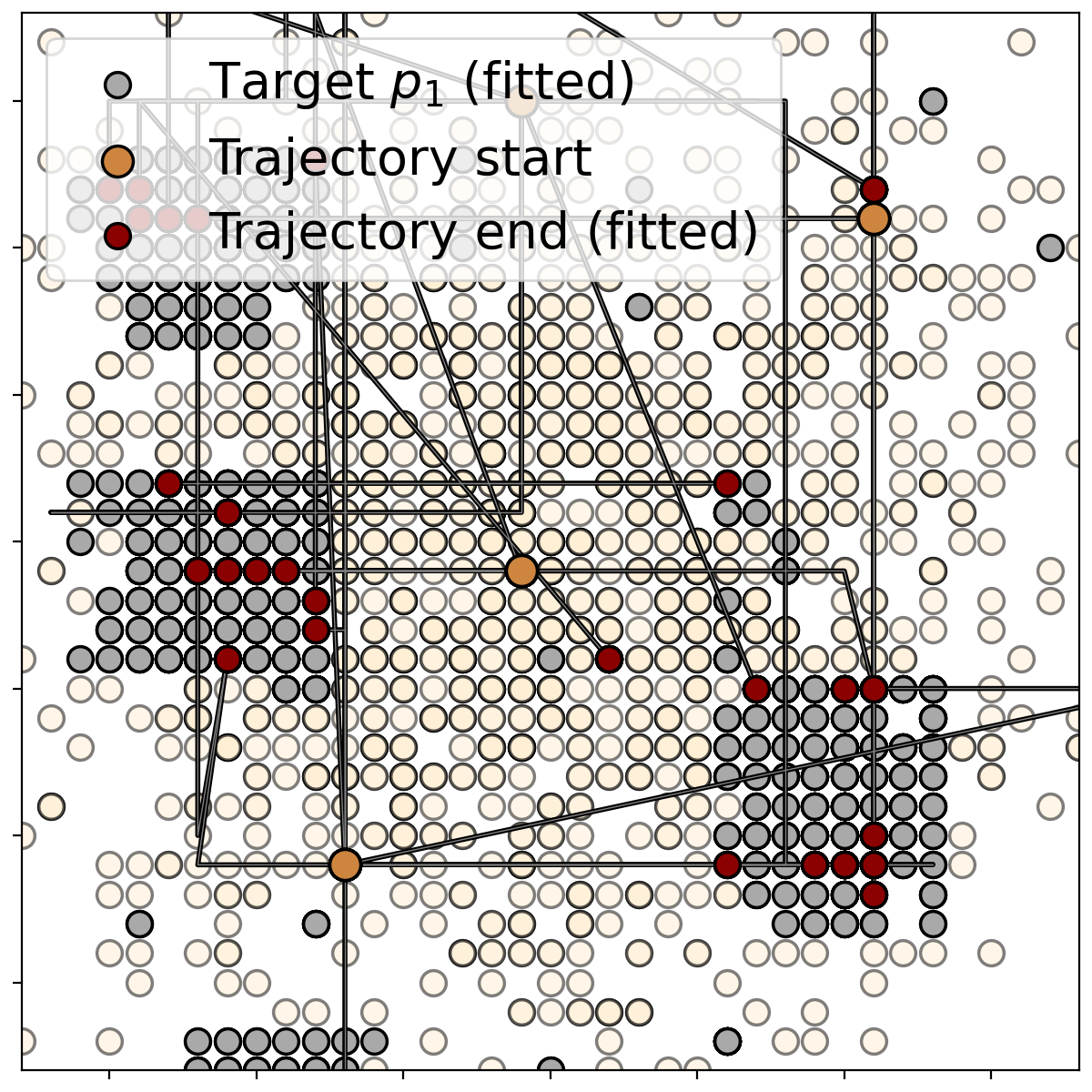}
        \caption{CSBM MSE/16}
    \end{subfigure}
    \begin{subfigure}[b]{0.19\linewidth}
        \centering
        \includegraphics[width=0.995\linewidth]{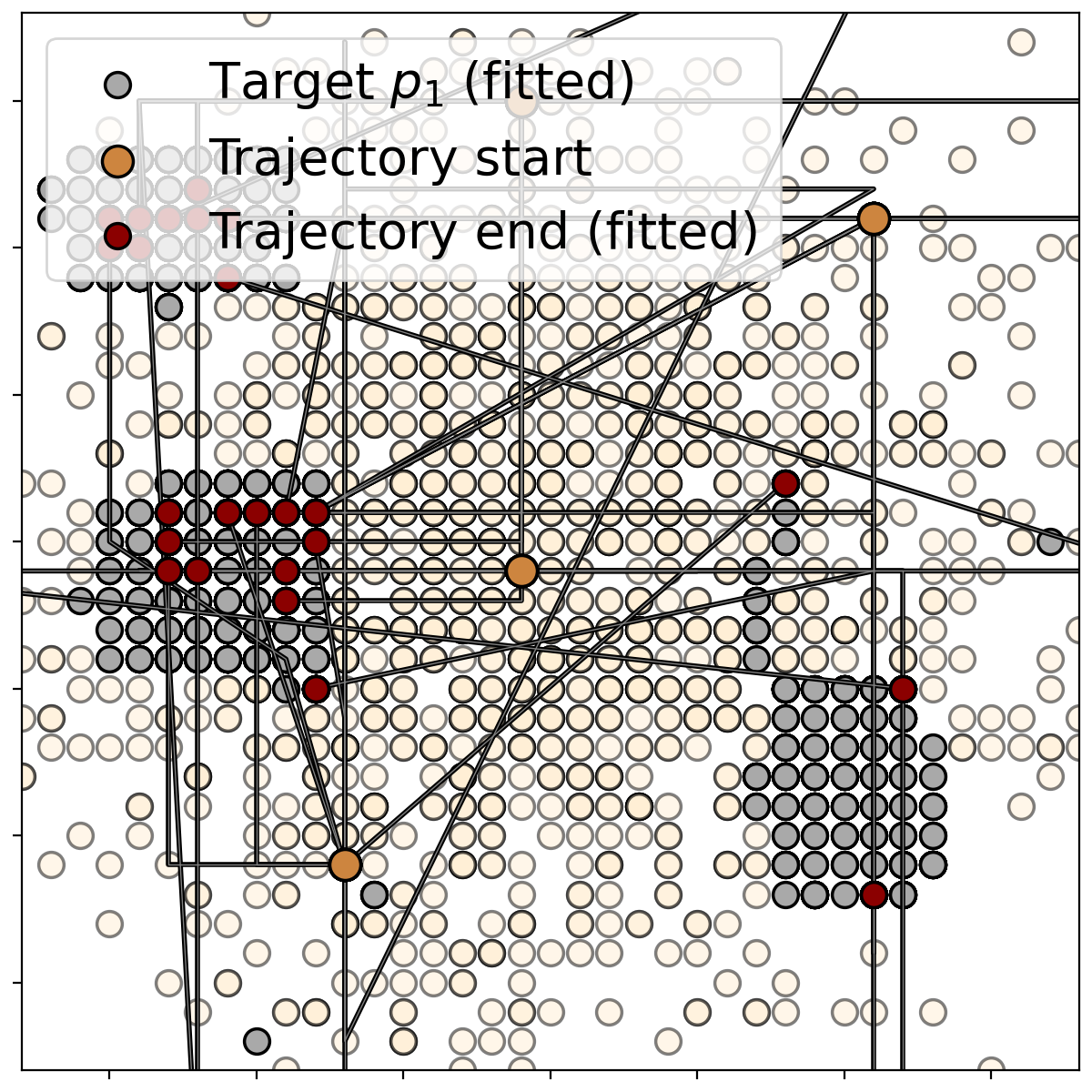}
        \caption{CSBM MSE/64}
    \end{subfigure}
    \\
    \begin{subfigure}[b]{0.19\linewidth}
        \centering
        \includegraphics[width=0.995\linewidth]{images/benchmark/d2_u0005.png}
        \caption{Benchmark}
    \end{subfigure}
    \begin{subfigure}[b]{0.19\linewidth}
        \centering
        \includegraphics[width=0.995\linewidth]{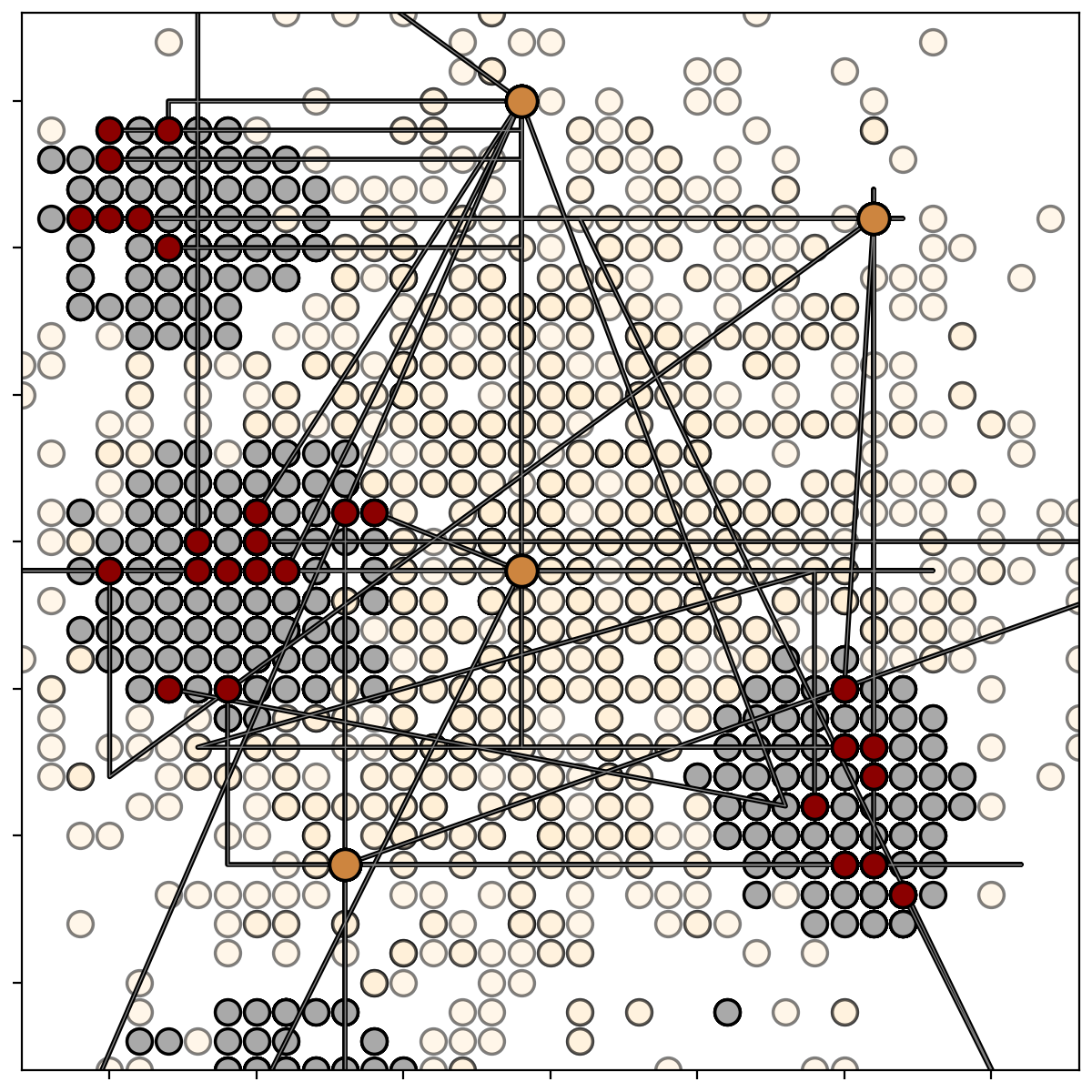}
        \caption{$\alpha$-CSBM KL/16}
    \end{subfigure}
    \begin{subfigure}[b]{0.19\linewidth}
        \centering
        \includegraphics[width=0.995\linewidth]{images/alpha_csbm/d2_u0005_t63_kl.png}
        \caption{$\alpha$-CSBM KL/64}
    \end{subfigure}
    \begin{subfigure}[b]{0.19\linewidth}
        \centering
        \includegraphics[width=0.995\linewidth]{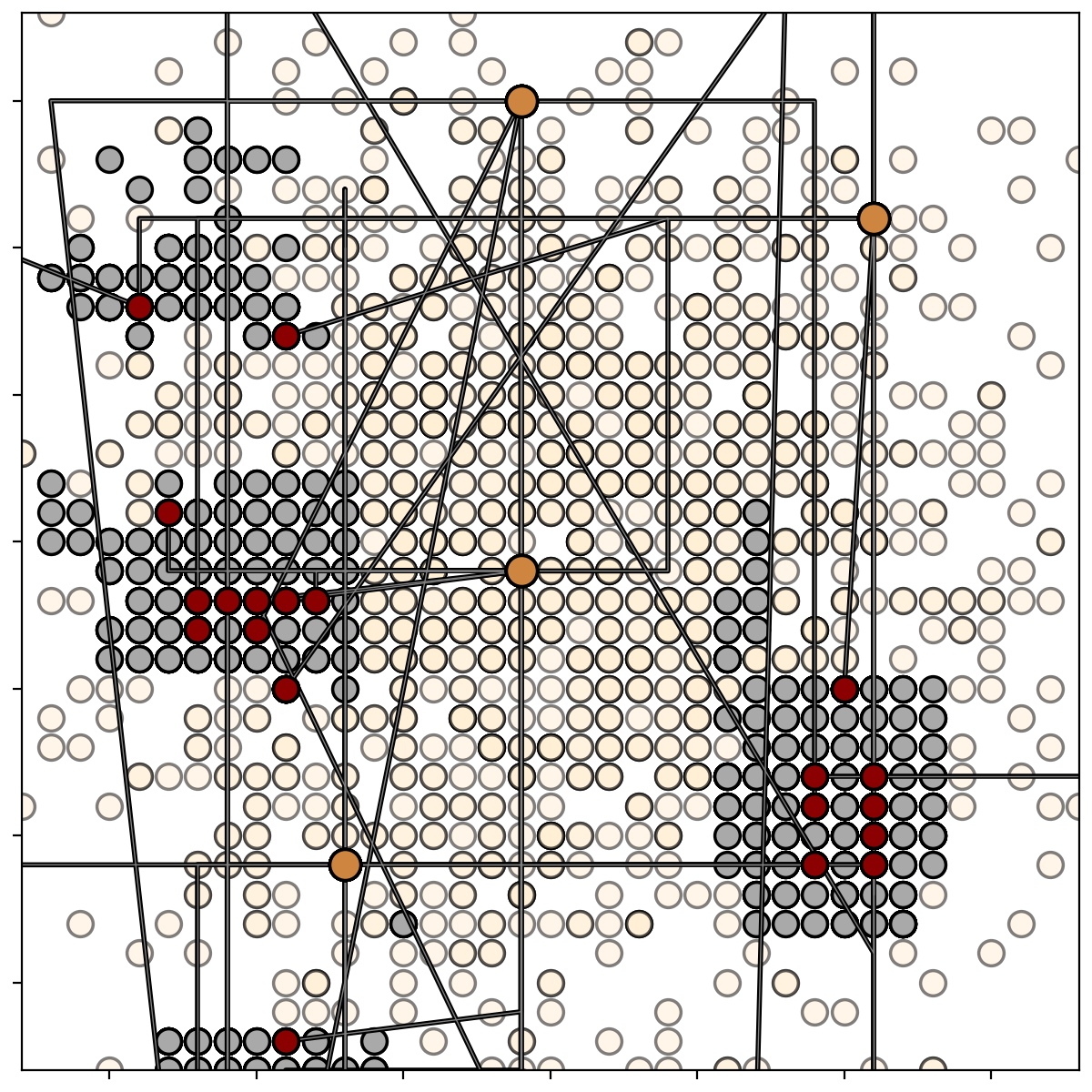}
        \caption{$\alpha$-CSBM MSE/16}
    \end{subfigure}
    \begin{subfigure}[b]{0.19\linewidth}
        \centering
        \includegraphics[width=0.995\linewidth]{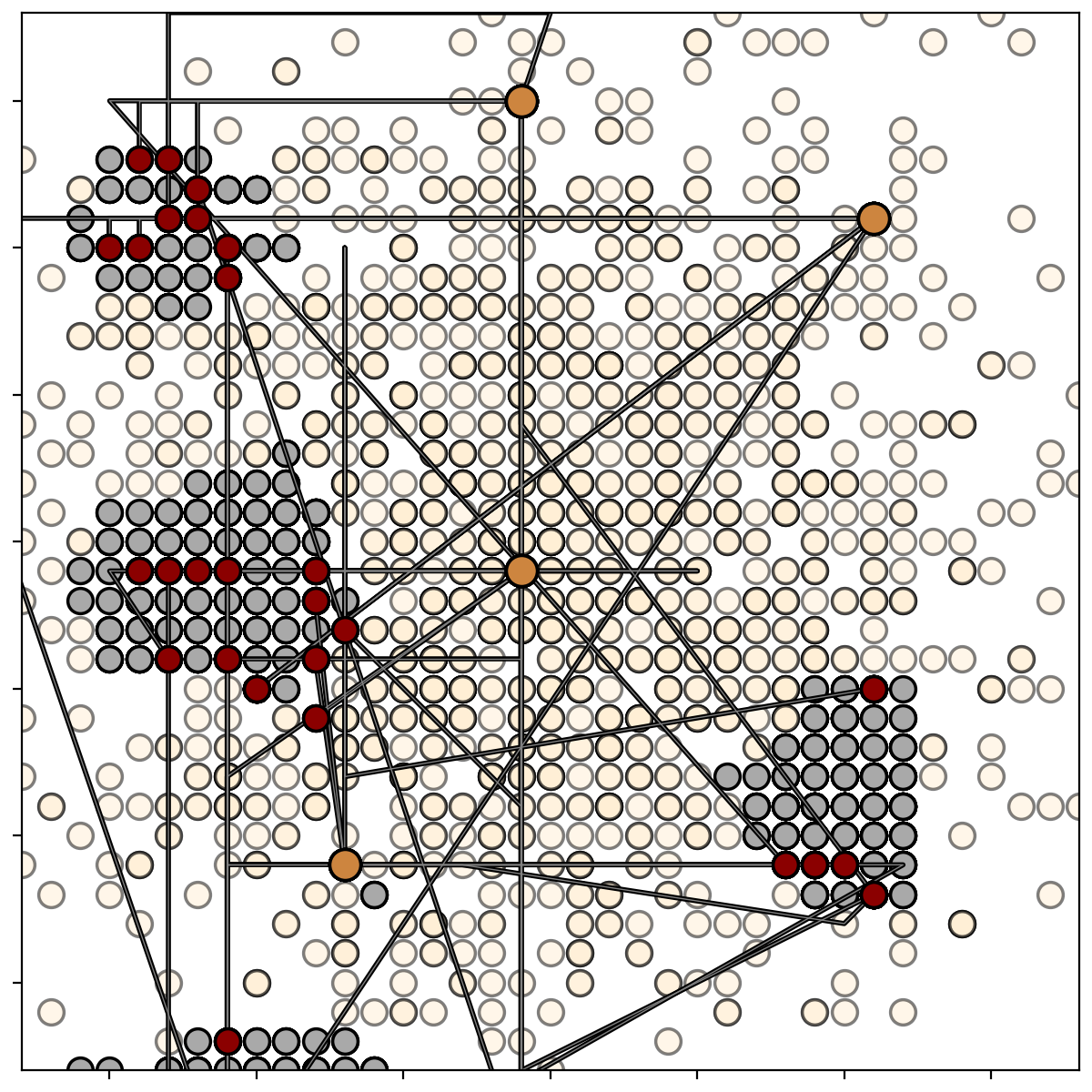}
        \caption{$\alpha$-CSBM MSE/64}
    \end{subfigure}
    \\
    \begin{subfigure}[b]{0.19\linewidth}
        \centering
        \includegraphics[width=0.995\linewidth]{images/dlight_sb/d2_u0005.png}
        \caption{DLightSB}
    \end{subfigure}
    \begin{subfigure}[b]{0.19\linewidth}
        \centering
        \includegraphics[width=0.995\linewidth]{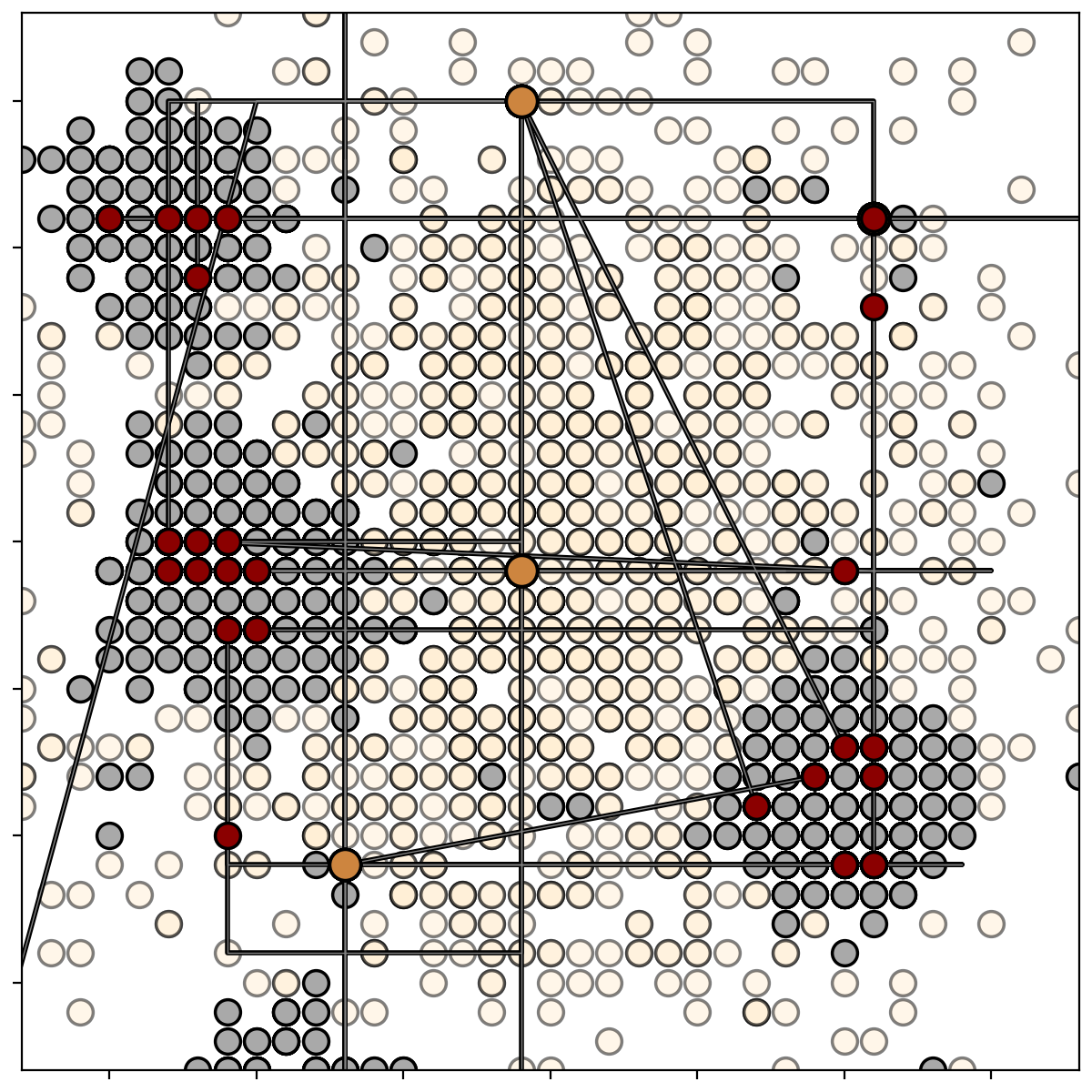}
        \caption{DLightSB-M KL/16}
    \end{subfigure}
    \begin{subfigure}[b]{0.19\linewidth}
        \centering
        \includegraphics[width=0.995\linewidth]{images/dlight_sb_m/d2_u0005_t63_kl.png}
        \caption{DLightSB-M KL/64}
    \end{subfigure}
    \begin{subfigure}[b]{0.19\linewidth}
        \centering
        \includegraphics[width=0.995\linewidth]{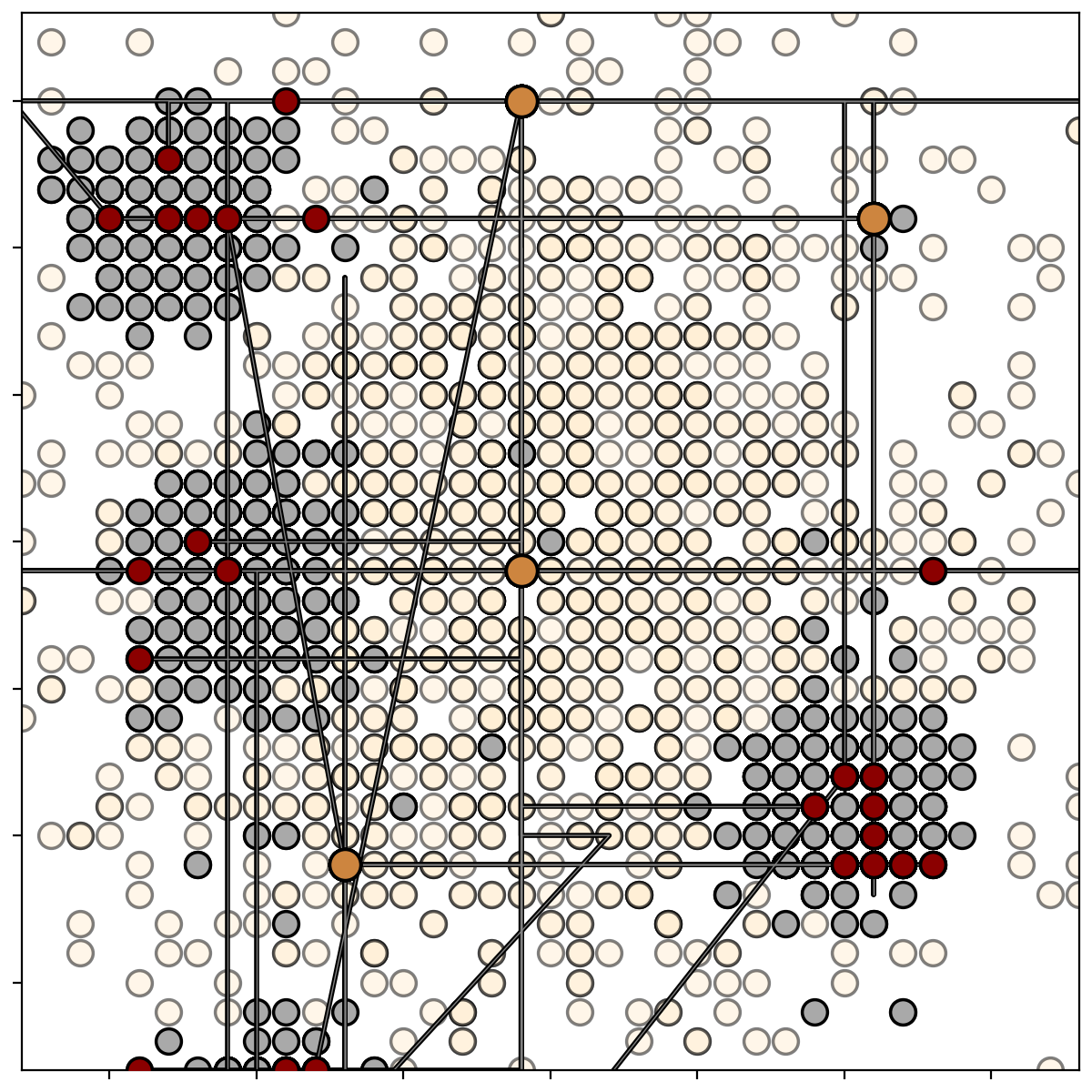}
        \caption{DLightSB-M MSE/16}
    \end{subfigure}
    \begin{subfigure}[b]{0.19\linewidth}
        \centering
        \includegraphics[width=0.995\linewidth]{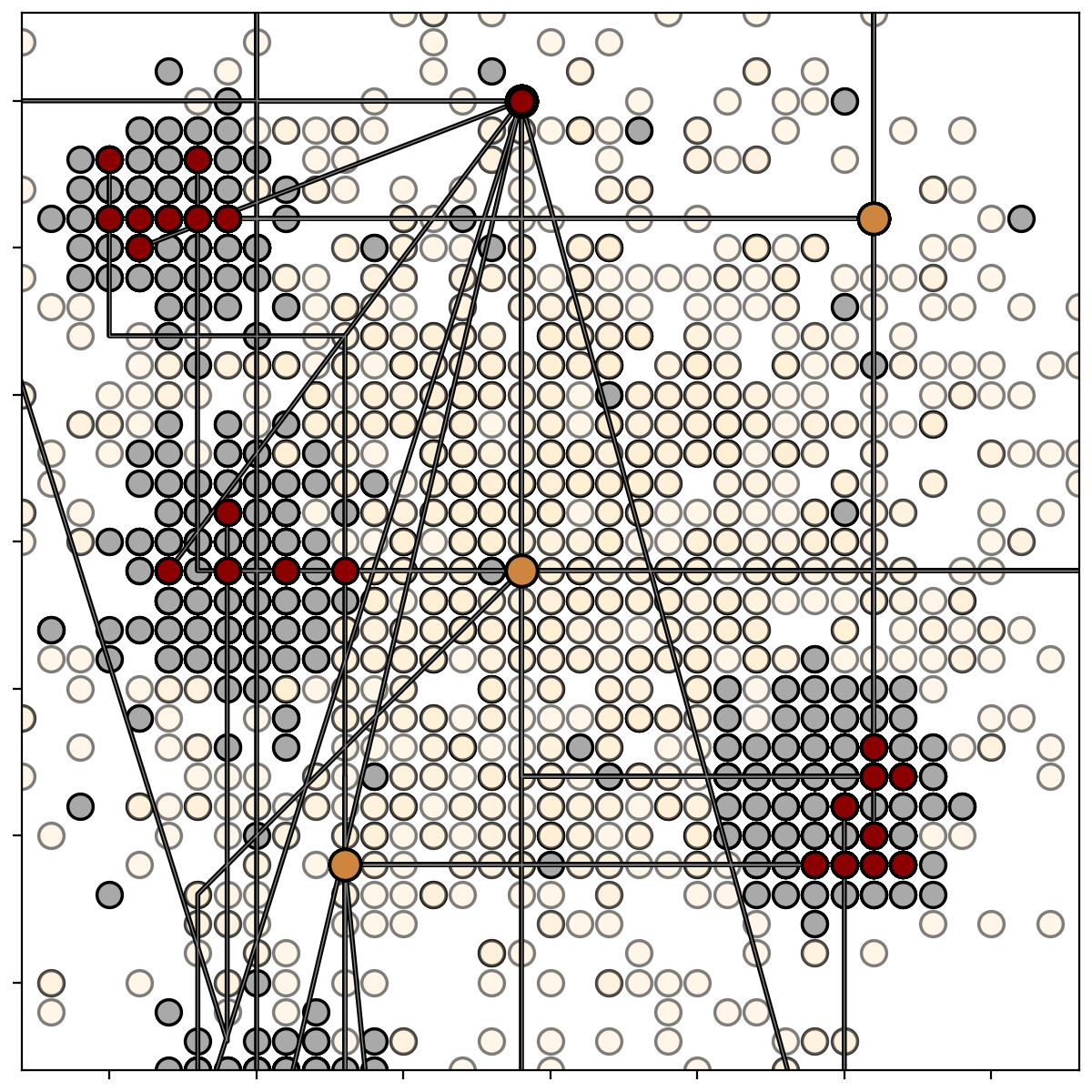}
        \caption{DLightSB-M MSE/64}
    \end{subfigure}
    \caption{\centering Samples from all methods on the high-dimensional Gaussian mixture benchmark using the uniform reference process $q^{\text{unif}}$ with $\gamma = 0.005$.}
    \label{figure:d2_u0005_samples}
\end{figure}

\begin{figure}[h]
    \centering
    \captionsetup[subfigure]{font=scriptsize}
    \begin{subfigure}[b]{0.19\linewidth}
        \centering
        \includegraphics[width=0.995\linewidth]{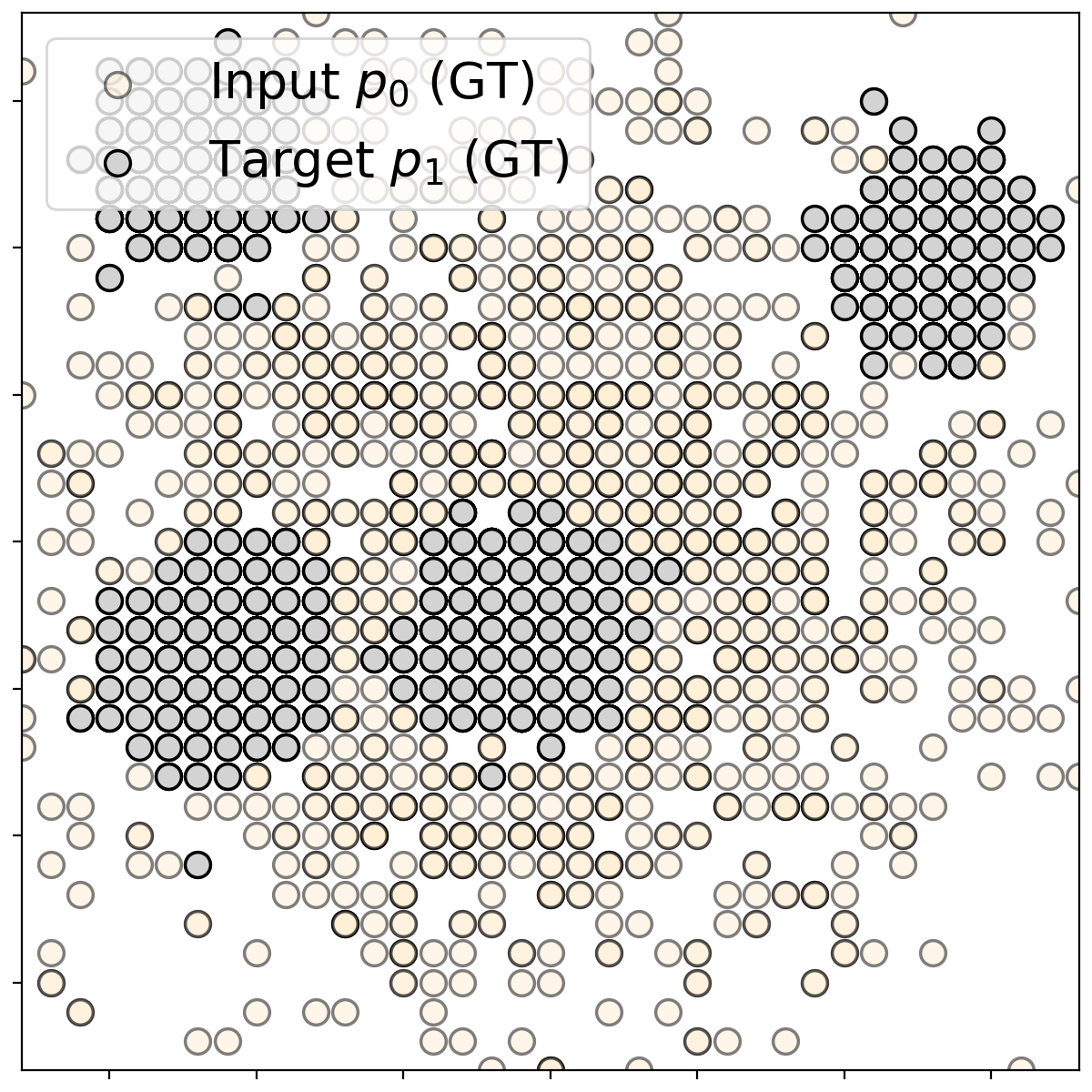}
        \caption{Input and Target}
    \end{subfigure}
    \begin{subfigure}[b]{0.19\linewidth}
        \centering
        \includegraphics[width=0.995\linewidth]{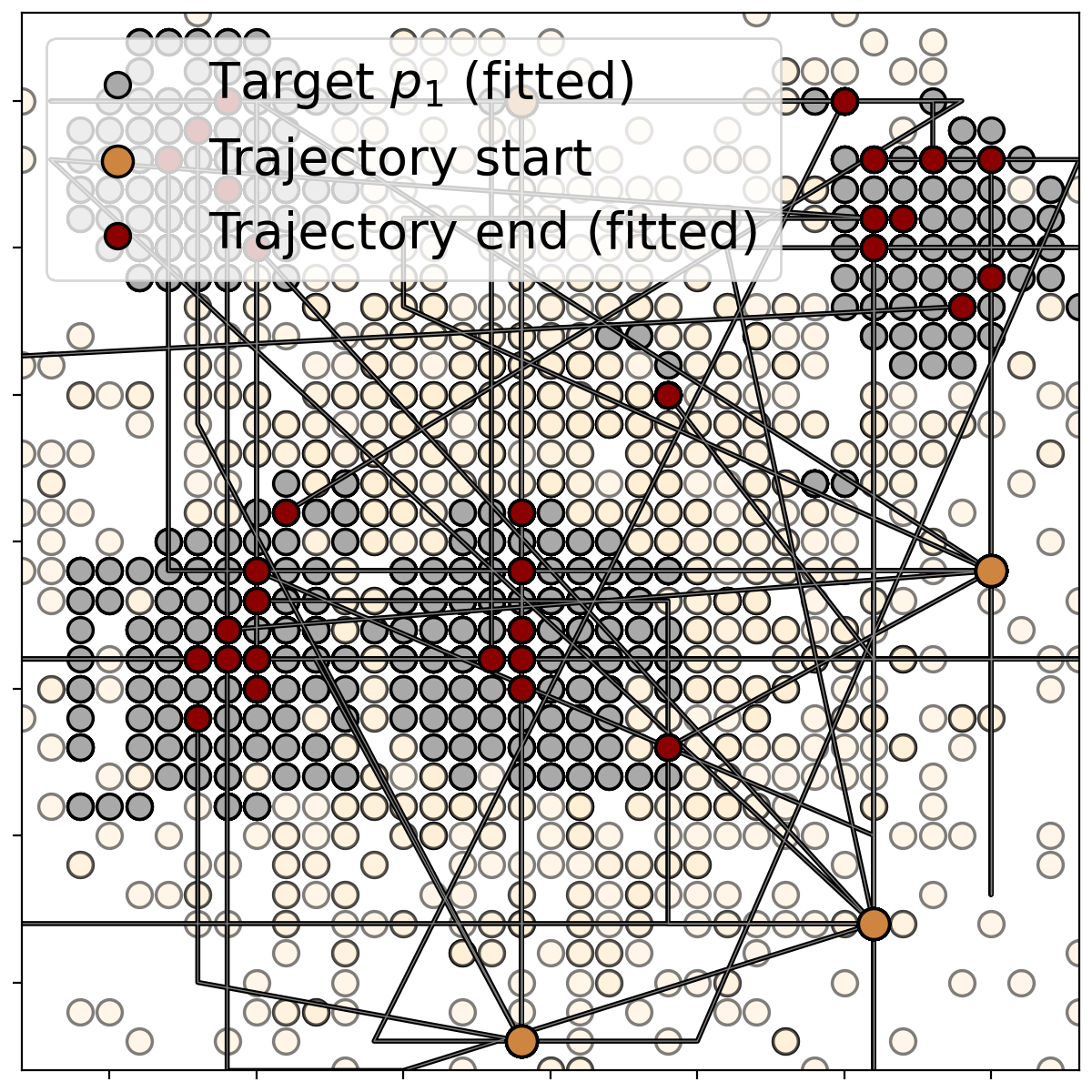}
        \caption{CSBM KL/16}
    \end{subfigure}
    \begin{subfigure}[b]{0.19\linewidth}
        \centering
        \includegraphics[width=0.995\linewidth]{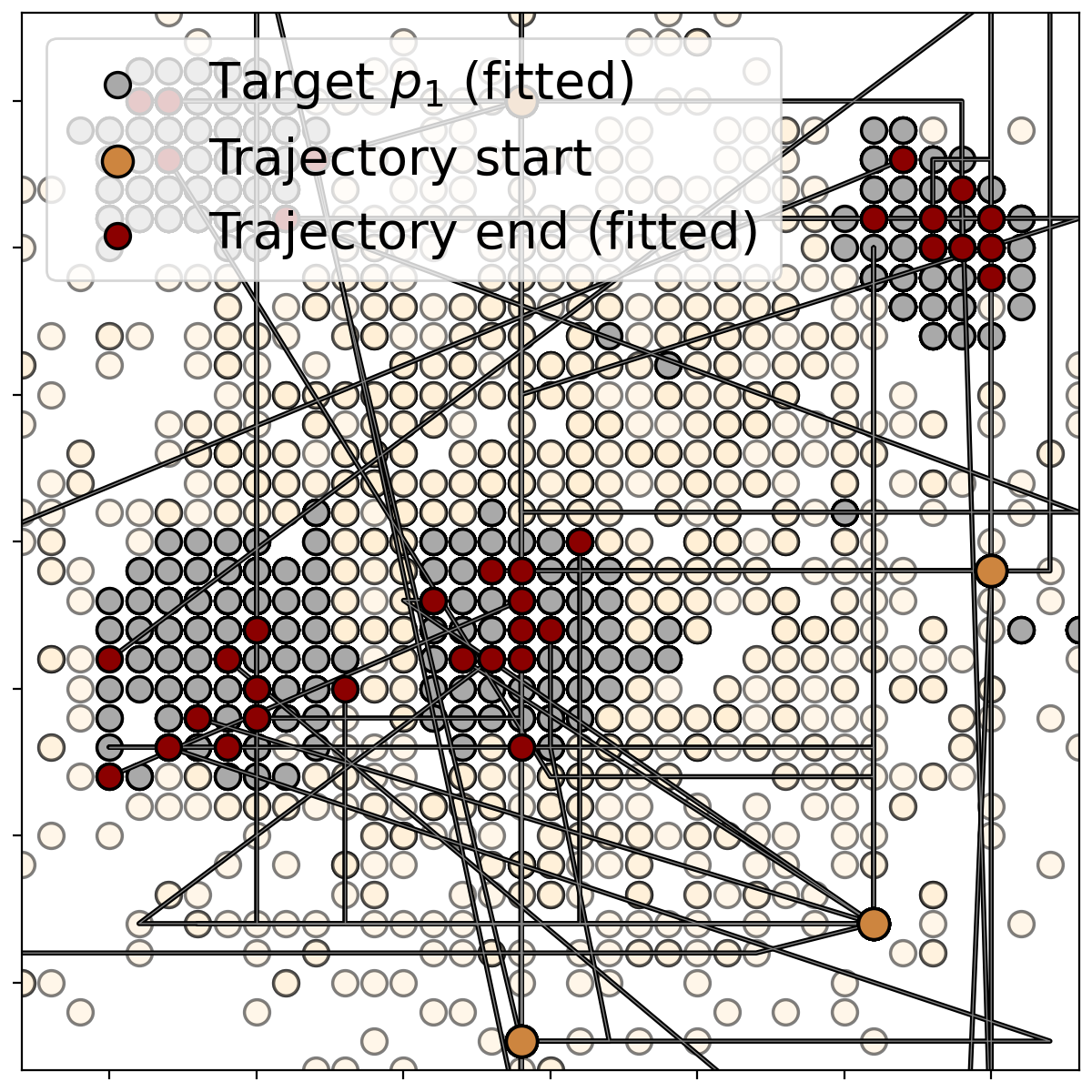}
        \caption{CSBM KL/64}
    \end{subfigure}
    \begin{subfigure}[b]{0.19\linewidth}
        \centering
        \includegraphics[width=0.995\linewidth]{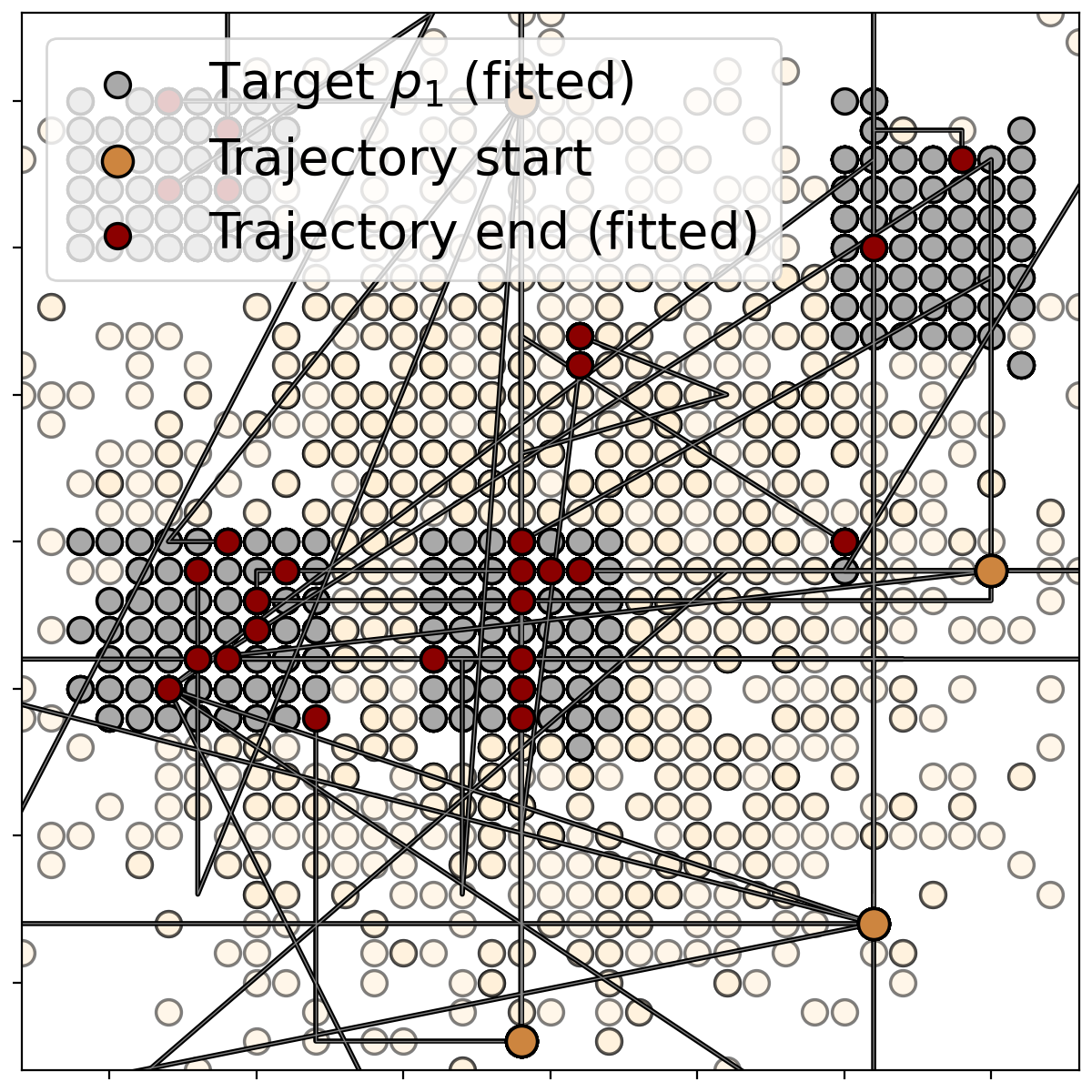}
        \caption{CSBM MSE/16}
    \end{subfigure}
    \begin{subfigure}[b]{0.19\linewidth}
        \centering
        \includegraphics[width=0.995\linewidth]{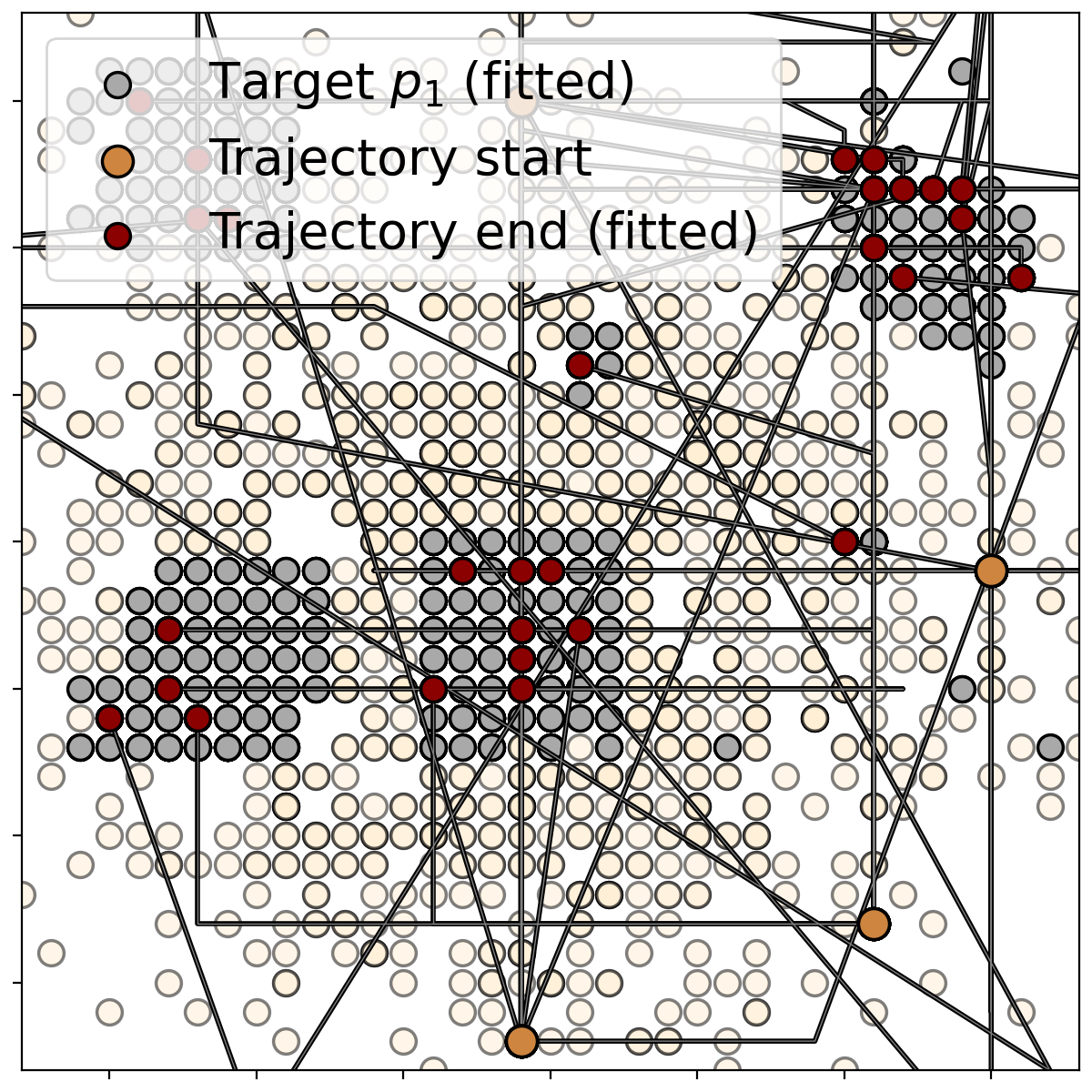}
        \caption{CSBM MSE/64}
    \end{subfigure}
    \\
    \begin{subfigure}[b]{0.19\linewidth}
        \centering
        \includegraphics[width=0.995\linewidth]{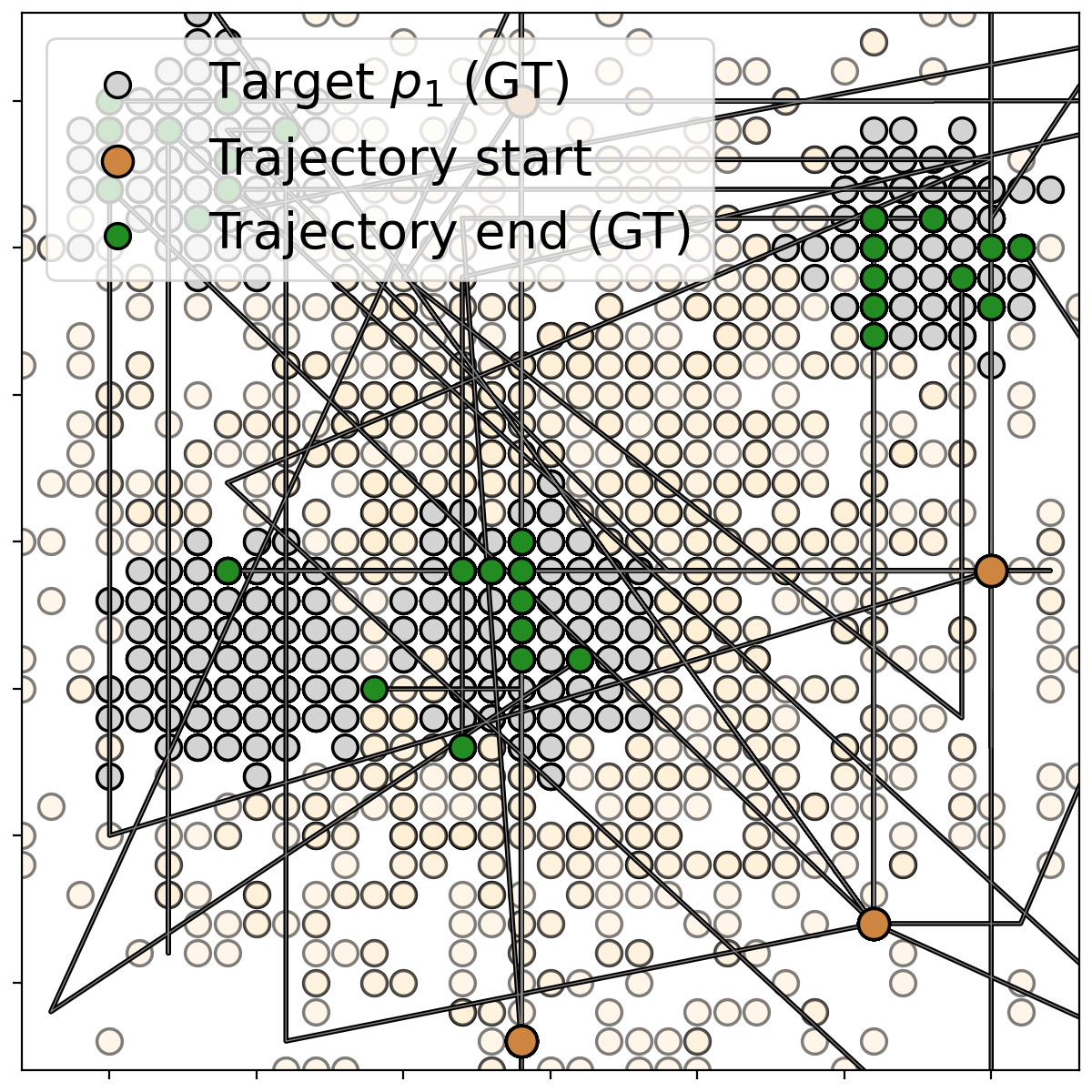}
        \caption{Benchmark}
    \end{subfigure}
    \begin{subfigure}[b]{0.19\linewidth}
        \centering
        \includegraphics[width=0.995\linewidth]{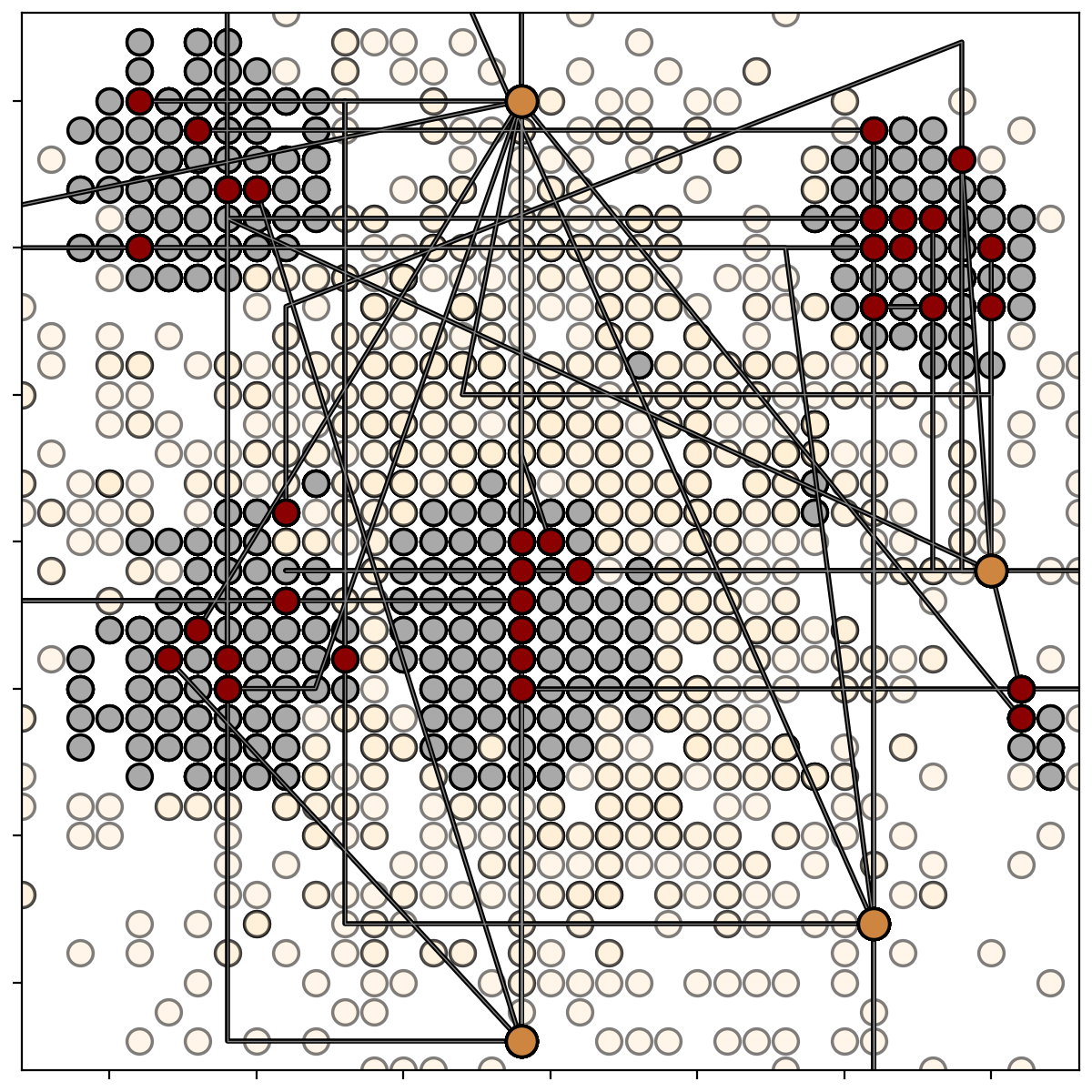}
        \caption{$\alpha$-CSBM KL/16}
    \end{subfigure}
    \begin{subfigure}[b]{0.19\linewidth}
        \centering
        \includegraphics[width=0.995\linewidth]{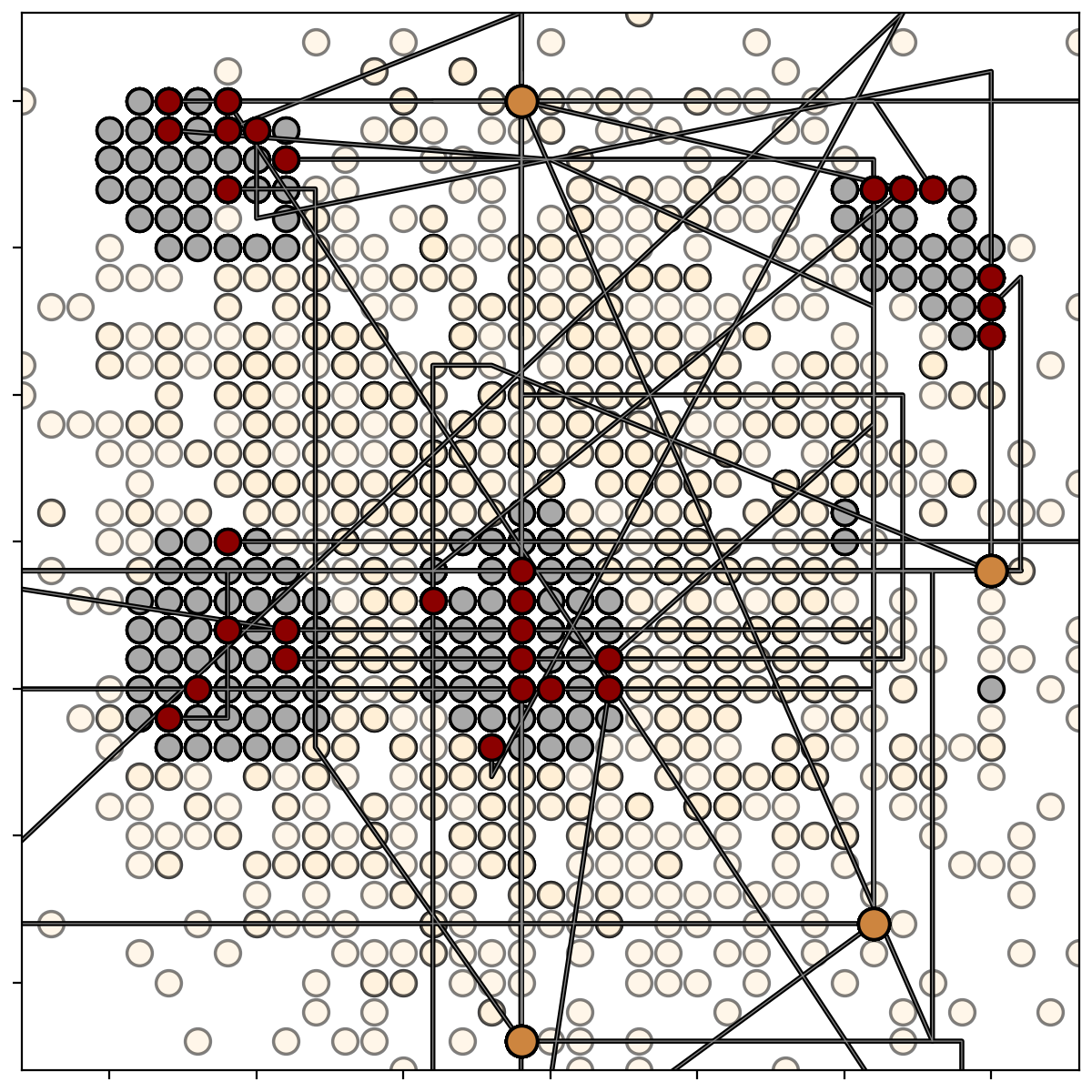}
        \caption{$\alpha$-CSBM KL/64}
    \end{subfigure}
    \begin{subfigure}[b]{0.19\linewidth}
        \centering
        \includegraphics[width=0.995\linewidth]{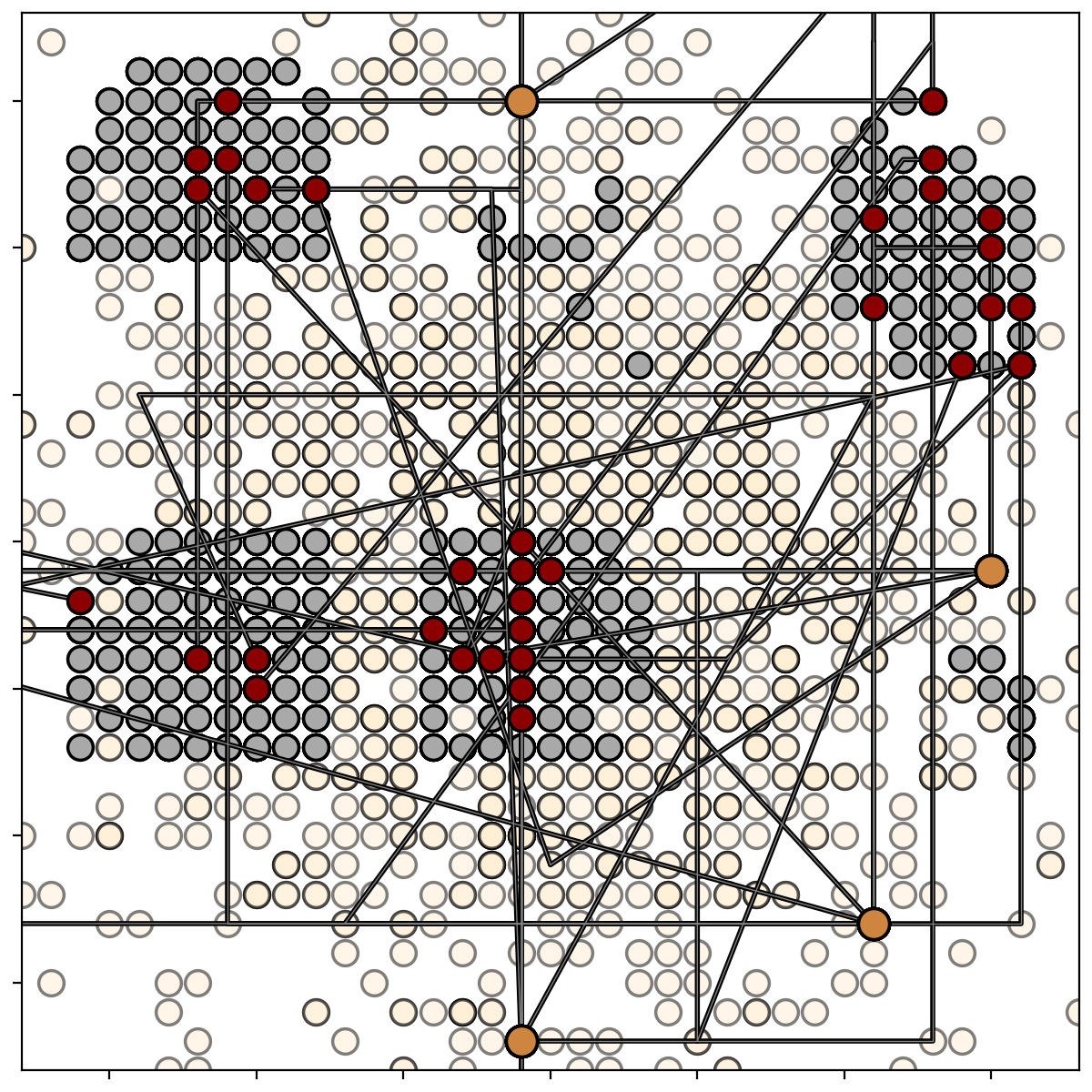}
        \caption{$\alpha$-CSBM MSE/16}
    \end{subfigure}
    \begin{subfigure}[b]{0.19\linewidth}
        \centering
        \includegraphics[width=0.995\linewidth]{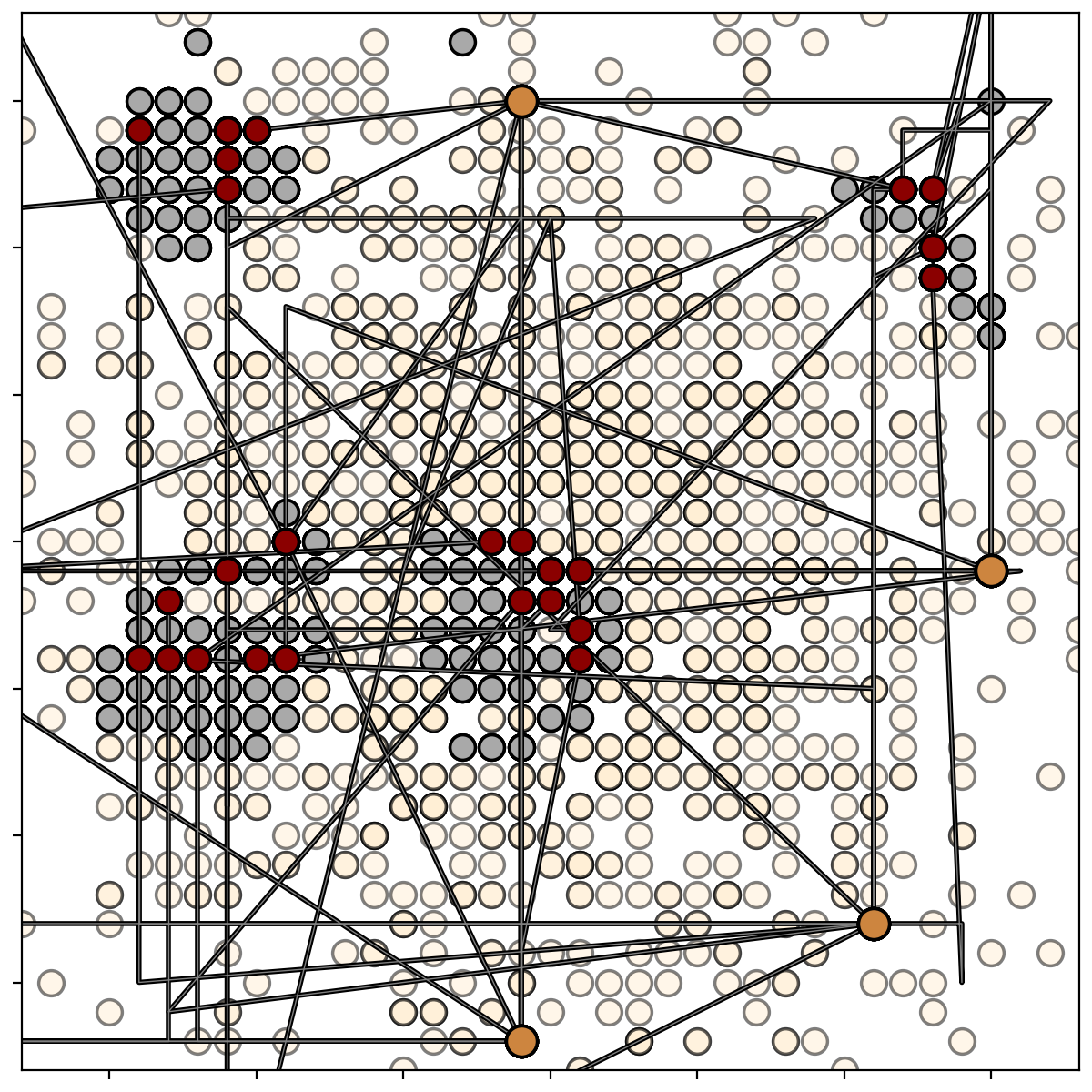}
        \caption{$\alpha$-CSBM MSE/64}
    \end{subfigure}
    \\
    \begin{subfigure}[b]{0.19\linewidth}
        \centering
        \includegraphics[width=0.995\linewidth]{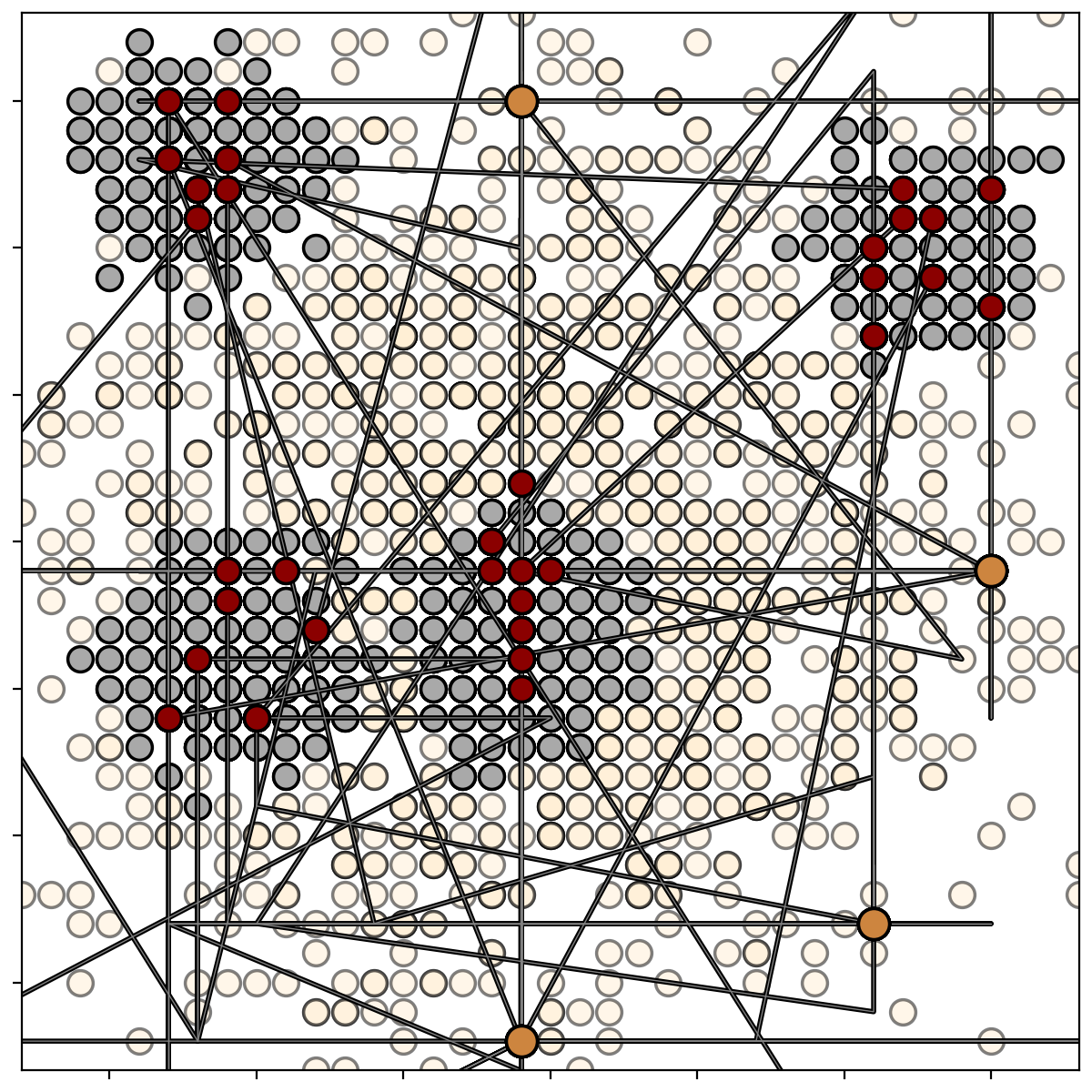}
        \caption{DLightSB}
    \end{subfigure}
    \begin{subfigure}[b]{0.19\linewidth}
        \centering
        \includegraphics[width=0.995\linewidth]{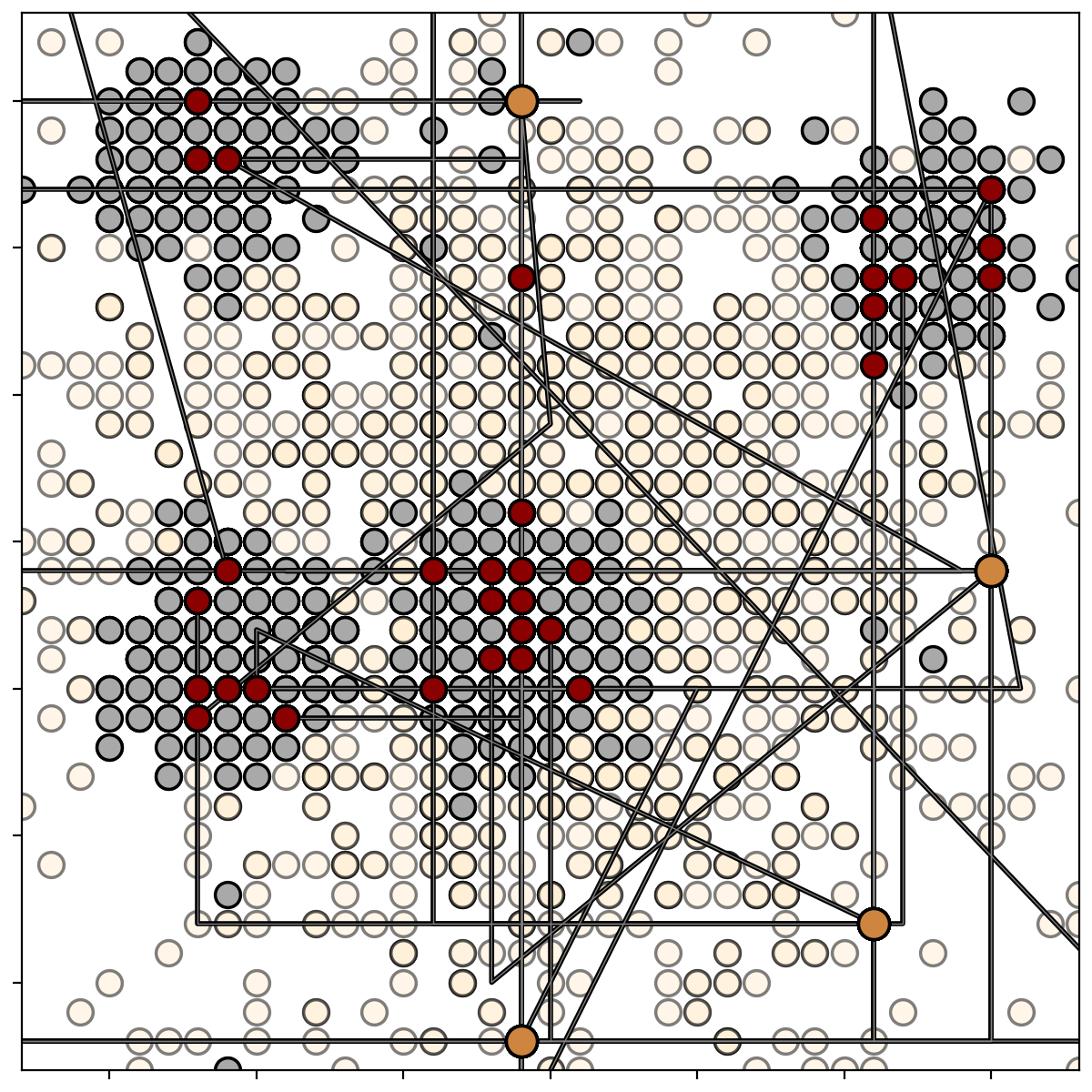}
        \caption{DLightSB-M KL/16}
    \end{subfigure}
    \begin{subfigure}[b]{0.19\linewidth}
        \centering
        \includegraphics[width=0.995\linewidth]{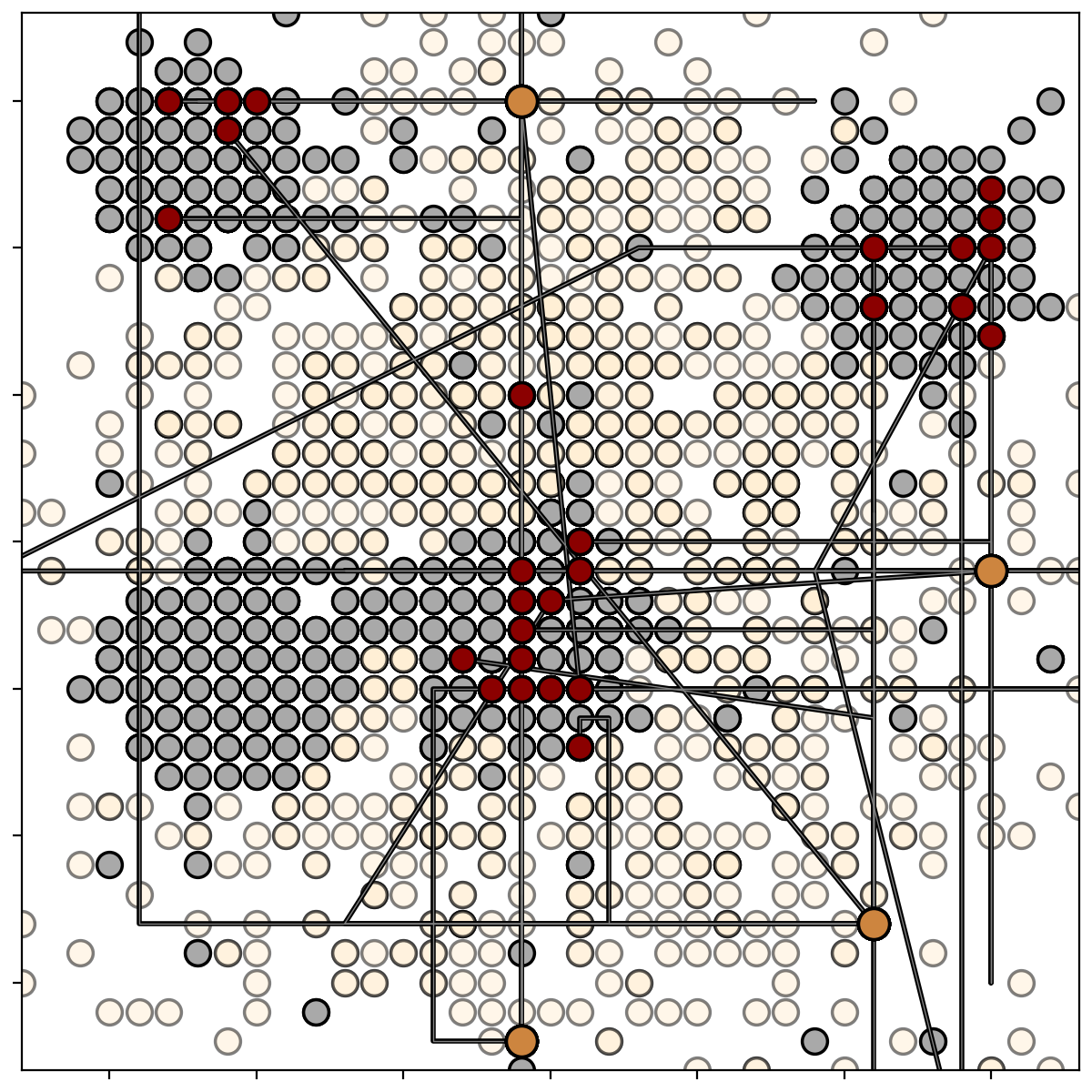}
        \caption{DLightSB-M KL/64}
    \end{subfigure}
    \begin{subfigure}[b]{0.19\linewidth}
        \centering
        \includegraphics[width=0.995\linewidth]{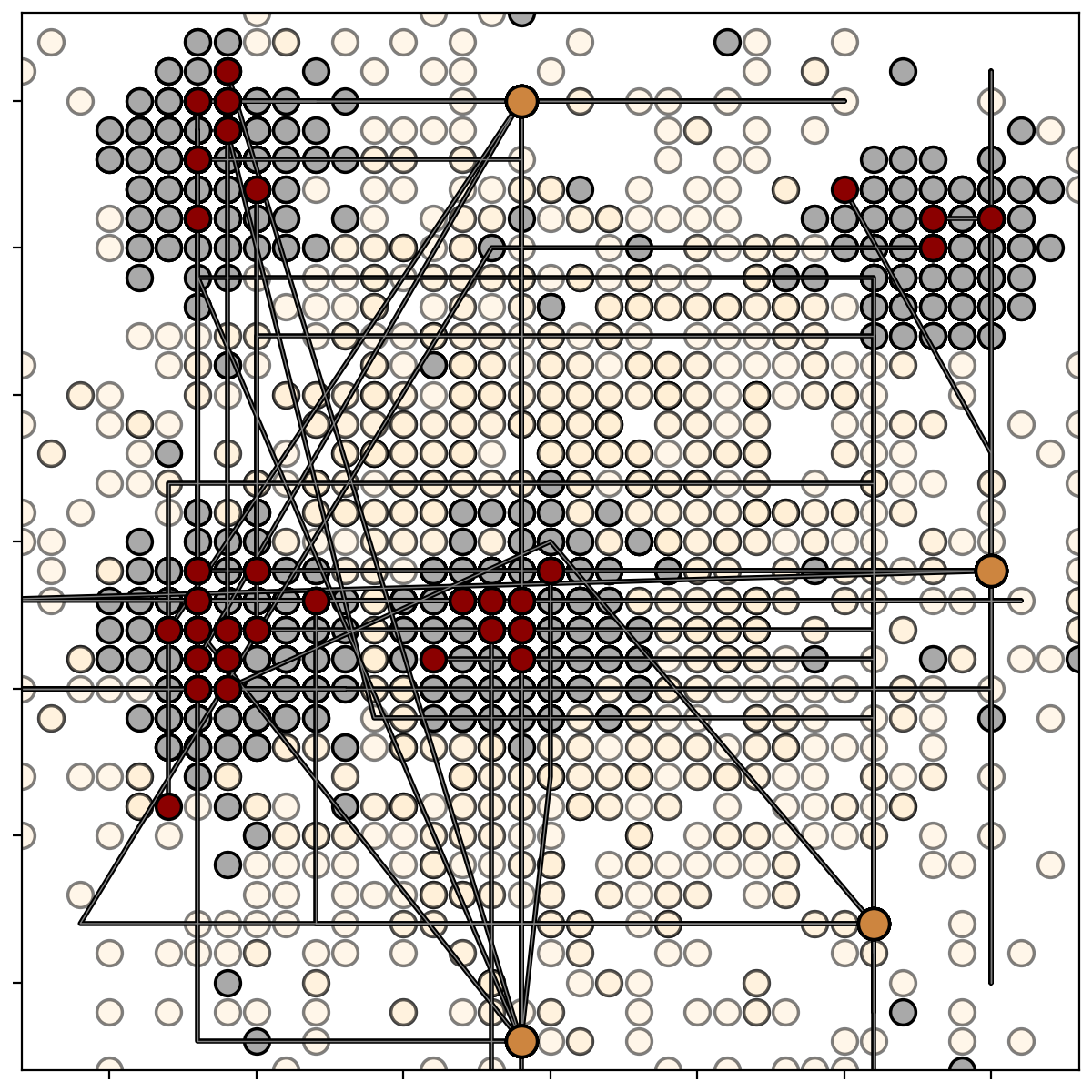}
        \caption{DLightSB-M MSE/16}
    \end{subfigure}
    \begin{subfigure}[b]{0.19\linewidth}
        \centering
        \includegraphics[width=0.995\linewidth]{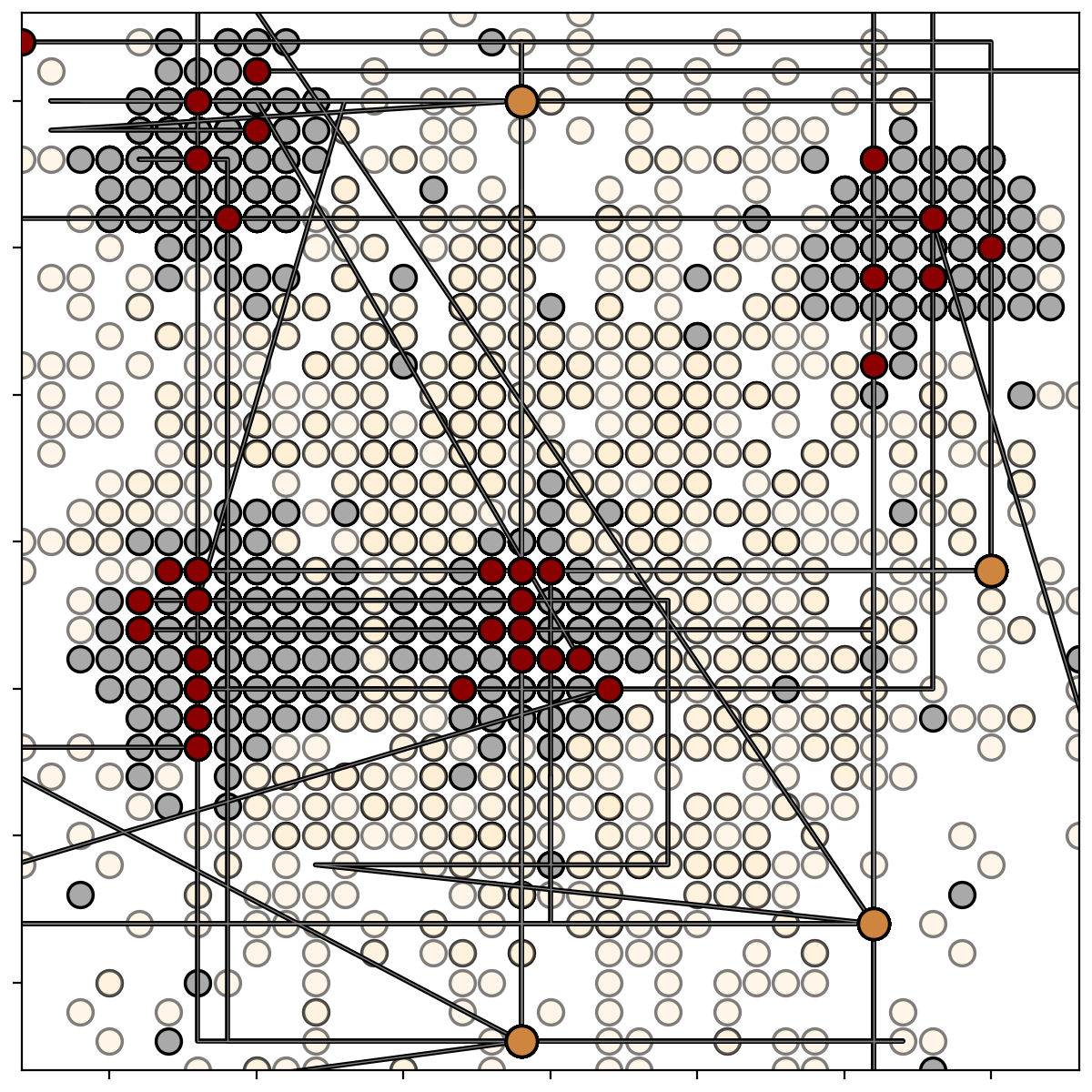}
        \caption{DLightSB-M MSE/64}
    \end{subfigure}
    \caption{\centering Samples from all methods on the high-dimensional Gaussian mixture benchmark using the uniform reference process $q^{\text{unif}}$ with $\gamma = 0.01$.}
    \label{figure:d2_u001_samples}
\end{figure}

\end{document}